\documentclass[11pt]{article}
\usepackage{fullpage}
\usepackage{wustyle}

\usepackage{braket}
\usepackage[a4paper,margin=1in]{geometry}
\usepackage[toc,page,header]{appendix}
\usepackage{minitoc}

\newcommand\blfootnote[1]{%
  \begingroup
  \renewcommand\thefootnote{}\footnote{#1}%
  \addtocounter{footnote}{-1}%
  \endgroup
}

\newcommand{\tH}{\tilde{\cH}}

\title{Optimal Rates and Saturation for Noiseless \\ Kernel Ridge Regression}

\author{
Jihao Long\thanks{Institute for Advanced Algorithmic Research, Shanghai, \texttt{longjh1998@gmail.com}}
\and
Xiaojun Peng\thanks{University of Science and Technology of China, \texttt{pengxiaojun@mail.ustc.edu.cn}}
\and
Lei Wu\thanks{Peking University,  \texttt{leiwu@math.pku.edu.cn}}
\thanks{AI for Science Institute, Beijing}
}

\date{\today}

\begin{document}

\maketitle

\begin{abstract} 
Kernel ridge regression (KRR), also known as the least-squares support vector machine, is a fundamental method for learning functions from finite samples. While most existing analyses focus on the noisy setting with constant-level label noise, we present a comprehensive study of KRR in the {\it noiseless regime} -- a critical setting in scientific computing where data are often generated via high-fidelity numerical simulations.

We establish that, up to logarithmic factors, noiseless KRR achieves  minimax optimal convergence rates, jointly determined by the eigenvalue decay of the associated integral operator and the target function's smoothness. These rates are derived under  Sobolev-type interpolation norms, with the $L^2$ norm as a special case.
Notably, we uncover two key phenomena: an {\it extra-smoothness} effect, where the KRR solution exhibits higher smoothness than typical functions in the native reproducing kernel Hilbert space (RKHS), and a {\it saturation} effect, where the  KRR's adaptivity to the target function's smoothness plateaus beyond a certain level. Leveraging these insights, we also derive a novel error bound for noisy KRR that is \emph{noise-level aware} and achieves minimax optimality in both noiseless and noisy regimes.
As a key technical contribution, we introduce  a refined notion of degrees of freedom, which we believe has broader applicability in the analysis of kernel methods. Extensive numerical experiments validate our theoretical results and provide insights beyond existing theory.
\end{abstract}


\noindent \textbf{MSC2020 Math Subject Classifications:} 41A25, 41A46,  46E22, 62G08, 65J22

\doparttoc 
\faketableofcontents 
\part{} 

\blfootnote{All authors contributed equally, and the order follows the alphabetical convention.}
\vspace*{-5em}
\section{Introduction}
In this paper, we consider the problem of  learning an unknown target function $f^*:\cX\mapsto\RR$ from a finite set of  samples $\{(x_i, y_i = f^*(x_i) + \xi_i)\}_{i=1}^n$, where $\{x_i\}_{i=1}^n$ are input points  and $\{y_i\}_{i=1}^n$ are the potentially noisy labels with $\{\xi_i\}_{i=1}^n$ representing the label noise. The goal is to learn $f^*$ as accurately as possible, typically measured using a  norm such as the $L^2$ norm. This regression problem plays a fundamental role across  applied mathematics, statistics, and machine learning. 

Among the many regression methods, kernel ridge regression (KRR), also known as least-squares support vector machines~\cite{steinwart2008support}, stands out for its robust theoretical guarantees,  practical flexibility, and  computational efficiency. Specifically,
KRR employs a reproducing kernel Hilbert space (RKHS) \cite{aronszajn1950theory} associated with a positive  semidefinite kernel $k: \cX \times \cX \to \RR$ as its hypothesis space. Specifically, the KRR estimator $\hf_\lambda$ is obtained by   solving:
\begin{equation}\label{eqn: krr-0}
\hf_\lambda = \argmin_{f\in \cH} \frac{1}{n}\sum_{i=1}^n (f(x_i)-y_i)^2 + \lambda \|f\|_{\cH}^2,
\end{equation}  
where $\lambda$ is a regularization hyperparameter that balances the trade-off between data fitting and function smoothness. Although $\cH$ may be infinite-dimensional, the representer theorem~\cite{scholkopf2001generalized} ensures that $\hf_\lambda$ must lie in the linear span of $\{k(x_i,\cdot)\}_{i=1}^n$. Consequently, the  optimization problem in \eqref{eqn: krr-0}, though potentially infinite-dimensional, can be reduced to an $n$-dimensional quadratic program, allowing it to be solved computationally efficiently.

Over the past decades, the theoretical underpins of KRR have been extensively developed in applied mathematics~\cite{schaback2006kernel}, non-parametric statistics~\cite{gyorfi2002distribution}, and machine learning~\cites{smale2007learning,steinwart2008support} literature. A central focus has been on analyzing the convergence rates of the generalization error, typically measured by the $L^2$ norm. Assuming that the level of label noise is constant,   optimal  rates can be established under various assumptions; see, e.g., \cites{caponnetto2007optimal,steinwart2008support,steinwart2009optimal,dicker2017kernel,fischer2020sobolev,li2023asymptotic} and the references therein. In particular,  \cite{fischer2020sobolev} extended these results by measuring the generalization error in a family of  Sobolev-type  norms \(\cH^p\) for $p\in [0,1]$. Notably, the $\cH^p$ norm provides a continuous interpolation between the standard $L^2$ norm (at \(p=0\)) and the native RKHS norm (at $p=1$)~\cite{smale2007learning}, thereby enabling a more fine-grained characterization of the learning behavior of KRR. 

While these results are foundational for classical statistical learning scenarios, they rely on the premise of substantial label noise. This assumption becomes restrictive in modern applications such as \emph{scientific computing}, where data often comes from high-fidelity numerical simulations (e.g., solving partial differential equations \cites{sirignano2018dgm,han2018solving,raissi2019physics,yu2018deep,rfm-pde} or modeling molecular dynamics \cite{zhang2018deep}), in which the noise can be negligible. This mismatch raises a pivotal question: 
\begin{center}
{\it Do standard convergence guarantees and related phenomena for noisy \\ KRR remain valid in the noiseless regime?}
\end{center}
Moreover, understanding the noiseless case is crucial for characterizing how learning error evolves with varying noise levels. As we will show, existing convergence rates for noisy KRR -- derived under constant-level noise assumptions -- do not smoothly transition to the correct noiseless rate in the vanishing noise limit.

\subsection{Our Contribution}
In this paper, we present a comprehensive study of KRR in the noiseless setting. Specifically,
we consider the classical source condition~\cite{caponnetto2007optimal} on the target function:
 $f^* \in \cH^s$ with $s\geq 0$, where $\cH^s$ is a Sobolev-type interpolation space, quantifying the relative smoothness of $f^*$ (see Definition~\ref{def: Hks-space}).  A larger $s$ corresponds to a smoother $f^*$. Notably, $s > 1$ indicates that $f^*$ has smoothness exceeding that of  the RKHS (since $\cH=\cH^1$), while $s < 1$ represents a \emph{misspecified case} where $f^*$ lies outside the hypothesis space.
For the underlying kernel $k:\cX\times \cX\mapsto\RR$, we assume $\int k(x,x)\dd\rho(x)<\infty$ and let $\{\mu_j\}_{j=1}^\infty$ denote the eigenvalues of the associated integral operator, arranged in non-increasing order. Accordingly, let $\{e_j\}_{j=1}^\infty$ be the corresponding orthonormal eigenfunctions and assume they form a complete basis of $L^2(\rho)$. For more details, we refer to Section~\ref{sec: preliminary}.
We measure the learning error under the interpolation norm:
\[
\cE_{p}(n,\lambda;f^*):=\|\hf_\lambda-f^*\|_{\cH^p}^2, \qquad \cE_{p}(n;f^*)=\min_{\lambda\geq 0} \cE_{p}(n,\lambda;f^*). 
\]
Note that $\cE_{p}(n;f^*)$ represents the error achieved by optimally-tuned KRR. 


\paragraph*{\bf Optimal rates of noiseless KRR.}
We first show that as long as $p$ is sufficiently small such that $\sum_j\mu_j^{2-p}<\infty$ and $p\leq s$, then, with a high probability, the error satisfies 
\[
\cE_{p}(n;f^*)=\tO(\mu_n^{\min(s,2)-p}),
\] 
where $\tO(\cdot)$ denotes the standard big-O notation with logarithmic factors omitted.
Furthermore, if $s$ is sufficiently large  such that $\sum_j\mu_j^s<\infty$, uniform learning is achievable in the sense that
$
\sup_{f^*\in \cH^s}\cE_{p}(n;f^*)=\tO(\mu_n^{\min(s,2)-p}).
$
This indicates that the same set of $\{x_i\}_{i=1}^n$ can be used to learn effectively all functions within $\cH^s$. Additionally, we establish lower bounds that confirm the optimality of these learning rates, up to logarithmic factors. These results reveal two key insights about KRR:
\begin{itemize}
	   \item {\it The adaptability to the target function's smoothness:} In the absence of noise,   KRR can achieve optimal learning rates for target functions with smoothness up to $s=2$, mirroring the behavior observed in noisy settings.
  \item {\it The extra-smoothness effect:} The KRR solution exhibits a higher degree of smoothness than expected,  as the parameter $p$ can exceed $1$, despite KRR being designed to search solutions within $\cH^1$.
\end{itemize}

\vspace*{-.5em}
\paragraph*{Improved results for noisy KRR.} Building on insights from our noiseless results, we derive a new upper bound for KRR in the presence of noise. This bound offers two key advantages over existing results: (1) First, it accommodates the case of $p > 1$, while existing bounds are restricted to $p \in [0,1]$. This broadens the scope of problems to which KRR can be effectively applied. (2) Second, our bounds are {\it noise-level aware}, simultaneously achieving optimal rates in both noisy and noiseless settings. 

\vspace*{-.5em}
\paragraph*{Saturation effect in noiseless KRR.}  
We present a systematic theoretical analysis of the saturation effect in noiseless KRR, confirming the necessity of the assumptions made about the smoothness parameters \( s \) and \( p \).  Our analysis highlights two critical findings:
\begin{itemize}
  \item {\bf Smoothness constraints on the KRR solution.} We prove that if \( p \) is sufficiently large such that \( \sum_j \mu_j^{2-p} = \infty \), then there exists  RKHSs in which the $\cH^p$ error of KRR solutions diverges (i.e., \( \cE_{p}(n; f^*) = \infty \)) with probability $1$ for any nonzero \( f^* \in \cH^s \). This result indicates that the smoothness of a KRR solution is constrained by the threshold 
$
p_0 = \inf \{ p : \sum_j \mu_j^{2-p} < \infty \},
$
justifying the necessity of our assumption on $p$.
  \item {\bf Saturation of smoothness adaptation.} We  show that if $f^*$ has smoothness greater than $2$ (i.e., $s>2$), the performance of noiseless KRR reaches a plateau and does not improve further. This saturation stems from the representer theorem, which confines the solution in  \(\spn\{ k(x_i,\cdot): i\in [n]\} \) and hence, restricts the model's expressive power. Leveraging this insight, we establish lower bounds for dot-product kernels on the unit sphere and periodic kernels on the torus, showing that the saturation indeed occurs precisely at \( s = 2 \).
\end{itemize}

In Section~\ref{sec: experiment}, we present extensive numerical experiments that validate both the optimality of the convergence rates and the presence of the saturation effect. These experiments indicate that our theoretical findings reflect typical performance rather than merely representing worst-case scenarios.
\begin{remark}[Polynomial decay]
When \( \mu_j \lesssim j^{-\beta}(\log j)^{-\alpha} \) with \( \beta \geq 1 \) and \( \alpha \geq 1 \), the error under $\cH^p$ norm is given by \( \tilde{O}(n^{-\beta(s-p)}) \) for \( s \in [0,2], p \in [0, \min(s,2-1/\beta)] \).  A larger value of \( \beta \), corresponding to faster eigenvalue decay, permits larger values of \( p \), meaning the learned function can achieve higher-order smoothness. Notably, when the eigenvalue decay is sufficiently fast, \( p \) can be as large as $2$. 
\end{remark}

\begin{remark}[Exponential decay]
When \( \mu_j \lesssim c^j \) with \( c \in (0,1) \), KRR achieves an error rate of \( e^{-\tilde{\Omega}(n)} \). This is known as {\it spectral accuracy} in numerical analysis \cite{trefethen2000spectral}, suggesting that logarithmically many samples suffice for learning functions in the corresponding RKHS.
\end{remark}

\paragraph*{Technical innovations.} A central technical contribution of this work is the introduction of a refined notion of degrees of freedom (DoFs) (see Definition~\ref{def: degree-Fp}), which enables  our fine-grained analysis of noiseless KRR. In particular,  the refined DoFs allow us to to establish error bounds under the $\cH^p$ norm for values of $p$ exceeding the conventional range of $[0,1]$. In addition, departing from prior approaches that rely on the embedding properties~\cite{fischer2020sobolev}, we develop a symmetry-based argument to bound  DoFs, yielding a general estimate that applies to all symmetric kernels. We believe that these techniques have broader implications for the analysis of kernel methods, as evidenced by our improved results on noisy KRR discussed earlier.

\vspace*{-.3em}
\subsection{Related Work} 
\vspace*{-.1em}
To contextualize our contributions and highlight how our results advance the understanding of KRR and the learning of RKHS functions, we provide a detailed comparison with prior work.

\vspace*{-.4em}
\paragraph*{KRR in the noisy regime.}
 Pioneering works, adopting the {\it integral operator} technique, established that KRR is minimax optimal for the well-specified case where the relative smoothness of the target function satisfies $s\in [1,2]$~\cites{caponnetto2007optimal,steinwart2009optimal}. Later, by employing {\it empirical process} techniques, \cite{steinwart2009optimal} extended the  optimal rates under the $L^2$ norm to the misspecified case, where  $s$ can be smaller than $1$. However, empirical process techniques  cannot handle the convergence under the general interpolation norm, as it typically requires the loss function to be Lipschitz continuous. To address this, \cite{fischer2020sobolev} combined the integral operator technique with an  embedding property—which captures the relationship between eigenvalues and eigenfunctions—enabling  an analysis of noisy KRR in the misspecified setting under the interpolation norm \( \cH^p \) for \( p \in [0,1] \). 
However, \cite{fischer2020sobolev} imposed the additional requirement \(\|f^*\|_{\infty} < \infty\), i.e., $f^*\in \cH\cap L^\infty(\rho)$. This condition was later relaxed in \cite{zhang2023optimality} to \(f^* \in \cH^s \cap L^q(\rho)\) for some \(q > 2\). Moreover, \cite{zhang2024optimality} showed that  $\cH$ can be  continuously embedded into $L^q(\rho)$, thereby establishing the optimality of KRR with no further assumptions.  In this paper, we instead incorporate the classical integral operator technique with a newly introduced degrees of freedom (DoFs), enabling a comprehensive and precise analysis of KRR in the noiseless regime. Moreover, leveraging this new DoF, we also establish the optimality of noisy KRR under the \(\cH^p\) norm with \(p > 1\), extending beyond existing results that are limited to  $p\in [0,1]$.

For the case where \( s > 2 \), previous works such as \cites{bauer2007regularization, dicker2017kernel, li2022saturation} have shown that KRR achieves the same convergence rate as in the case where \( s = 2 \). This phenomenon, known as \emph{saturation}, implies that KRR does not benefit from smoothness beyond \( s = 2 \). 
In this paper, we establish that noiseless KRR exhibits a similar saturation effect at the same critical threshold \( s = 2 \). Furthermore, while spectral algorithms such as spectral cut-off can overcome saturation in the presence of noise~\cites{bauer2007regularization, lin2020optimal}, we show that they fail to do so in the noiseless setting. This is because the saturation in noiseless KRR arises from the representer theorem, rather than from the ridge-type regularization as in the noisy regime~\cite{li2022saturation}. 

\vspace*{-.2em}
\paragraph*{KRR in the noiseless regime.}
The analysis of noiseless KRR has recently gained attention in the context of learning curve studies in machine learning. Several works, such as \cites{spigler2020asymptotic, bordelon2020spectrum, cui2021generalization}, employed non-rigorous approaches under Gaussian design to derive convergence rates. A more rigorous analysis was provided in  
\cite{li2023asymptotic}, which removed the Gaussian assumption but considered a modified source condition assumption:
$
    f^* = \sum_{j=1}^\infty a_j j^{-\frac{1}{2}}\mu_j^{\frac{s}{2}}e_j
$
with $s>0$ and $|a_j| \le C$. 
In contrast, we focus on the standard source condition:
$
    f^* = \sum_{j=1}^\infty a_j \mu_j^{\frac{s}{2}} e_j
$
with $\sum_{j=1}^\infty a_j^2 < \infty$. 
Our contributions include a comprehensive analysis of noiseless KRR under the general \(\cH^p\) norm, along with a detailed characterization of the admissible range of \(p\). In comparison, \cite{li2023asymptotic} focused solely on the \(L^2\) error (i.e., \(p = 0\)). Technically, our approach differs significantly, enabling us to work under the standard source condition without modifications and to provide a thorough study of the range of $p$. Additionally, we establish lower bounds that reveal the fundamental origin of the saturation effect in noiseless KRR, showing that it stems from the representer theorem rather than ridge-type regularization. This  highlights a fundamental limitation of kernel-based methods for function learning in the absence of noise.

\vspace*{-.2em}
\paragraph*{Sampling number of RKHS.}  
The sample complexity of learning functions in a RKHS in the noiseless setting has been studied in approximation theory under the notion of the \emph{sampling number}; see, e.g., \cites{wasilkowski2001power, kuo2009power}. 
Recently, \cite{krieg2021function} showed that the \( L^2 \) error is of order \( O(\mu_n) \) (up to logarithmic factors) when employing a  weighted sampling strategy. \cite{dolbeault2023sharp} further demonstrated that if input points can be chosen arbitrarily, the logarithmic factors can be removed, relying on a solution to the Kadison-Singer problem. 
From the perspective of KRR, these analyses focus on the specific case of \( s = 1 \) and \( p = 0 \). In contrast, we consider the learning of the entire range \( s \in (0, \infty) \) and under a general \(\cH^p\) norm. Additionally, we also strengthen the results of \cite{krieg2021function} by showing that, for kernels with certain symmetry properties, weighted sampling is unnecessary.

   \vspace*{-.5em}
\subsection{Notations}

We use $\NN^{+}$ denote the set of positive integers. 
We shall use $\|\cdot\|$ to denote an operator norm without specifying the input and output space for brevity.  Let  $\cP(\Omega)$ denote the set of probability measures over $\Omega$.  
 For any $\gamma \in \cP(\Omega)$, we denote by $L^q(\cX, \gamma)$ the classical $L^q$ space, and often write $L^q(\gamma)$ when the underlying space $\cX$ is clear from context. For any $f, g \in L^2(\gamma)$, we write $\langle f, g \rangle_\gamma = \int f(x)g(x) , \dd \gamma(x)$ for the inner product, and use $\|\cdot\|_\gamma$ to denote the associated norm. 
 Given two Banach spaces $\cA$ and $\cB$  with $\cA\subset \cB$, we denote by $\|\cA\hookrightarrow \cB\|:=\sup_{\|x\|_{\cA}=1}\|x\|_{\cB}$ the norm of the inclusion operator. Let $H_1,H_2$ be two Hilbert spaces, and $v\in H_1, u\in H_2$. The operator $u\otimes v: H_1\mapsto H_2$ is defined as $(u\otimes v) x = \langle v,x\rangle_{H_1} u$ for $x\in H_1$. We use $\cI$ to denote a general identity operator without specifying the domain for brevity and let $I_d$ be the $d\times d$ identity matrix. For a bounded linear operator $A$  over a Hilbert space and $r\in\RR$, we use $(A+\lambda)^{r}$ as a shorthand for $(A+\lambda \cI)^{r}$. 

We use big-O notations like $O(\cdot), o(\cdot),\Omega(\cdot)$ and $\Theta(\cdot)$ to hide constants and  $\tilde{O}(\cdot)$ to further hide logarithmic factors.
We also adopt the notations $a \lesssim b$, $a\gtrsim b$, and $a\asymp b$ to denote $a  =  O(b), a =\Omega(b)$, and $a = \Theta(b)$, respectively.

\vspace*{-.5em}
\section{Preliminaries}
\label{sec: preliminary}
\vspace*{-.1em}

Let  $\cX$ be a measurable input space, 
and $\rho\in\cP(\cX)$ denote the input distribution of interest. We aim to learn an unknown target function  $f^*:\cX\mapsto \RR$ from a finite set of training samples $\{(x_i,y_i=f^*(x_i)+\xi_i)\}_{i=1}^n$. Throughout most of our discussion, we focus on the noiseless setting, where $\xi_i=0$ for all $i\in [n]$ except in Section~\ref{sub:implication_for_noisy_krr} and Appendix~\ref{sec:the_krr_solution}, where we explicitly address the noisy case.
Suppose that $x_1,\dots,x_n$ are \iid samples drawn from a potentially different distribution $\rho'$, which satisfies the following condition:
\begin{assumption}
 $\rho'$ is absolutely continuous with respect to $\rho$. Let $q=\frac{\dd \rho'}{\dd \rho}$ be the corresponding density (Radon-Nikodym derivative). Assume that
$
 q(x) > 0
$
holds $\rho$-almost surely.
\end{assumption}

\paragraph*{Kernel and RKHS.} A  function $k:\cX\times\cX \mapsto \RR$ is said to be  a (positive semidefinite) kernel if 1) it is symmetric, i.e., $k(x,x') = k(x',x), \forall x,x' \in \cX$; and 2)  $\forall n \ge 1$, $x_1,\dots,x_n \in \cX$, the kernel matrix $(k(x_i,x_j))_{i,j}\in\RR^{n\times n}$ is positive semidefinite. Given a kernel $k$, there exists a unique Hilbert space $\cH$ such that $k(x,\cdot)\in \cH, \forall x\in \cX$; and $\langle f,k(x,\cdot)\rangle_\cH = f(x), \forall f\in\cH,x\in\cX$.
The second property is known as the reproducing property and thus,  $\cH$ is often referred to as  RKHS~\cite{aronszajn1950theory}.    RKHS  can also be characterized using the associated integral operator $\cL: L^2(\cX,\rho)\mapsto L^2(\cX,\rho)$, given by
\begin{equation}\label{eqn: 2-1}
    \cL f = \int_\cX k(\cdot,x') f(x') \dd \rho(x').
\end{equation}
Throughout this paper, we assume 
$\int k(x,x)\dd\rho(x)<\infty$, which guarantees that  $\cL$ is a self-adjoint, positive, and compact operator; see \cite{steinwart2008support}*{Theorem 4.27}. By the spectral theory of compact operator, we have the following  decomposition:
\begin{equation}\label{eqn: mercer's decomposition}
\cL = \sum_{j=1}^\infty \mu_j e_j\otimes e_j,
\end{equation}
 where $\{\mu_j\}_{j=1}^\infty$ are the eigenvalues in a decreasing order and $\{e_j\}_{j=1}^\infty$ are the corresponding orthonormal eigenfunctions.
 Without loss of generality, we assume that $\{e_j\}_{j=1}^\infty$ forms a complete basis of $L^2(\cX,\rho)$, otherwise we can always focus on the subspace spanned by the eigenfunctions associated with nonzero eigenvalues.  

\paragraph*{The Sobolev-type interpolation space.}
  For any $s\in\RR$, we define $\cL^s: L^2(\cX,\rho)\mapsto L^2(\cX,\rho)$ as $\cL^s f=\sum_{j=1}^\infty \mu_j^s e_j\otimes e_j$.
 Consider the following interpolation space:
\begin{definition}[Interpolation space]\label{def: Hks-space}
For any $s\in \RR$ and $f\in L^2(\cX,\rho)$, define
$
\|f\|_{\cH^s}=\|\cL^{-\frac{s}{2}}f\|_\rho.
$  
\end{definition}
By the spectral decomposition~\eqref{eqn: mercer's decomposition},  $\cH^s$ can  be  written in the following form
\begin{align}\label{eqn: Hks-space2}
\cH^s=\left\{\sum_{j=1}^\infty a_j\mu_j^{s/2}e_j: \sum_{j=1}^\infty a_j^2<\infty\right\},\quad \left\|\sum_{j=1}^\infty a_j\mu_j^{s/2}e_j\right\|_{\cH^s}^2=\sum_{j=1}^\infty a_j^2.
\end{align}
Note that 
$\cH^1=\cH$, $\cH^0=L^2(\cX,\rho)$, and $\cH^s$  provides a natural interpolation between them, and beyond (i.e., $s>1$ or $s<0$).   For any $s \ge 0$, $\cH^s\subset L^2(\cX,\rho)$ is a well-defined separable Hilbert space. 
When $s<0$, $\cH^s$ should  be viewed as the space of the bounded linear functional on $\cH^{-s}$. For the relationship between this interpolation space and classical Sobolev spaces, we refer the reader to \cite{steinwart2012mercer}.

Throughout this paper, we make the following assumption on the target function:
\begin{assumption}[Source condition]
$f^*\in \cH^s$ with $s\geq 0$.
\end{assumption}
The index $s$ quantifies the smoothness  of target function; larger values of $s$ correspond to  smoother target functions. This standard source condition has been widely used in the analysis of  kernel-based methods~\cites{caponnetto2007optimal,steinwart2012mercer,fischer2020sobolev,bach2017equivalence}, particularly for studying how the  smoothness of target functions affects the learnability. In the literature, the case of $s\geq 1$ is often referred to as the {\em well-specified case}, while $s<1$ is called the {\em misspecified case} or hard learning scenario. This distinction arises  due to when $s<1$, the target function $f^*\not\in \cH$, making the analysis of the corresponding learning problem more challenging.

For further discussion on kernels and RKHSs, we refer readers to \cite{wainwright2019high}*{Section 12} and \cite{steinwart2008support}*{Section 4}. A detailed treatment of interpolation spaces can be found in \cite{steinwart2012mercer}. We provide a concise overview of bounded linear operators on Hilbert spaces in Appendix~\ref{sec: bounded-operator}.

\paragraph*{The embedding property.} To obtain optimal rates for KRR in the misspecified setting, prior work invokes an additional \emph{embedding property} of the RKHS. Specifically, we say that $\cH$ has the embedding property of order $\alpha$, if
\begin{equation}\label{eqn: emb}
	A_\alpha :=\|\cH^\alpha\hookrightarrow L^\infty(\cX,\rho)\|\stackrel{(i)}{=}\sup_{x\in\cX}\sum_{j} \mu_j^\alpha e_j(x)^2 <\infty,
\end{equation}
where $(i)$ is due to \cite{fischer2020sobolev}*{Theorem 9}.
Noting $A_1=\sup_{x\in\cX}\sum_{j} \mu_j e_j(x)^2=\sup_{x}k(x,x)$, so the embedding property automatically holds for any $\alpha\geq 1$ as long as $\sup_x k(x,x)<\infty$. Additionally, the embedding index $\alpha$ must be large enough such that $(\mu_j)_{j\geq 1}\in \ell^\alpha$, as 
\[
\sum_j \mu_j^\alpha =\int \sum_j\mu_j^\alpha e_j(x)^2 \dd \rho(x) \leq \sup_{x\in\cX}\sum_{j} \mu_j^\alpha e_j(x)^2=A_\alpha<\infty.
\] 
If the  eigenvalues follow a power-law decay $\mu_j\asymp j^{-\beta}$ with $t>1$, then we must have $\alpha>1/\beta$. If, in addition, the eigenfunctions are uniformly bounded ($\sup_{j\in\NN^+}\|e_j\|_{L^\infty(\rho)}^2<\infty$), then the embedding index $\alpha$ can approach $1/\beta$, but the endpoint case $\alpha=1/\beta$ is not admissible. For more discussions, we refer the reader to \cite{fischer2020sobolev}.

\paragraph*{Weighted KRR.}
In this paper, we consider the following weighted KRR:
\begin{equation}\label{eqn: krr-def}
    \hat{f}_\lambda =\argmin_{f \in \cH} \left[\frac{1}{n}\sum_{i=1}^n\frac{1}{q(x_i)} (f(x_i) - y_i)^2 + \lambda \|f\|_{\cH}^2\right].
\end{equation}
When $\lambda\to 0^{+}$, $\hf_\lambda$ converges to the  minimum-norm solution:
$
 \hf_0  = \argmin_{f \in \cH, f(x_i) = y_i, 1 \le i \le n}\|f\|_{\cH}
 $
(see Lemma \ref{lemma: minimum-norm}).
Note that the weighted empirical risk is an unbiased estimate of the true risk  and standard KRR is recovered by setting $q(x)\equiv 1$. The idea of weighted sampling is inspired by \cite{bach2017equivalence}, which examines the approximation power of weighted random features. As we will elaborate in  Section~\ref{sec:the_modified_degree_of_freedom}, weighted sampling allows for a unified analysis of KRR.

By the  representer theorem \cite{scholkopf2001generalized}, $\hf_\lambda$ and $\hf_0$ must lie in $\mathrm{span}\{k(x_1,\cdot),\dots,k(x_n,\cdot)\}$. Let $\hk_n:\cX\mapsto \RR^n$ with $(\hat{k}_n(\cdot))_i = \frac{1}{\sqrt{n q(x_i)}}k(x_i,\cdot)$,
$\hy\in\RR^n$ with $\hy_i = \frac{y_i}{\sqrt{n q(x_i)}}$, and $K_n=\left(\frac{1}{n\sqrt{q(x_i)q(x_j)}}k(x_i,x_j)\right)_{1\leq i,j\leq n}\in\RR^{n\times n}$ be the kernel matrix. Then,
the KRR solution can be expressed as:  
$\hf_\lambda  = \hk_n(\cdot)^\top \hat{b}_\lambda$ for $\lambda>0$ 
and $\hf_0=\hk_n(\cdot)^\top \hat{b}_0$
with 
\begin{equation}\label{eqn: krr-a}
\hat{b}_\lambda = (K_n + \lambda I_n )^{-1} \hy,\quad \hat{b}_0 = (K_n)^\dagger \hy, 
\end{equation}
where $(\cdot)^\dagger$ denotes the Moore-Penrose inverse of a matrix \cite{ben2003generalized}.

\section{The Refined Degrees of Freedom}
\label{sec:the_modified_degree_of_freedom}

To study the sample complexity of learning within a RKHS, the commonly used notion of complexity is the degree of freedom (DoF), defined as $N_1(\lambda)=\tr[\cL(\cL+\lambda)^{-1}]=\sum_{j=1}^\infty \frac{\mu_j}{\mu_j+\lambda}$ (see, e.g., \cites{zhang2005learning,caponnetto2007optimal,blanchard2018optimal,lin2020optimal}). Let $m(\lambda)=\#\{j: \mu_j\geq \lambda\}$. Then, we have
\begin{equation}\label{eqn: average-dof-1}
\begin{aligned}
	N_1(\lambda)&\geq \sum_{\mu_j\geq \lambda}\frac{\mu_j}{\mu_j+\lambda}\geq \half m(\lambda)\\ 
	N_1(\lambda)&\leq \sum_{\mu_j\geq \lambda} 1 + \sum_{\mu_j< \lambda}\frac{\mu_j}{2\lambda} = m(\lambda) + \frac{1}{2\lambda}\sum_{j=1}^\infty \mu_j.
\end{aligned}
\end{equation}
Since $\sum_{j=1}^\infty \mu_j=\int k(x,x)\dd\rho(x)<\infty$ and $m(\lambda)$ decays typically slower than $\lambda^{-1}$, the DoF $N_1(\lambda)$ effectively counts how many eigenvalues exceed the threshold $\lambda$. 
In addition to $N_1$,  establishing high-probability bounds  often involves to control the following max DoF:  
\begin{equation}\label{eqn: max-dof-original}
	F_1(\lambda) := \esssup_{x\in\cX}\big\|(\cL+\lambda )^{-\frac{1}{2}}k(x,\cdot)\big\|_\cH^2 = \esssup_{x\in \cX}\sum_{j=1}^\infty \frac{\mu_j}{\mu_j+\lambda} |e_j(x)|^2,
\end{equation}
where the second step follows from the spectral decomposition~\eqref{eqn: mercer's decomposition}. Controlling $F_1(\lambda)$  plays a critical role in obtaining sharp rates for kernel methods~\cite{bach2017equivalence,fischer2020sobolev,zhang2023optimality,li2024towards}.

Our main technical contribution is the introduction of the following refined DoFs:
\begin{definition}[$\gamma$-DoFs]\label{def: degree-Fp}
Suppose $\gamma>0$ satisfies  $\sum_{j=1}^\infty \mu_j^\gamma < \infty$. Define
\begin{align*}
N_\gamma(\lambda)&=\tr[(\cL+\lambda )^{-\gamma}\cL^\gamma]\\
F_\gamma(\lambda) &= \esssup_{x \in \cX}\frac{1}{q(x)}\big\|\cL^{\frac{\gamma-1}{2}}(\cL+\lambda )^{-\frac{\gamma}{2}}k(x,\cdot)\big\|_\cH^2,
\end{align*}
where the weighting factor $q(x)$ is introduced to accommodate the weighted KRR~\eqref{eqn: krr-def}
\end{definition}
By setting $\gamma=1, q(x)\equiv 1$, we recover the classical DoFs $N_1(\cdot)$ and $F_1(\cdot)$. It is evident that  if $\gamma_2\geq \gamma_1>0$, then $\cF_{\gamma_2}(\lambda)\leq \cF_{\gamma_1}(\lambda)$ for any $\lambda>0$. 
We will demonstrate that  
 $\gamma$-DoFs not only  provide a  sharper characterization of noiseless KRR in the misspecified setting but also allow us to control over generalization error in the $\cH^p$ norm for $p>1$. Specifically,  bounding $\|\hf_\lambda-f^*\|_{\cH^p}$ with $p>1$ depends on the $\gamma$-DoFs with $\gamma=2-p$; see Theorem~\ref{thm:krr_error} and \ref{thm: noisy-krr}. 

\begin{remark}  
The $\gamma$-DoFs are motivated by the observation that for any $x\in \cX$, the kernel function \( k(x, \cdot) \)  can be smoother than typical functions in \( \mathcal{H} \). Specifically, \( k(x, \cdot) \in \mathcal{H}^p \) for any \( p \) such that \( \sum_j \mu_j^{2 - p} < \infty \), provided the eigenfunctions satisfy certain benign conditions. Hence, when $\mu_j$ decays faster than $O(j^{-1})$, $p$ can exceed $1$, indicating extra smoothness. This benign property of kernel functions is elaborated in Section~\ref{sub:the_extra_smoothness_effect}.
Our refined DoFs are  designed to leverage this extra smoothness, enabling tighter error bounds and improved analysis. 
\end{remark}

Applying the spectral decomposition \eqref{eqn: mercer's decomposition} yields 
\begin{equation}\label{eqn: F-def}
     N_\gamma(\lambda)=\sum_{j=1}^\infty \left(\frac{\mu_j}{\mu_j+\lambda}\right)^\gamma,\quad F_\gamma(\lambda) 
     = \esssup_{x \in \cX}\frac{1}{q(x)}\sum_{j=1}^{\infty}\left(\frac{\mu_j}{\mu_j+\lambda}\right)^\gamma|e_j(x)|^2.
\end{equation}
We immediately have 
\begin{align}\label{eqn: F_bound} 
\notag    F_\gamma(\lambda)& \ge \int_{\cX}\frac{1}{q(x)}\sum_{j=1}^{\infty}\frac{\mu_j^\gamma}{(\mu_j+\lambda)^\gamma}|e_i(x)|^2\dd \rho'(x) 
     = \sum_{j=1}^\infty \frac{\mu_j^\gamma}{(\mu_j+\lambda)^\gamma} = N_\gamma(\lambda).
\end{align}
Since $N_\gamma(\lambda)\geq N_\gamma(\mu_1)\geq (2\mu_1)^{-1}\sum_j \mu_j^\gamma$ for $\lambda\in (0, \mu_1]$, the condition $\sum_{i=1}^\infty \mu_j^\gamma<\infty$ in Definition \ref{def: degree-Fp} is necessary to ensure  that $F_\gamma(\lambda)$ is finite. Additionally,   
the behavior of $N_\gamma(\cdot)$ can be explicitly characterized using the eigenvalue decay:
\begin{lemma}\label{lemma: dof}
If $\mu_j\asymp j^{-\beta}$ with $t>1$, then $N_\gamma(\lambda)\asymp \lambda^{-1/\beta}$ for any $\gamma>1/\beta$. If $\mu_j\asymp  c^{-j}$ with $c\in (0,1)$, then $N_\gamma(\lambda)\asymp  \log(1/\lambda)$ for any $\gamma>0$.
\end{lemma}
The proof parallels the argument in \eqref{eqn: average-dof-1} and is deferred to Appendix~\ref{sec: proof-dof}. Note that the case $\gamma = 1$ has been widely used in the prior analysis of kernel methods~\cite{smale2007learning,fischer2020sobolev,sriperumbudur2022approximate}.

\subsection{Controlling $F_\gamma(\lambda)$}
As we will detail in our proofs, obtaining a tight bound on \(F_\gamma(\lambda)\) is technically crucial for showing that KRR can achieve optimal rates. Specifically, if \(F_\gamma(\lambda)\) is on the same order as \(N_\gamma(\lambda)\) for sufficiently small \(\lambda\), the optimal rates can be realized. However, this step is challenging because \(F_\gamma(\lambda)\)  depends on the specific properties of the eigenfunctions; see Eq.~\eqref{eqn: F-def}. In the worst case -- when the eigenfunctions are arbitrary \(L^2\) functions -- \(F_\gamma(\cdot)\) may behave poorly. To address this issue, we can consider two approaches: 
\begin{itemize}
\item \textbf{Benign kernels.} Many practical kernels (and their associated eigenfunctions) may exhibit properties that ensure \(F_\gamma(\lambda)\) and \(N_\gamma(\lambda)\) are of the same order. What remains is to identify suitable, easily verifiable conditions that yield the necessary control over \(F_\gamma(\lambda)\). 

\item \textbf{Weighted sampling.} For more general kernels lacking such benign properties, we show that a carefully chosen weighted sampling scheme can restore the desired control on \(F_\gamma(\lambda)\).
\end{itemize}

\paragraph*{Uniform boundedness of eigenfunctions.} A straightforward scenario arises when there is a constant $C$ such that $\sup_{j\in\NN}\|e_j\|_{L^\infty(\rho)}\leq C$. In the unweighted case where $\rho'=\rho$, it follows from \eqref{eqn: F-def}
 that $F_\gamma(\lambda)\leq C^2 N_\gamma(\lambda)$ for all $\lambda>0$ .
A concrete example is the \emph{periodic kernel}
$k(x,x')=\kappa(x-x')$ on the  torus $\TT^d:=[0,1)^d$, for which the eigenfunctions are the Fourier basis functions $x\mapsto e^{2\pi \ii j x}$ with $j\in \ZZ$, each bounded by $1$.

Despite its simplicity,  many widely used kernels do not admit uniformly bounded eigenfunctions. A crucial example is the dot-product kernels on the unit sphere, whose
eigenfunctions  are spherical harmonics~\cite{dot-product-kernel}; these functions are not uniformly bounded~\cite{dong2023toward}. Consequently, other conditions are needed to handle such kernels.

\paragraph*{The embedding property.}
Following \cite{fischer2020sobolev}, one can control $F_\gamma(\lambda)$ by utilizing the embedding property~\eqref{eqn: emb}. Indeed, 
\[
	F_\gamma(\lambda) = \sum_{j=1}^\infty \frac{\mu_j^{\gamma-\alpha}}{(\mu_j+\lambda)^\gamma} \mu_j^{\alpha}e_j(x)^2\leq A_\alpha \sup_{j\in \NN^+} \frac{\mu_j^{\gamma-\alpha}}{(\mu_j+\lambda)^\gamma} \leq A_\alpha \lambda^{-\alpha},
\]
where the second step is due to \eqref{eqn: emb} and the last step follows from  Lemma~\ref{lemma: f_lambda}. 
To assess if this control of $F_\gamma(\lambda)$ can yield optimal rates for KRR, we consider the power-decay case $\mu_j\asymp j^{-\beta}$. In such case, Lemma~\ref{lemma: dof} shows that $N_\gamma(\lambda)\asymp \lambda^{-1/\beta}$ and consequently, obtaining optimal rates requires the embedding index $\alpha=1/\beta$, for which $F_\gamma(\lambda)\leq A_\alpha \lambda^{-1/\beta}\asymp N_\gamma(\lambda)$. However,  it is unclear whether this requirement  holds for  commonly-used kernels.  Even for kernels with uniformly bounded eigenfunctions—where \(\alpha\) can be chosen arbitrarily close to \(1/\beta\)—the embedding property may still fail to hold exactly at \(\alpha = 1/\beta\). This is because \(\sum_j \mu_j^{1/\beta} \asymp \sum_j (j^{-\beta})^{1/\beta} = \infty\). As a result, applying the embedding property yields only a \emph{nearly} optimal rate of \(O(n^{-t + \delta})\) for any \(\delta \in (0,1)\), rather than the exact optimal rate \(O(n^{-t})\).

\paragraph*{Symmetric kernels.}  In this paper, 
we instead propose an entirely different condition that can overcome the limitation of embedding technique discussed above.  Specifically, we show that when the kernel possesses  certain symmetries, the desired control of $F_\gamma(\lambda)$ arises naturally. 
Let  $k:\cX\times \cX\mapsto\RR$, $\rho\in \cP(\cX)$, and $\cG$ be a group of transforms  acting on $\cX$. We say that $(k, \cX, \rho)$ is $\cG$-invariant if the following holds:
    \begin{itemize}
        \item {\it Transitivity:}  For any $x,x' \in \cX$, there exists a  $A\in\cG$ such that $Ax = x'$.
         \item {\it Measure preservation:} For any measurable set $E$ and $A\in\cG$, $\rho(\{x: Ax \in E\}) = \rho(E)$.
        \item {\it Kernel invariance:} For any $x,x' \in \cX$ and $A \in \cG$, $k(Ax,Ax') = k(x,x')$.
    \end{itemize}

\begin{lemma}[]\label{lem:symmetry}
Given a kernel $k:\cX\times \cX\mapsto\RR$, if there exist a group $\cG$  such that
$(k,\cX,\rho)$ is $\cG$-invariant, 
    then $F_\gamma(\lambda) = N_\gamma(\lambda)$ for all $\lambda>0$ when $q(x) \equiv 1$.
\end{lemma}
The proof relies on the observation that  $x\mapsto \|\cL^{\frac{\gamma-1}{2}}(\cL+\lambda )^{-\frac{\gamma}{2}}k(x,\cdot)\big\|_{\cH}$ is constant when $k$ is symmetric and the details can be found in Appendix~\ref{sec:proof_of_lemma_ref_lem_symmetry}. Examples of such kernels include  the popular dot-product kernels of the form $k(x,x'):=\kappa(x^\top x')$ with $\cX=\SS^{d-1}$ and $\kappa: [-1,1]\mapsto\RR$, equipped with $\rho=\mathrm{Unif}(\SS^{d-1})$. Obviously, dot-product kernels are invariant under rotation transformations. Another important example is translation-invariant kernels $k(x,x)=\kappa(x-x')$ with $\cX=\TT^d$ and $\rho=\mathrm{Unif}(\cX)$. These  kernels are invariant under translations, making them symmetric as well.

\paragraph*{Weighted sampling.}
For the general case,  one can achieve the equality $F_\gamma(\lambda)=N_\gamma(\lambda)$  through a suitable weighted sampling, without making additional assumption on the kernels. 
\begin{lemma}
For any $\lambda> 0$, let $\nu_\lambda(x)=\sum_{j=1}^{\infty}\frac{\mu_j^\gamma}{(\mu_j+\lambda)^\gamma}|e_j(x)|^2/N_{\gamma}(\lambda)$. Then, taking $q=\nu_\lambda$ yields $F_\gamma(\lambda)=N_\gamma(\lambda)$.
\end{lemma}
\begin{proof}
Let $q_\gamma(x,\lambda)=\sum_{j=1}^{\infty}\frac{\mu_j^\gamma}{(\mu_j+\lambda)^\gamma}|e_j(x)|^2$. Noting $\int_{\cX} q_\gamma(x,\lambda)\dd\rho(x)=N_\gamma(\lambda)$,  $\nu_\lambda=q(\cdot,\lambda)/\cN_\gamma(\lambda)$ is a probability distribution. By definition, $F_\gamma(\lambda)=q_\gamma(x,\lambda)/q(x)=\cN_\gamma(\lambda)$.
\end{proof}
We remark that a similar but different weighted sampling approach was also employed in \cite{krieg2021function,krieg2021function2} to obtain the optimal $L^2$-sampling numbers for RKHS.

\section{Main Results}
\subsection{Lower Bounds}
In this section, we first establish information-theoretical lower bounds for learning functions within $\cH^s$.
Let $S_n=(x_1,x_2,\dots,x_n)^\top\in\cX^n$  and $f(S_n)=(f(x_1),f(x_2),\dots, f(x_n))^\top\in\RR^n$. We define the space of estimators that use only the finite clean samples $(S_n, f(S_n))$ as follows
\begin{equation*}
	\cM_n = \left\{\cQ\,| \,\cQ f = G(S_n, f(S_n)), \, G:\cX^{n}\times\RR^n\mapsto \RR^{\cX} \right\}.
\end{equation*}
We have the following lower bound for learning $\cH^s$ functions:
\begin{theorem}\label{thm:lower_bound}
For any $0\leq p \le s<\infty$, we have that
\begin{align}\label{eq:lower_bound_1}
    &\inf_{\cQ\in \cM_n}\sup_{\|f\|_{\cH^s} \le 1}\left\|\cQ f - f\right\|^2_{\cH^p}\ge \mu_{n+1}^{s-p}.
\end{align}
\end{theorem}
A similar lower bound has been established under the $L^2$ norm in works such as \cite{li2023asymptotic}, albeit for a modified interpolation space. In contrast, our result pertains to the standard interpolation space and is established under a general interpolation norm. The primary purpose of these lower bounds is to justify the optimality of our upper bounds presented in Section~\ref{sec:upper_bounds}.

\subsection{Upper Bounds}
\label{sec:upper_bounds}

To streamline the presentation  of our upper bounds, we introduce a concept of critical penality: 
\begin{definition}[Critical penalty]\label{def: critical-penality}
For any $n\in\NN^+$ and $\delta\in (0,1)$, let 
\[
\Lambda_\gamma(n,\delta):=\inf\left\{\lambda\in\RR_{\geq 0}\,:\, n \ge 5 F_\gamma(\lambda)\max\left(1,\log(14 F_\gamma(\lambda)/\delta)\right)\right\}.
\]
\end{definition}
This quantity  denotes the threshold beyond which $(\hT_n+\lambda)^{-1}$ effectively concentrates around $(\cT+\lambda)^{-1}$ in certain sense; see Theorem~\ref{thm: concentration} for details.

We first establish upper bounds for uniform learning over the entire function space. 
\begin{theorem}[Uniform learnability]\label{thm:krr_error}
Let  $\delta \in (0,1)$, $n\in \NN^+$, and $0\leq p\leq s\leq 2$.
\begin{enumerate}
    \item {\bf The well-specified case where $s\in [1,2]$:}
\begin{itemize}
\item If $p\in [0,1]$, then  for any $\lambda\in [0, \lambda_n]$ with $\lambda_n=\Lambda_1(n,\delta)$, it holds with probability at least $1-\delta$ that
\begin{equation*}
    \sup_{\|f^*\|_{\cH^s}\leq 1}\|\hat{f}_\lambda - f^*\|_{\cH^p}^2 \le 16 \lambda_n^{s-p}.
\end{equation*}
\item If $p>1$ and $\sum_{j=1}^\infty \mu_j^{2-p}<\infty$, then for any $\lambda\geq \lambda_n:=\Lambda_{2-p}(n,\delta)$, it holds with probability at least  $1-\delta$ that
\begin{equation*}
    \sup_{\|f^*\|_{\cH^s}\leq 1}\|\hat{f}_{\lambda} - f^*\|^2_{\cH^p}\le 16 \lambda^{s-p}.
\end{equation*}
\end{itemize}

\item {\bf The misspecified case where $s<1$ and $\sum_{j}\mu_j^s<\infty$:}  For any $\lambda\geq\lambda_n:= \Lambda_{s}(n,\delta)$, it holds  with probability at least $1-\delta$ that
\begin{equation*}
    \sup_{\|f^*\|_{\cH^s}\leq 1}\|\hat{f}_{\lambda} - f^*\|_{\cH^p}^2 \le 16 \lambda^{s-p}.
\end{equation*}
\end{enumerate}
\end{theorem}

For the misspecified case, the following lemma  demonstrates that the condition $\sum_{j}\mu_j^{s}<\infty$ is necessary for obtaining uniform learnability, whose proof can be found in Appendix~\ref{sec:proof_of_lemma_ref_lemma_uniform_necessity}.
\begin{lemma}\label{lemma: uniform-necessity}
Let $k$ be a  kernel satisfying the condition in Lemma~\ref{lem:symmetry} and $\rho'=\rho$. If  $\sup_{\|f^*\|_{\cH^s}\leq 1}\|\hat{f}_{\lambda} - f^*\|_{\cH^p}<\infty$ holds almost surely for some $p\in (0,s)$, then $\sum_j \mu_j^s<\infty$.
\end{lemma}

The next theorem further shows that when $\sum_{j=1}^\infty \mu_j^s =\infty$, although  a uniform learning within  $\cH^s$ is not possible, KRR can still achieves  optimal learning for each  single target function.

\begin{theorem}[Single-target learnability]\label{thm:krr_error-one}
    Let $n \in \NN^+$, $\delta_1,\delta_2 \in (0,1)$, and $0\leq p\leq s\leq 1$. Pick any $\gamma\in (s,1]$ such that
    $\sum_j \mu_j^\gamma<\infty$ and consider KRR with  $\lambda\geq \lambda_n:=\Lambda_{\gamma}(n,\delta_1)$. 
    Then for a given $f^*$ with $\|f^*\|_{\cH^s} \le 1$, it holds with probability at least $1-\delta_1-\delta_2$ that
    \begin{equation}\label{eq:f_upper_2}
        \|\hat{f}_{\lambda} - f^*\|_{\cH^p}^2 \le \left(5 + C_{s,\gamma}\delta_2^{-\frac{\gamma-s}{2\gamma}}\right)^2 \lambda^{s-p},
    \end{equation}
    where $C_{s,\gamma}$ is a positive constant only depending on $s$ and $\gamma$.
\end{theorem}
When $s$ approaches $0$, the dependence on $\delta_2$ deteriorates to $\delta_2^{-1/2}$. If $s$ nears $s_0=\inf\{s: \sum_j\mu_j^s<\infty\}$, one can choose $\gamma$ close to $s_0$ such that the dependence becomes $\delta_2^{-\tau}$ with rather small $\tau$.  
This fine-grained characterization of $\delta_2$-dependence  is achieved by leveraging the fact that \( \cH \)  embeds continuously into \( L^q \) for some \( q \) significantly larger than \( 2 \); for details, we refer to Appendix~\ref{sec:the_l_q_embedding_of_ch_s_}.

\paragraph*{Comparison with existing results on noiseless KRR.}
Prior analyses of noiseless KRR, such as \cite{barzilai2023generalization, li2023asymptotic}, focused on the convergence of $L^2$-error.  In contrast, we establish the optimal rates  under the general  $\cH^p$ norm and, in particular,  precisely identify the range of $p$ that can be used to assess KRR solutions.  
Specializing to the \( L^2 \)-error, our results also improve upon those  by refining the \( \delta_2 \)-dependence from \( \delta_2^{-1/2} \) to \( \delta_2^{-(\gamma-s)/(2\gamma)} \), which is especially significant  when \( s \) is close to $s_0$. 
Moreover,  when $s$ is large enough to ensure \( \sum_j \mu_j^s < \infty \),  Theorem~\ref{thm:krr_error} shows that the dependence on $\delta_2$ can be further improved to \( \log(1/\delta_2) \), marking a substantial improvement over \cite{barzilai2023generalization,li2023asymptotic}. Finally, \cite{li2023asymptotic} employed  a modified source condition, leading to  less tight bounds when applied to the classical source condition. In contrast, our analysis directly addresses the standard source condition, yielding more accurate error estimates.

\subsection{Improved Results for Noisy KRR}
\label{sub:implication_for_noisy_krr}

The above results for noiseless KRR suggest that two aspects of  existing optimal rates for noisy KRR may be ``suboptimal'':
(1) the established rates are optimal only under a fixed  noise level and fail to  recover the noiseless rate as the noise magnitude diminishes to zero; 
(2) the interpolation norm used to measure learning error is restricted to  $p \in [0,1]$, whereas our noiseless analysis indicates that the valid range of $p$ can be significantly larger. These motivate us to adapt techniques from the noiseless setting to derive sharper bounds for noisy KRR.

\begin{assumption}\label{assumption: noise}
Assume that the noise $\{\xi_i\}_{i=1}^n$ are \iid, mean-zero, $\sigma$-sub-Gaussian (i.e., $\EE[e^{\xi_i^2/\sigma^2}]\leq 2$ for all $i\in [n]$), and are independent of the inputs $\{x_i\}_{i=1}^n$.
\end{assumption}

\begin{assumption}\label{assumption: weighted sampling}
	Assume $q(x)\geq 1/2$ almost surely.
\end{assumption}
This assumption essentially requires that the density of the weighted sampling distribution is bounded from below. The specific threshold of \(1/2\) is chosen purely for simplicity. In addition, we argue that such a weighted sampling distribution always exists. For any density \( q \) with the associated maximal \(\gamma\)-DoF denoted by \( F_\gamma(\lambda) \), consider the modified distribution \( q'(x) := \frac{q(x) + 1}{2} \). Let \( F'_\gamma(\lambda) \) denote the maximal \(\gamma\)-DoF under \( q' \). Then, we have
\[
F'_\gamma(\lambda) = \esssup_{x \in \cX} \frac{1}{q'(x)} \sum_{j=1}^\infty \left( \frac{\mu_j}{\mu_j + \lambda} \right)^\gamma e_j(x)^2 
\leq \esssup_{x \in \cX} \frac{2}{q(x)} \sum_{j=1}^\infty \left( \frac{\mu_j}{\mu_j + \lambda} \right)^\gamma e_j(x)^2 = 2F_\gamma(\lambda).
\]
Thus, replacing \( q \) with \( q' \) preserves the maximal \(\gamma\)-DoF up to a constant factor, while ensuring that Assumption~\ref{assumption: weighted sampling} is satisfied.

\begin{theorem}\label{thm: noisy-krr}
  Suppose Assumptions~\ref{assumption: noise} and~\ref{assumption: weighted sampling} hold. Let \(s \in [1,2]\) and \(0 \le p < s\) be such that \(\sum_{j=1}^\infty \mu_j^{2-p} < \infty\). For any \(\delta \in (0,1/2)\), if \(\lambda \ge \lambda_n := \Lambda_{\min(2-p,1)}(n,\delta)\), then \wp at least \(1 - \delta\), we have
  \[
    \|\hat{f}_\lambda - f^*\|_{\cH^p}^2 \,\lesssim\, \bigl(1 + \log\bigl(\tfrac{1}{\delta}\bigr)\bigr)
    \left(\,\lambda^{\,s - p} \;+\; \frac{\sigma^2 F_{2-p}(\lambda)}{n\,\lambda^p}\right).
  \]
\end{theorem}
The proof of this theorem is deferred to Appendix~\ref{sec:proof_of_theorem_ref_thm_noisy_krr}. Specifically,
 we provide improved bounds on the learning error of noisy KRR for \(s \in [1,2]\) by combining a new \(\cH^p\) bound  on the variance term (see Proposition~\ref{pro:variance_estiamte}) with the well-specified result of Theorem~\ref{thm:krr_error}. For the misspecified case, similar bounds can be obtained by combining our bound of variance term with the misspecified result in Theorems~\ref{thm:krr_error} or~\ref{thm:krr_error-one}. 

Compared to existing upper bounds for noisy KRR, Theorem~\ref{thm: noisy-krr} is \textit{noise-level aware} and holds for general \(\cH^p\) norms, allowing \(p > 1\). Notably, as \(\sigma \to 0\), our bound recovers the optimal noiseless rate. In the regime where \(\sigma \asymp 1\) and the eigenvalues satisfy \(\mu_j \asymp j^{-\beta}\), if the weighted sampling distribution is chosen such that \(F_{2-p}(\lambda) \leq C N(\lambda)\), the optimal rate is given by
\begin{align}\label{eqn: noisy-rate}
  \inf_{\lambda \geq \lambda_n} \left( \lambda^{s - p} + \frac{F_{2 - p}(\lambda)}{n \lambda^p} \right) 
  = \inf_{\lambda \geq \lambda_n} \left( \lambda^{s - p} + \frac{\lambda^{-1/\beta}}{n \lambda^p} \right) 
  \asymp \left( \frac{1}{n} \right)^{\frac{\beta(s - p)}{s\beta + 1}},
\end{align}
where the second step follows from Lemma~\ref{lemma: dof}. This matches the rates in \cite{fischer2020sobolev} while extending the valid range of \(p\) beyond \(1\). Moreover, the optimal choice of regularization parameter is \(\lambda_{\mathrm{op}} \asymp n^{-s / (s\beta + 1)}\), which notably does not depend on \(p\).

\subsection{The Extra-Smoothness Effect}
\label{sub:the_extra_smoothness_effect} 

In Theorems~\ref{thm:krr_error}, \ref{thm:krr_error-one}, and \ref{thm: noisy-krr}, we show that the parameter \( p \) can be chosen arbitrarily close to \( p_0 = \inf\{p : \sum_j \mu_j^{2-p} < \infty\} \). Therefore,  \( p \) can exceed $1$ whenever the eigenvalues \( \mu_j \) decay faster than \( j^{-1} \).
This finding is surprising because KRR  seeks solutions within \( \cH \), where functions typically  only have smoothness of order $1$. Indeed, all previous interpolation-norm bounds, such as those in \cite{fischer2020sobolev,li2024towards} assume \( p \in [0,1] \). Technically, our improvement arises from leveraging the newly introduced $\gamma$-DoFs defined in Definition~\ref{def: degree-Fp}. To shed further light on this unexpected phenomenon, we provide an intuitive explanation below. 

The ability to choose \( p > 1 \) originates from the inherent properties of the kernel function. Specifically, for almost every \( x \in \cX \), the kernel function \( k(x, \cdot) \) in fact resides in the smoother space \( \cH^p \), whenever \( \sum_j \mu_j^{2-p} < \infty \). To see why, consider the norm of \( k(x, \cdot) \) in \( \cH^p \):
\[
\|k(x, \cdot)\|_{\cH^p}^2 = \big\| \sum_{j} \mu_j e_j(x) e_j \big\|_{\cH^p}^2 = \sum_j \mu_j^{2-p} e_j^2(x).
\]
Assuming the eigenfunctions are uniformly bounded, i.e., \( \sup_{j\in \NN^{+}}\|e_j\|_{L^\infty(\rho)} \leq C \) for some constant \( C \), the norm can be bounded as follows:
\[
\|k(x, \cdot)\|_{\cH^p}^2 \leq C^2 \sum_j \mu_j^{2-p}.
\]
Hence, it follows from $\sum_j\mu_j^{2-p}<\infty$  that $k(x, \cdot)\in \cH^p$. Furthermore, the representer theorem  ensures \( \hat{f}_\lambda\in \spn\{ k(x_i, \cdot) : i \in [n] \} \). Because each \( k(x_i, \cdot) \in \cH^p \), we deduce that \( \hat{f}_\lambda \in \cH^p \).

In essence, the extra smoothness of \( \hat{f}_\lambda \) arises from two factors. First, the joint properties of the kernel’s eigenfunctions and eigenvalues ensure that \( k(x, \cdot) \) inhabits a smoother space. Second, the representer theorem imposes a structural constraint, confining the solution in the span of kernel evaluations at the training points. The synergy of these two factors enables KRR to produce solutions whose smoothness surpasses that of the native RKHS. 

\section{The Saturation Effects}
\label{sec: saturation}
\subsection{The Smoothness of KRR Solution}
We start by  examining the saturation in $p$, which reflects the inherent smoothness of the KRR solution. Previously, we  assume  $p$ to be sufficiently small to ensure the summability condition $ \sum_{j=1}^\infty \mu_j^{2-p} < \infty$. We now show that this condition is, in fact,  necessary.

\begin{theorem}\label{thm: saturation-1}
    Assume that $\rho' = \rho$ and $\sum_{j=1}^\infty \mu_j^{2-p} = \infty$. Suppose
    \begin{equation}\label{eqn: condition-eigenfunction}
    \limsup_{m\rightarrow\infty} \inf_{x \in \cX}(\sum_{j=1}^{m}\mu_j^{2-p}|e_j(x)|^2)/(\sum_{j = 1}^{m}\mu_j^{2-p})>0.
   \end{equation}
    Then, for any $f^* \in \cH^s$, with probability $1$, $\|\hat{f}_\lambda-f^*\|_{\cH^p} = \infty$ unless $\hat{f}_\lambda = 0$.
\end{theorem}
Let us elaborate a little on what this theorem is saying. 
Consider kernels with a  spectrum that follows a   power-law decay, $\mu_j\asymp j^{-\beta}$. In this case, the series $\sum_{j}\mu_j^{2-p}$ converges if and only if $t(2-p)>1$, which yields the constraint $p<2-1/\beta$. Consequently, Theorem~\ref{thm: saturation-1} indicates that the corresponding KRR solution can have at most $(2-1/\beta)$-order smoothness, even when the underlying data-generating function is significantly smoother, i.e., $s\geq 2$. 

The condition \eqref{eqn: condition-eigenfunction} is satisfied for many popular kernels, including  dot-product kernels on the hypersphere, periodic kernels on the hypercube, and kernels on the Hamming cube $\{-1,1\}^d$. For a detailed verification of why these kernels satisfy the condition~\eqref{eqn: condition-eigenfunction}, we refer the reader to  Appendices G and H in \cite{barzilai2023generalization}.

When $p<0$, $\cH^p$ is a norm weaker than $L^2(\rho)$. This weaker norm is particularly useful in the analysis of partial differential equations (PDEs) and variational problems. For example, convergence in the $\cH^{-1}$ norm is often of interest in applications of analyzing  elliptic and parabolic problems.
However, in Theorems~\ref{thm:krr_error} and \ref{thm:krr_error-one}, we have assumed $p\geq 0$. One might intuitively expect that the convergence rates could be improved when measured in $H^p$ norms with $p<0$, analogous to the behavior observed for $p\geq 0$. Contrary to this intuition, we demonstrate that this is not the case.
By Lemma~\ref{lemma: uniform-operaotr-bound}, we have 
\begin{align*}
    \sup_{\|f^*\|_{\cH^s}\le 1}\|f^* - \hat{f}_\lambda\|_{\cH^p} &= \lambda \|\cT^{\frac{1-p}{2}}(\hT_n + \lambda)^{-1}\cT^{\frac{s-1}{2}}\| \\ 
    &= \lambda\|\cT^{\frac{s-1}{2}}(\hT_n + \lambda)^{-1}\cT^{\frac{1-p}{2}}\| = \sup_{\|f^*\|_{\cH^{2-p}} \le 1}\|f^* - \hat{f}_\lambda\|_{\cH^{2-s}}. 
\end{align*}
This means that the saturation in $ p < 0$ can be transferred into the saturation effect in $ s > 2$, as discussed later in Section~\ref{sec:the_saturation_of_smoothness_adaptivity}. However, it is worth noting that this uniform-type argument does not rule out the possibility that results analogous to  Theorem~\ref{thm:krr_error} may exist. To further investigate, we will conduct numerical experiments to demonstrate that the saturation effect for  $ p < 0 $ persists  even for a single target function.

\subsection{The Adaptivity to the Smoothness of Target Functions}
\label{sec:the_saturation_of_smoothness_adaptivity}
In Theorems~\ref{thm:krr_error} and \ref{thm:krr_error-one}, we have demonstrated that noiseless KRR can adapt to the target function's smoothness up to $s=2$. Here,  we further show that additional smoothness ($s>2$) does not yield further improvement, indicating that a saturation occurs at precisely $s=2$. We argue that  this phenomenon arises due to the constraints  imposed by the representer theorem.
Specifically, the solution of KRR  must lie in $\spn\{k(x_i,\cdot): i \in [n]\}$, implying that   the number of kernel functions can be leveraged by KRR is  at most $n$. This leads to  the following lower bound:
\begin{align}\label{eqn: app-lower-bound-saturation}
\|\hf_\lambda - f^*\|_{\cH^p} \geq \sup_{a\in\RR^n}\Big\|\sum_{i=1}^n a_i k(x_i,\cdot) -f^*\Big\|_{\cH^p}.
\end{align}
When the kernel  function $k(x,\cdot)$ is less smooth, it is not surprising that $\hf_\lambda$ is unable to exploit the higher-order smoothness of $f^*$.

To clarify how the above argument can explain why the saturation occurs precisely at $s=2$, we consider the specific Brownian kernel
\begin{equation}\label{eqn: min-kernel}
    k_B(x,x' ) = \min(x,x'), \quad x,x'\in [0,1].
\end{equation}
In this case, $\sum_{i=1}^n a_i k(x_i,\cdot)$ must lie in $\mathrm{CPWL}_n([0,1])$, the space of  continuous piecewise linear (CPWL) functions on $[0,1]$ with at most $n$ pieces. The approximation properties of CPWL models has
been well-studied~\cite{devore1993constructive,devore2021neural}. It is well-known that 
$
    \inf_{f\in \mathrm{CPWL}_m([0,1])}\|f-f^*\|_{L^2}\asymp  m^{-s},
$
if $f^*$ has $s$-order smoothness with $s\in (0,2]$ and is not CPWL. This   saturation for CPWL models occurs at $s=2$ since we are approximating locally by linear functions. A function with smoothness of order $s>2$ would need locally polynomial of degree higher than $1$ to improve the local error (think of Taylor expansion). Hence, when $f$ has smoothness of order $s>2$, CPWL approximation does not improve the rate. Analogously, this provides an explanation for the saturation of KRR for the specific kernel~\eqref{eqn: min-kernel}.

We next show that the argument behind \eqref{eqn: app-lower-bound-saturation} can establish that saturation  occurs precisely at  $s=2$ for a family of  dot-product kernels and periodic kernels. 
\begin{theorem}\label{thm: target-saturation}
Let  $k$ be a dot-product kernel over $\SS^{d-1}$ or periodic kernel over  $\TT$. Then, there exists a positive constant $C$ which only depends on $d$, such that for any  $x_1,x_2,\dots,x_n\in\cX$, $0\leq p\leq 2$, and $s>2$, it holds for any $\lambda\geq 0$  that
$\sup_{f^*\in \cH^s}\|\hf_\lambda-f^*\|_{\cH^p}^2\geq \frac{1}{2}\mu_{Cn}^{2-p}$.
\end{theorem}
The proof can be found in Appendix~\ref{sec:proof_of_theorem_ref_thm_target_saturation}. It is worth noting that this theorem is somehow stronger than what we expect since  the saturation occurs for any $x_1,x_2,\dots,x_n$. 

\begin{remark} 
It is instructive to compare the above result with the well-known saturation that arises in the presence of noise. In that case, the saturation is  caused by the ridge-type regularization~\cite{li2022saturation} and can  be resolved by employing spectral algorithms, such as gradient descent or spectral cut-off~\cite{bauer2007regularization}. In contrast, in the noiseless setting considered here, such spectral algorithms cannot mitigate the saturation as their solutions still satisfy the representer theorem and, as explained above, remain subject to the same structural restriction. 
\end{remark}

\vspace*{-.5em}
\section{Numerical Experiments} 
\vspace*{-.5em}

\label{sec: experiment}
In this section, we present numerical experiments to validate  our theoretical findings. Specifically, we consider the Brownian kernel given by \eqref{eqn: min-kernel}
and the uniform distribution $\rho=\text{Unif}([0,1])$. The corresponding eigenfunctions and eigenvalues  are given by 
\[
	e_j(x)=\sqrt{2}\sin\left(\frac{2j-1}{2}\pi x\right),\quad \mu_j = \left(\frac{2}{\pi(2j-1)}\right)^2, \text{ for } j=1,2,\dots.
\]
Notably, $\mu_j\asymp j^{-\beta}$ with $t=2$.
We refer to \cite{wainwright2019high}*{Section 12} for further details of this kernel and the associated RKHS. In this case, the eigenfunctions are uniformly bounded, and thus we  can consider the unweighted sampling case where $\rho'=\rho$.
We examine the learning of target functions with varying orders of smoothness, given by
\begin{equation} \label{old_target}
    F^*_s=\sum_{j=1}^\infty j^{-(0.5 + s)}e_j \text{ for } s\in(0,\infty), \text{ and } 
    F^*_\infty = e_1,
\end{equation}
where, for $s<\infty$, the choice of coefficients ensures $f^* \in \cH^{s-\epsilon}$ for any $\epsilon\in (0,1)$. 
For a given target function $f^*=\sum_{j=1}^\infty f^*_j e_j$ and a solution $f_a = \sumin a_i k(x_i,\cdot)$ with $a\in\RR^n$, the $\cH^p$-error can be computed as follows
\begin{align*}
\left\|f_a - f^*\right\|_{\cH^p}^2 &= \left\|\sum_{i=1}^na_i\sum_{j=1}^\infty \mu_je_j(x_i)e_j - f^*\right\|_{\cH^p}^2 =\left\|\sum_{j=1}^\infty \mu_j \sum_{i=1}^na_ie_j(x_i)e_j - \sum_{j=1}^\infty f_j^* e_j\right\|_{\cH^p}^2 \\ 
&= \sum_{j=1}^\infty \left(\mu_j \sum_{i=1}^na_ie_j(x_i) -  f_j^*\right)^2 \mu_j^{-p}= \sum_{j=1}^\infty \left(a^\top \hat{e}_j  -  f_j^*/\mu_j\right)^2 \mu_j^{2-p},
\end{align*}
where $\hat{e}_j = (e_j(x_1),\dots, e_j(x_n))^\top\in \mathbb{R}^n$. In practice, we approximate this  series by truncation:
\begin{align}
\left\|f_a - f^*\right\|_{\cH^p}^2 &\approx \sum_{j=1}^m \left(a^\top \hat{e}_j  -  f_j^*/\mu_j\right)^2 \mu_j^{2-p}.
\end{align}
Note that the truncation parameter $m$ is chosen based on the decay of the summands; specifically,  smaller $s$ and larger $p$ requires a larger $m$ to ensure a reliable approximation.

We first investigate  how the regularization hyperparameter $\lambda$ affects the performance of noiseless KRR, with the result presented in Figure~\ref{fig: lambda}. Theorem~\ref{thm:krr_error} guarantees that, in the well-specified case with $p\in [0,1]$,  noiseless KRR achieves optimal rates for any  $\lambda\in [0,\lambda_n]$ with $\lambda_n=\Lambda_1(n,\delta)$.  Figure~\ref{fig: lambda}  further demonstrates  that for both misspecified and well-specified cases, across a wide range of $p$, the performance of noiseless KRR actually improves {\it monotonically} as $\lambda$ decreases towards zero. Therefore, unlike the noisy case where $\lambda$ requires careful tuning,   the performance of noiseless KRR is highly robust to the choice of $\lambda$ and it is advised to set $\lambda$ as small as possible.

\begin{figure}[!th]
\centering
\includegraphics[width=0.4\textwidth]{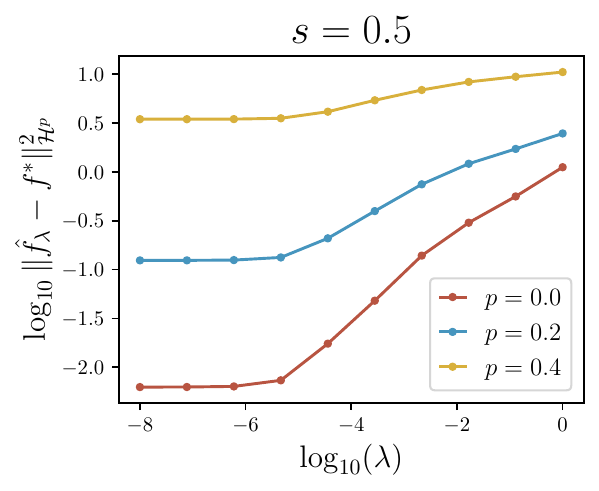}
\includegraphics[width=0.4\textwidth]{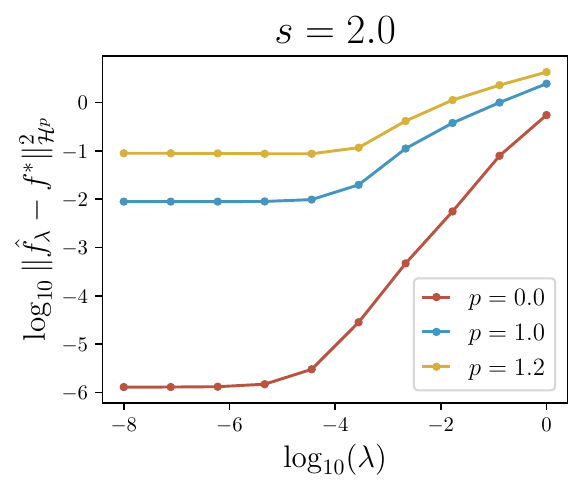}
\vspace*{-1em}
\caption{The performance of noiseless KRR improves {\it monotonically} as $\lambda$ decreases. The target functions are $F_s^*$ with $s=0.5$ (left) and $s=2$ (right), and the  sample size $n=100$.}
\label{fig: lambda}
\end{figure}

We next investigate how the actual convergence rate of KRR depends on \(s\) (the smoothness of the target function) and \(p\) (the order of the norm used to measure performance), with the results presented in Figure~\ref{fig: noiseless-krr}. Overall, the empirical results for all examined \(s\) and \(p\) pairs are in perfect agreement with our theoretical predictions. In Figure~\ref{fig: noiseless-krr} (left), we observe that the noiseless KRR solution exhibits smoothness of order greater than $1$, and reducing \(p\) below $0$ does not further enhance the convergence rate. In Figure~\ref{fig: noiseless-krr} (middle), we see that noiseless KRR adapts to the target function's smoothness as predicted by Theorems~\ref{thm:krr_error} and \ref{thm:krr_error-one}; notably, this adaptation saturates precisely at \(s=2\), thereby validating our theoretical analysis in Section~\ref{sec:the_saturation_of_smoothness_adaptivity}. Finally, Figure~\ref{fig: noiseless-krr} (right) shows that the solution of noisy KRR also exhibits smoothness of order greater than $1$ and  the corresponding  convergence rate aligns perfectly with the predictions of Theorem~\ref{thm: noisy-krr}.

\begin{figure}[!ht]
\centering
\includegraphics[width=0.333\textwidth]{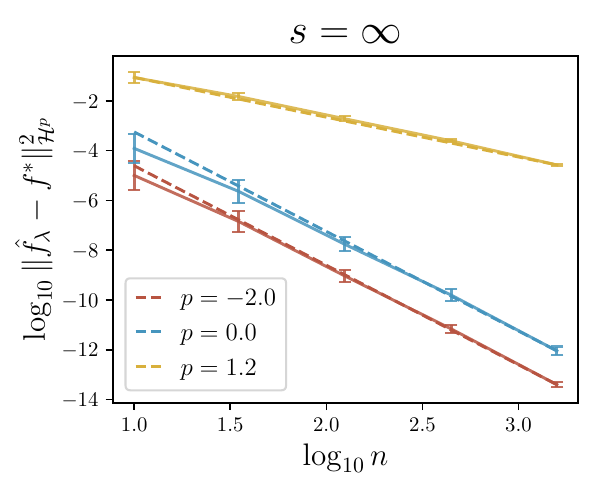}
\hspace*{-1em}
\includegraphics[width=0.333\textwidth]{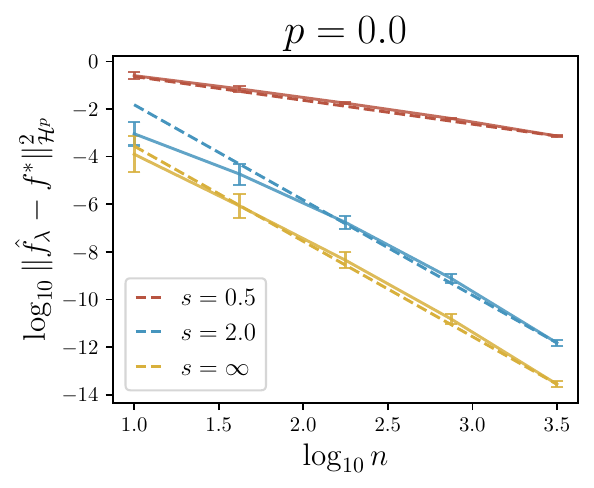}
\hspace*{-1em}
\includegraphics[width=0.333\textwidth]{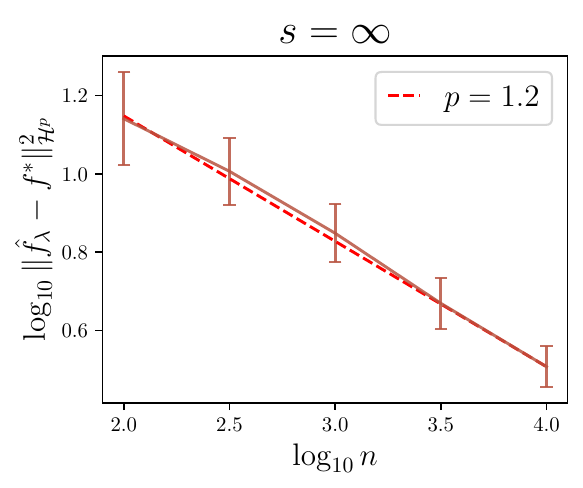}
\vspace*{-1em}
\caption{The observed convergence rates align perfectly with our theoretical predictions for all examined pairs of \(s\) and \(p\). 
{{\bf Left:}} Convergence rates of noiseless KRR for various values of $p$ when the target function is \(F_\infty^*\). {{\bf Middle:}} Convergence rates of noiseless KRR for various values of \(s\), measured in \(L^2\) norm (i.e., \(p=0\)). {{\bf Right:}} The solution of noisy KRR for the target function $F_\infty^*$ (with  noise level $\sigma=1$) exhibits smoothness of order strictly greater than $1$.
In all panels, the dashed line represents the theoretical predictions, and each experiment is repeated 20 times with error bars indicating the standard deviation across runs. For the noiseless experiments (left and middle), we set $\lambda=10^{-20}$,  sufficiently small following the guidance from Figure~\ref{fig: lambda}. For the the noisy experiment (right), we set $\lambda=0.05n^{-\min(s,2)/(\min(s,2)\beta+1)}$, following the theoretical prescription given  by the analysis in Eq.~\eqref{eqn: noisy-rate}, with   $s=\infty$ and $t=2$.
}
\label{fig: noiseless-krr}
\end{figure}

\vspace*{-1em}
\section{Conclusion and Discussion}
\label{sec:concluding_remark}
\vspace*{-.1em}

In this work, we present a comprehensive analysis of KRR in the noiseless regime. We show that KRR, up to logarithmic factors, achieves minimax optimal rates under a broad class of Sobolev-type interpolation norms for general source conditions. This characterization not only recovers classical $L^2$-error bounds as a special case but also reveals how the interplay between the eigenvalue decay of the associated integral operator and the smoothness of the target function governs the convergence of learning error.
A central contribution of our work is the identification and rigorous characterization of the extra-smoothness and saturation effects. The extra-smoothness effect reveals that KRR solutions actually exhibit higher-order smoothness than previously expected. The 
the saturation effect highlights an intrinsic limitation of KRR in adapting to increasingly smooth target functions in the absence of noise. 

To enable these results, we introduced a refined notion of DoFs, which exploits the extra smoothness of kernel functions and has played a key role in deriving both our upper and lower bounds. This novel complexity measure also yield improved theoretical guarantees for KRR in the noisy regime, where we establish new error bounds that are explicitly aware of the noise level, thereby bridging the gap between the noiseless and noisy regimes.

Looking ahead, an intriguing direction is to extend our techniques to analyze other kernel-based methods. A particularly relevant case is to study the efficiency of random feature approximations, which are widely adopted in practice to accelerate KRR~\cite{rahimi2007random}. Existing analyses   are limited to the noisy regime \cite{rudi2017generalization,liu2021random}, where the number of features is determined to maintain a convergence rate of at most $O(n^{-1})$ -- the fastest rate one can expect for learning in the noisy regime. In contrast, the noiseless regime, where significantly faster rates are attainable as demonstrated in Theorems~\ref{thm:krr_error} and~\ref{thm:krr_error-one}, remains largely unexplored in this context. This gap raises a natural and important question:
{\it 
	Can random feature approximations   preserve the optimal   noiseless rates  without compromising computational efficiency?
}

Interestingly, our analysis offers a starting point toward this question. In our proofs, we develop a random feature characterization of the interpolation spaces (Proposition~\ref{pro: feature-coeff} ) and establish  an equivalence between estimation in noiseless KRR and approximation using random features (Proposition~\ref{pro: equivalence}). This connection provides a promising bridge between the statistical and computational aspects of random feature approximation, and may serve as a foundation for tackling this open problem.

\subsection*{Acknowledgement} 
This work was supported by the National Key R\&D Program of China (No.~2022YFA1008200) and National Natural Science Foundation of China (No.~12288101).
We thank Zixun Huang and Dr.~Zhu Li for helpful discussions.
\bibliography{ref}

\newpage
\appendix
\addcontentsline{toc}{section}{Appendix} 
\part{Appendix} 
\parttoc 

\section{Mathematical Preliminaries}
\subsection{Bounded Linear Operators in Hilbert Spaces}
\label{sec: bounded-operator}
Given two Hilbert spaces $\cH_1$ and $\cH_2$, we consider a linear operator $A:\cH_1\mapsto\cH_2$. The operator norm is defined as 
$
    \|A\| = \sup_{ x \in \cH_1, x \neq 0}\frac{\|A x\|_{\cH_2}}{\|x\|_{\cH_1}},
$
where  we omit the dependence  on $\cH_1$ and $\cH_2$ in $\|A\|$ for simplicity. 

{\bf Adjoint operator.~} The adjoint operator $A^*:\cH_2\mapsto \cH_1$ is defined through:
\begin{equation*}
    \langle y,A x\rangle_{\cH_2} = \langle A^* y, x\rangle_{\cH_1}, \, \forall x \in \cH_1 \text{ and }y \in \cH_2,
\end{equation*}
for which it holds that $\|A\|^2 = \|A^*\|^2 = \|A^*A\| = \|AA^*\|$. We call a bounded operator $A$ self-adjoint, if $\cH_1 = \cH_2$ and $A^* = A$. 
A self-adjoint operator $B$ is said to be positive , if for any $x \in \cH$, $\langle x, B x \rangle_\cH \ge 0$. For two self-adjoint operator $B$ and $B'$, we will use $B \lem B'$ to mean $B'-B$ is positive. 

{\bf Outer product.~} 
Let $v\in\cH_1$ and $u\in \cH_2$, we define the outer product $u\otimes v$ as a linear operator $\cH_1\mapsto\cH_2$ given by $(u\otimes v)x = \langle v,x\rangle_{\cH_1}u$ for any $x\in\cH_1$. Moreover, $\|(u\otimes v)\|=\sup_{\|x\|_{\cH_1}=1}\|\langle v,x\rangle_{\cH_1}u\|_{\cH_2}=\|u\|_{\cH_2}\|v\|_{\cH_1}$.

{\bf Eigenvalue decomposition.~}
An operator $A$ is said to be compact, if it maps the bounded subset of $\cH_1$ to a relatively compact subset of $\cH_2$. If a  self-adjoint operator $B$ is compact, then it has the following spectral decomposition:
\begin{equation*}
    B  = \sum_{j=1}^{\infty} \mu_j \psi_j\otimes \psi_j, 
\end{equation*}
where $\{\mu_j\}_{j=1}^{\infty}$ are the eigenvalues in a decreasing order and $\{\psi_j\}_{j=1}^{\infty}$ are the corresponding orthonormal eigenfunctions satisfying $\langle \psi_i,\psi_j\rangle_\cH = \delta_{ij}$ and  $B \psi_j = \mu_j \psi_j$. If $B$ is further positive, then all $\mu_j$'s are non-negative. In addition, we define the trace of a self-adjoint operator $B$ by $\tr(B) = \sum_{j=1}^\infty \mu_j$. Let $N = \sup\{j \in \bN^+,\, \mu_j > 0\}$ (possibly infinity), we  define $B^r$ for any $r \in \RR$:
\begin{equation}\label{eqn: operator-power}
    B^r  = \sum_{j=1}^{N} \mu_j^r \psi_j\otimes \psi_j.
\end{equation}
Noticing that when $r < 0$, $B^r$ can only define on the range of $B^{-r}$. 

{\bf Singular value decomposition (SVD).~}
For a compact operator $A:\cH_1\mapsto\cH_2$, both $AA^*$ and $A^*A$ are compact, self-adjoint and positive operators. 
Moreover, $AA^*$ and $A^*A$ share the same non-zero eigenvalues and the singular value decomposition of $A$ is given by
\begin{equation*}
    A  = \sum_{j=1}^{\infty} \mu_j^{\half} \psi'_j\otimes \psi_j,
\end{equation*}
where $\{\mu_j\}_{j=1}^\infty$ are the eigenvalues of $AA^*$ in a decreasing order, $\{\psi_j\}_{j=1}^\infty$ and $\{\psi'_j\}_{j=1}^\infty$ are the corresponding orthonormal eigenfunctions of $A^*A$ and $AA^*$, respectively.  Through the singular value decomposition, it is easy to prove that for any $r \in \RR,\lambda>0$, it holds that
\begin{equation}\label{matrix_inverse_lemma}
    (A^*A + \lambda)^{r}A^* = A^*(AA^*+\lambda)^{r}.
\end{equation}
In our proof, we will frequently use this equality.

For more information of bounded linear operators in Hilbert spaces, we refer to \cite{kato2013perturbation}.

\subsection{Useful Operator Inequalities}
\label{sec: useful-inequalties}
\begin{lemma}[Cordes inequality, \cite{fujii1993norm}]\label{lemma: Cordes}
Given two compact, self-adjoint and positive operators $A$ and $B$, we have for all $r \in [0, 1]$ that 
\begin{equation}\label{ineq:cordes}
    \|A^r B^r\| \le \|A B\|^r.
\end{equation}
\end{lemma}


\begin{theorem}[Theorem 3.1 in \cite{minsker2017some}]\label{thm: concentration}
Let $X_1, \ldots, X_n $ be a sequence of independent self-adjoint random operators on Hilbert space $\cH$ such that $\mathbb{E} X_i=0$ for $i=1, \ldots, n$ and $\left\|\sum_{i=1}^n \mathbb{E} X_i^2\right\|\leq \sigma^2$. Assume that $\left\|X_i\right\| \leq U$ almost surely for all $1 \leq i \leq n$ and some positive $U \in \mathbb{R}$. Then, for any $t \ge \frac{1}{6}\left(U+\sqrt{U^2+36 \sigma^2}\right)$,
\begin{equation}\label{eqn: 111}
\mathbb{P}\left(\left\|\sum_{i=1}^n X_i\right\|>t\right) \leq 14 \frac{\tr(\sum_{i=1}^n \mathbb{E} X_i^2)}{\sigma^2} \exp \left[-\frac{t^2 / 2}{\sigma^2+t U / 3}\right].
\end{equation}
\end{theorem}
We remark that the above theorem is slightly tighter than the version stated in \cite{minsker2017some}*{Theorem 3.1}, where $\frac{\tr(\sum_{i=1}^n \mathbb{E} X_i^2)}{\sigma^2}$ in \eqref{eqn: 111} is replaced by $\frac{\tr(\sum_{i=1}^n \mathbb{E} X_i^2)}{\left\|\sum_{i=1}^n \mathbb{E} X_i^2\right\|}$. However,  \cite{minsker2017some} in fact has proved the tighter presented here; we refer to the paragraphs below Eq.~(3.9) in \cite{minsker2017some} for more details.
It is worth noting that a larger $\sigma$ imposes a stricter condition on $t$, which limits the size of deviation for which we can provide concentration guarantees. 

\subsection{$K$-Functionals and Interpolation Spaces}\label{sec:interpolation_space}
In this section, we let $X,Y$ be two Banach spaces that are continuously embedded in a common topological vector space (i.e., they are compatible). We denote 
$$
X+Y=\{a+b: a\in X, b\in Y\}.
$$
\begin{definition}[$K$-functional]
    For $t\geq 0$ and $z\in X+Y$, the $K$-functional is defined by 
    \[
    K_{X,Y}(t,z) = \inf_{z = a+b}(\|a\|_X+t\|b\|_Y).
    \]
\end{definition}
The $K$-method of real interpolation defines a family of interpolation spaces as follows:
\begin{definition}[Interpolation space] 
For any $0 < \theta < 1$ and $q \in [1,\infty]$, the real interpolation space $(X,Y)_{\theta,q}\subset X+Y$ is defined as
\begin{equation*}
    (X,Y)_{\theta,q} = \left\{z \in X+Y : \, t\mapsto t^{-\theta} K_{X,Y}(t,z) \in L^q(\RR^+; \dd t/t)\right\},
\end{equation*}
equipped with the norm
\[
	\|z\|_{(X,Y)_{\theta,q}}=\left(\int_0^\infty \left(t^{-\theta} K_{X,Y}(t,z)\right)^q\frac{\dd t}{t}\right)^{\frac{1}{q}}.
\] 
\end{definition}
\begin{proposition}[Proposition 1.3 in \cite{lunardi2018interpolation}]\label{pro: embedding-interpolation-space-Kmethod}
For any $0 < \theta < 1$ and $1 \le q_1 \le q_2 \le \infty$, it holds that
\begin{equation}\label{continous_embeding}
    (X,Y)_{\theta,q_1} \subset (X,Y)_{\theta,q_2} \quad \text{ with continuous embedding.}
\end{equation}
\end{proposition}
Our subsequent analysis needs to introduce the Lorentz space $L^{p,q}(\rho)$  with $1\le p,q\le \infty$:
\[
    \|f\|_{L^{p,q}(\rho)} = \left(\int_0^\infty \big(t\rho\{x: |f(x)|^p\geq t\}\big)^{q/p}\frac{\dd t}{t}\right)^{1/q}.
\]
The Lorentz space refines the $L^p(\rho)$ space by providing a finer scale of integrability and rearrangement properties. Particularly, $L^{p,p}(\rho)=L^p(\rho)$ and $L^{p,q}(\rho)$ is an interpolation space:
\begin{proposition}[\cite{tartar2007introduction}]
For any $1\leq p,q\leq \infty$, it holds that 
$L^{p,q}(\rho) = (L^1(\rho), L^\infty(\rho))_{1-\frac{1}{p},q} $
\end{proposition}
Moreover, Theorem 1.3.5 in \cite{lunardi2018interpolation} (see also Lemma 4 in \cite{zhang2024optimality}) gives a refiner characterization as follows:
\begin{proposition}\label{L_p interpolation}
Let $\frac{1}{p_\theta} = \frac{1-\theta}{p_1} +\frac{\theta}{p_2}$. Then,
$
    L^{p_\theta, q} = (L^{p_1}(\rho), L^{p_2}(\rho))_{\theta,q}.
$
\end{proposition}

In addition, the interpolation space associated a separable RKHS given in Definition~\ref{def: Hks-space} can be constructed using $K$-method:
\begin{proposition}[Theorem 4.6 in  \cite{steinwart2012mercer}]
For $s\in (0,1)$,  $\cH^s = (L^2(\rho),\cH)_{s,2}$.
\end{proposition}
Therefore, for any $0 < s_1 < s_2$,
\begin{equation}\label{H interpolation}
    \cH^{s_1} = (\cH^{s_2})^{\frac{s_1}{s_2}} = \left(L^2(\rho),\cH^{s_2}\right)_{\frac{s_1}{s_2},2}.
\end{equation}
We shall  use the Riesz-Thorin interpolation theorem given below.
\begin{theorem}[Proposition 1.1.6 in \cite{lunardi2018interpolation}]\label{thm:interpolation}
    Let $(X_1,X_2)$ and $(Y_1,Y_2)$ be two couples of compatible Banach spaces. Let $T$ be bounded linear operator in $X_1 \mapsto Y_1$ and $X_2 \mapsto Y_2$. Then for any $\theta\in (0,1)$ and $ p \in [1,\infty]$, $T$ is also a bounded operator from $(X_1,X_2)_{\theta,p}$ to $(Y_1,Y_2)_{\theta,p}$ and
    \begin{equation*}
        \|T\|_{(X_1,X_2)_{\theta,p}\mapsto (Y_1,Y_2)_{\theta,p}} \le \|T\|_{X_1 \mapsto Y_1}^{1-\theta} \|T\|_{X_2\mapsto Y_2}^\theta.
    \end{equation*}
\end{theorem}

We refer \cite{tartar2007introduction,lunardi2018interpolation} for more discussion on interpolation spaces and real interpolation method and \cite{steinwart2012mercer,zhang2024optimality} for its application in studying kernel methods.

\subsection{Auxiliary Lemmas}

\begin{lemma}\label{lemma: AB-bound}
Let $A$ be a positive self-adjoint compact operator defined over a separable Hilbert space and $\alpha,\beta,\gamma\in (0,\infty)$. If $\beta\leq \alpha$, we have
$\|(A+\alpha )^\gamma (A+\beta)^{-\gamma}\|\leq (\alpha/\beta)^\gamma$.
\end{lemma}
\begin{proof}
Let $A=\sum_{j} \lambda_j \varphi_j\otimes \varphi_j$ be the spectral decomposition of $A$. Then,
\begin{align*}
    \|(A+\alpha )^\gamma (A+\beta)^{-\gamma}\| &= \left\|\sum_j \left(\frac{\lambda_j+\alpha}{\lambda_j +\beta}\right)^\gamma \varphi_j\otimes \varphi_j\right\| \\ 
    &\leq \sup_{j}\left(\frac{\lambda_j+\alpha}{\lambda_j +\beta}\right)^\gamma\leq \left(\frac{\alpha}{\beta}\right)^\gamma,
\end{align*}
where the last step uses the fact $g(t):=\frac{t+\alpha}{t+\beta}$ is decreasing in $[0,\infty)$ when $\alpha\geq\beta > 0$.
\end{proof}

\begin{theorem}[Lemma 1 in \cite{rosenthal1970subspaces}]\label{lemma: rosenthal-inequality}
Let $X_1,\dots,X_n$ be $n$ independent non-negative random variables with finite $p$-th order moments with $1\leq p<\infty$. Then, we have
\[
    \EE\left(\sum_{i=1}^n X_i\right)^{p}\leq C_p \max\left\{\sum_{i=1}^n \EE[X_i^p], \left(\sum_{i=1}^n \EE[X_i]\right)^p\right\}
\]
\end{theorem}

\begin{lemma}\label{eqn: markov}
Let $X$ be a non-negative random variable and suppose $\EE[X^q]<\infty$ for some $q>0$. Then, for any $\delta\in (0,1)$, with probability at least $1-\delta$, we have 
\[
    X\leq \delta^{-1/q} (\EE[X^q])^{1/q}.
\]
\end{lemma}
\begin{proof}
By Markov's inequality, for any $t\geq 0$, it holds that
\[
    \PP\{X\geq t\} = \PP\{X^q\geq t^q\} \leq \frac{\EE[X^q]}{t^q}
\]
Setting the right hand side to $\delta$ gives $t = \delta^{-1/q}(\EE[X^q])^{1/q}$.
This completes the proof.
\end{proof}

\begin{lemma}\label{lemma: f_lambda}
For any $\alpha\in (0,1)$, we have
$\sup_{t>0}\frac{t^{\alpha}}{t+\lambda}\lesssim_\alpha \lambda^{\alpha-1}$ for any $\lambda>0$.	
\end{lemma}
\begin{proof}
Let $f(t)=\frac{t^{\alpha}}{t+\lambda}$. Then, 
\[
	f'(t) = \frac{\alpha t^{\alpha-1}(t+\lambda)-t^{\alpha}}{(t+\lambda)^2} = \frac{(1-\alpha) t^{\alpha-1}}{(t+\lambda)^2} \left(\frac{\lambda \alpha}{1-\alpha }-t\right).
\]
Since $\alpha\in (0,1)$, $f$ is increasing in $[0, t^*]$ and decreasing in $[t^*,\infty)$ with $t^*=\lambda\alpha/(1-\alpha)=:c_\alpha \lambda$. Thus,
\[
	f(t)\leq f(t^*) = c_\alpha^\alpha\frac{\lambda^\alpha}{c_\alpha \lambda+\lambda}= \frac{c_\alpha^\alpha}{1+c_\alpha}\lambda^{\alpha-1}.
\]
\end{proof}

\section{Properties of the Interpolation Spaces}

Our proofs shall use the following properties of the interpolation space $\cH^s$, given in Definition~\ref{def: Hks-space}.
\subsection{Embeddings into $L^q$ Spaces}
\label{sec:the_l_q_embedding_of_ch_s_}

\begin{theorem}\label{thm:embedding}
    Let $\gamma\in (0,1)$ satisfy $\sum_{j}\mu_j^{\gamma}<\infty$. For any $\lambda\geq 0$, let $m(\lambda) = \max\{j \in \NN^+: \mu_j \ge \lambda\}$. For $g = \sum_{j = m(\lambda)+1}^\infty a_j e_j$, let $g_q(x) = \frac{1}{\sqrt{q(x)}}g(x)$.
    Then, for any $0 \le s < \gamma$, there exists a constant $C_{s,\gamma} > 0$ such that 
    \begin{equation*}
        \left\|g_q\right\|_{L^{q_{s} }(\rho')} \le C_{s,\gamma}\lambda^{\frac{s}{2}}F_\gamma(\lambda)^{\frac{s}{2\gamma}} \|g\|_{\cH^s} \quad \text{ with }\, q_{s}=\frac{2\gamma}{\gamma-s}.
    \end{equation*}
\end{theorem}
This theorem shows that when $s>0$, $\cH^s$  can be continuously embedded into $L^{q_s}$ with $q_s>2$. Consider the power-decay spectrum $\mu_j\propto j^{-\beta}$  with $\beta>1$ and $q$ is chosen such that $F_\gamma(\lambda) = N_\gamma(\lambda)$, then $F_\gamma(\lambda)\propto \lambda^{-1/\beta}$. Then, 
\[
    \left\|g_q\right\|_{L^{q_{s} }(\rho')}\lesssim_{\gamma,s}\lambda^{\frac{s}{2}(1-\frac{1}{\beta \gamma})}.
\]
Note that it always holds that $\gamma\geq 1/\beta$ and thus $1-1/(\beta\gamma)\geq 0$. Thus, the higher-frequency component has smaller $L^{q_{s}}$ norm.

We remark that a similar embedding has been established in  \cite{zhang2024optimality}*{Theorem 5}, which shows that
$
    \cH^s \hookrightarrow L^{2\alpha/(\alpha-s)}(\rho)
$
for any $\alpha$ such that $\cH^\alpha$ can continuously embed into $L^\infty(\rho)$.
Theorem~\ref{thm:embedding} generalizes \cite{zhang2024optimality}*{Theorem 5} in two aspects. First, the introducing of weighted sampling allows to take $\gamma$ smaller than $\alpha_0$. For instance, consider the pow-decay spectrum, we can take $\gamma$ as close as possible to $1/\beta$. However, $\alpha_0$ may be away from $1/\beta$ when the eigenfunction is not bounded or the kernel does not have symmetry. Second, our theorem provide a fine-grained characterization of how the frequency affects the embedding constant, which play a crucial role in our subsequent analysis.

\begin{proof}
   Consider the operator $\QQ$ from $L^2(\rho)$ to $L^2(\rho')$:
    \begin{equation*}
        (\QQ f)(x) = \frac{1}{\sqrt{q(x)}}\sum_{j=m(\lambda_n)+1}^\infty a_j e_j(x),\, \text{ if }f = \sum_{j=1}^\infty a_j e_j.
    \end{equation*}
    Since 
    $
        \|g_q\|_{L^{q_{s,\gamma}}}= \|\QQ  g\|_{L^{q_{s,\gamma}}}\leq \|\QQ \|_{\cH^{s}\mapsto L^{q_{s,\gamma}}(\rho')}\|g\|_{\cH^s},
    $
    it suffices to show
    \begin{equation*}
        \|\QQ \|_{\cH^{s}\mapsto L^{q_{s,\gamma}}(\rho')} \le C_{s,\gamma} \lambda^{\frac{s}{2}}F_\gamma(\lambda)^{\frac{s}{2\gamma}}.
    \end{equation*}

    First, we have $\|\QQ \|_{L^2(\cX, \rho)\mapsto L^2(\cX, \rho')} \le 1$ due to 
    \begin{align*}
        \|\QQ  f\|_{\rho'}^2 &= \int_{\cX} \left(\sum_{j=m(\lambda_n)+1}^\infty a_j e_j(x)\right)^2 \frac{1}{q(x)}\dd \rho'(x) = \int_{\cX} \left(\sum_{j=m(\lambda_n)+1} a_j e_j(x)\right)^2 \dd \rho(x) \\ 
        &= \sum_{j=m(\lambda_n)+1}^\infty a_j^2 \leq\|f\|_{\rho}^2.
    \end{align*}
    Second, we have  $\|\QQ \|_{\cH^{\gamma}\mapsto L^\infty}\leq (2\lambda)^{\gamma/2} F_\gamma(\lambda)^{1/2}$ due to
    \begin{align*}
        \|\QQ f\|_\infty^2 &= \esssup_{x \in \cX}\left|\frac{1}{\sqrt{q(x)}}\sum_{j=m(\lambda_n)+1}^\infty a_j e_j(x)\right|^2 \\
        &\le\esssup_{x \in \cX} \frac{1}{q(x)}\left(\sum_{j=m(\lambda)+1}^\infty \mu_j^{\gamma} |e_j(x)|^2\right) \left(\sum_{j=m(\lambda)+1}^\infty \mu_j^{-\gamma} a_j^2\right)\\
        &\leq \|f\|_{\cH^{\gamma}}^2(2\lambda)^{\gamma}\esssup_{x \in \cX} \frac{1}{q(x)}\left(\sum_{j=m(\lambda)+1}^\infty \frac{\mu_j^{\gamma}}{(2\lambda)^{\gamma}} |e_j(x)|^2\right)\\
        &\le \|f\|_{\cH^{\gamma}}^2 (2\lambda)^{\gamma}\esssup_{x \in \cX} \frac{1}{q(x)}\sum_{j=m(\lambda)+1}^\infty \frac{\mu_j^{\gamma}}{(\mu_j + \lambda)^{\gamma}} |e_j(x)|^2 \\
        &\le \|f\|_{\cH^{\gamma}}^2(2\lambda)^{\gamma}F_{\gamma}(\lambda).
    \end{align*}
  
    Recalling Proposition~\ref{L_p interpolation} and Equation~\eqref{H interpolation}, we have
    \[
     L^{\frac{2\gamma}{\gamma-s}}(\rho')=(L^2(\rho'), L^\infty(\rho'))_{\frac{s}{\gamma},\frac{2\gamma}{\gamma-s}},\quad \cH^{s} = (L^2(\rho),\cH^{\gamma})_{\frac{s}{\gamma},2}
    \]
    Then, by Theorem \ref{thm:interpolation}, we have
    \begin{align*}
        \|\QQ \|_{\cH^{s}\mapsto L^{\frac{2\gamma}{\gamma-s} }(\rho')} &= \|\QQ \|_{(L^2(\rho),\cH^{\gamma})_{\frac{s}{\gamma},2}\mapsto (L^2(\rho'), L^\infty(\rho'))_{\frac{s}{\gamma},\frac{2\gamma}{\gamma-s}}} \\ 
        &\le C_{s,\gamma} \|\QQ\|_{(L^2(\rho),\cH^{\gamma})_{\frac{s}{\gamma},2}\mapsto (L^2(\rho'), L^\infty(\rho'))_{\frac{s}{\gamma},2}}\\
        &\le C_{s,\gamma}\|\QQ\|_{L^2(\rho)\mapsto L^2(\rho')}^{1-\frac{s}{\gamma}} \|\QQ\|_{\cH^{\gamma}\mapsto L^\infty(\rho')}^{\frac{s}{\gamma}}\\ 
        & \le C_{s,\gamma} \lambda^{\frac{s}{2}}F_\gamma(\lambda)^{\frac{s}{2\gamma}},
    \end{align*} 
    where the second step follows from Proposition~\ref{pro: embedding-interpolation-space-Kmethod}.
\end{proof}

\subsection{A Random Feature Characterization}
\label{sec:a_random_feature_perspective}
Our proof adopts the following random feature characterization of the interpolation spaces, which extends an analogous characterization of RKHSs developed in \cite{chen2023duality}.
Let $\phi:\cX\times \Omega\mapsto\RR$ be a parametric feature function with $\cX$ and $\Omega$ denoting the input and weight domain, respectively. For any $\pi\in \cP(\Omega)$, we can define an associated random feature kernel  as follows
\begin{equation}\label{eqn: RF-kernel}
k_\pi(x,x') = \int_{\Omega} \phi(x,\omega)\phi(x',\omega)\dd \pi(\omega).
\end{equation}
Without loss of generality, we  assume that the kernel of interest $k$ admits the above random feature representation. This assumption is justified by \cite{chen2023duality}*{Lemma 4.2}, which shows that for any kernel $k:\cX\times\cX\mapsto\RR$ satisfying $\int_{\cX} k(x,x)\dd\rho(x)<\infty$, 
there exists a probability space $(\Omega,\pi)$ and  $\phi:\cX\times\Omega \rightarrow \RR$ such that $k=k_\pi$ in the sense $L^2(\rho\otimes \rho)$.

Consider the conjugate kernel $\tilde{k}: \Omega \times\Omega \mapsto \RR$ given by
$
    \tilde{k}(\omega,\omega') = \int_{\cX}\phi(x,\omega)\phi(x,\omega')\dd\rho(x).
$
Then, $k$ and $\tk$ are defined over the input domain and weight domain, respectively. 
Let $\tH^s$ denote the interpolation space induced by the kernel $\tk$ and 
define  $\cT: L^2(\Omega,\pi)\mapsto L^2(\Omega,\pi)$ by 
$$
    \cT \alpha= \int_{\Omega}\tilde{k}(\cdot,\omega) \alpha(\omega)\dd \pi(\omega).
$$
To clarify the relationship between $\cT$ and $\cL$, we introduce the  feature operator $\Phi: L^2(\Omega, \pi)\mapsto L^2(\cX, \rho)$ and its adjoint operator as follows:
\[
 \Phi \alpha = \int_{\Omega}\alpha(\omega)\phi(\cdot,\omega)\dd\pi(\omega) \,\text{,  } \,\Phi^* f = \int_{\cX}f(x)\phi(x,\cdot)\dd\rho(x).
 \]
It is easy to verify that 
$
    \cL = \Phi\Phi^*,  \cT = \Phi^*\Phi.
$
Let $\{\tilde{e}_j\}_{j=1}^\infty$ be the orthonormal eigenfunctions of $\cT$. Then, $\Phi$ and $\Phi^*$ has  the following singular value decompositions (SVDs):
\begin{equation}\label{eqn: svd}
 \Phi = \sum_{j=1}^\infty\mu_j^{\half}e_j\otimes \te_j,\quad \Phi^* = \sum_{j=1}^\infty \mu_j^{\half} \te_j\otimes e_j.
\end{equation}
By the above SVDs, it holds that 
$
    \Phi\alpha = \sum_{j=1}^\infty \mu_j^{1/2} a_j e_j\text{ for any } \alpha = \sum_{j=1}^\infty  a_j \te_j.
$
Therefore, for any $s\geq -1$, we have
\begin{align*}
\|\Phi \alpha\|_{\rho}^2 = \sum_{j=1}^\infty \mu_ja_j^2 = \sum_{j=1}^\infty \mu_j^{1+s}\frac{a_j^2}{\mu_j^s}\leq \mu_1^{1+s}\|\alpha\|_{\tH^s}^2.
\end{align*}
This implies that $\Phi$ can be continuously extended to $\tilde{\cH}^s$ with $s\in [-1,0)$. Analogously, $\Phi^*$ can be extended to $\cH^s$ with $s\in [-1,0)$. We will then adopt this extension throughout this paper. 

\begin{proposition}\label{pro: feature-coeff}
 For any $f\in \cH^s$ with $s\geq 0$, 
$\|f\|_{\cH^s}=\inf_{f=\Phi \alpha}\|\alpha\|_{\tH^{s-1}}$ and there exists a unique $a\in \tH^{s-1}$ such that $f=\Phi\alpha$ and $\|f\|_{\cH^s}=\|a\|_{\tH^{s-1}}$.
\end{proposition}
\begin{proof}
Let $\alpha_f = \sum_{j} \langle f, e_j\rangle_\rho \mu_j^{-1/2} \te_j$. It follows from SVD~\eqref{eqn: svd} that $\Phi \alpha_f = f$ and $\|\alpha_f\|_{\tH^{s-1}}^2 = \sum_j \mu_j^{-s}\langle f, e_j\rangle_\rho^2=\|f\|_{\cH^s}^2$. This justifies the existence. 
Suppose $\alpha=\sum_j a_j\te_j\in\tH^{s-1}$ such that $\Phi \alpha=\sum_j \mu_j^{1/2} a_j e_j=f$. 
Since   $\mu_j>0$ for all $j\in \NN$, we must have $a_j = \langle f, e_j\rangle_\rho \mu_j^{-1/2}$ for all $j\in \NN$. Consequently, $\alpha=\alpha_f$, which justifies the uniqueness.
\end{proof}
This lemma characterizes the interpolation space $\cH^s$ using the feature operator $\Phi$. The special case of $s=1$, which corresponds to the RKHS, has already been established in \cite{rahimi2008uniform}. 

\begin{remark}
We stress that the random feature characterization of interpolation spaces introduced above is not essential for our subsequent proofs. Similar results can likely be obtained using the techniques developed in \cite{fischer2020sobolev}, in conjunction with our proposed refined degrees of freedom in Definition~\ref{def: degree-Fp}. Nevertheless, we find this perspective both intuitive and insightful for establishing connections with random feature approximations (see Propositions \ref{pro: feature-coeff} and \ref{pro: equivalence}). This connection may be helpful for a more refined analysis of the efficacy of random features, particularly in the noiseless regime as discussed in Section~\ref{sec:concluding_remark}.
\end{remark}

\section{The Bias-Variance Decomposition for KRR Solutions}
\label{sec:the_krr_solution}
In this section, we shall leverage the feature operator $\Phi$ and its discrete version to express the KRR solution and the associated learning error. To this end, we define the discrete feature operator $\cS_n: L^2(\Omega,\pi)\mapsto\RR^n$ and  its adjoint operator $\cS_n^*: \RR^n \mapsto L^2(\Omega, \pi)$ by 
\[
(\cS_n \alpha)_i = \frac{1}{\sqrt{nq(x_i)}}\int_\Omega \alpha(\omega) \phi(x_i,\omega) \dd\pi(\omega),\quad
\cS_n^* a = \sum_{i=1}^n a_i \frac{1}{\sqrt{nq(x_i)}} \phi(x_i,\cdot).
\]
Then, we have
\begin{equation}\label{eqn: empirical-kernel}
\begin{aligned}
   K_n:=\cS_n\cS_n^* &=  \left(\frac{1}{n\sqrt{q(x_i)q(x_j)}}k(x_i,x_j)\right)_{1\leq i,j\leq n}\\ 
   \hT_n:=\cS_n^*\cS_n &=\fn\sum_{i=1}^n \frac{1}{q(x_i)}\phi(x_i,\cdot)\otimes \phi(x_i,\cdot)
\end{aligned}.
\end{equation}
Note that $K_n\in\RR^{n\times n}$ is the kernel matrix and $\hT_n$ is an empirical approximations of $\cT$. Moreover, 
\begin{equation}\label{eqn: 22}
\cS_n \alpha^* =  \frac{1}{\sqrt{n}}\left(\frac{f^*(x_1)}{\sqrt{q(x_1)}},\cdots, \frac{f^*(x_n)}{\sqrt{q(x_n)}}\right)^\top \in\RR^n
\end{equation}
and for any $a\in\RR^n$, 
\begin{align}\label{eqn: 11}
\notag    \Phi\cS_n^* a &= \Phi \left(\sum_{i=1}^n a_i \frac{1}{\sqrt{nq(x_i)}}\phi(x_i,\cdot)\right) = \sum_{i=1}^n a_i \frac{1}{\sqrt{n q(x_i)}}\int_{\Omega} \phi(\cdot,\omega)\phi(x_i,\omega)\dd\pi(\omega)\\ 
    &= \sum_{i=1}^n a_i \frac{1}{\sqrt{n q(x_i)}}k(x_i,\cdot) = k_n(\cdot)^\top a.
\end{align}

\paragraph*{The bias-variance decomposition.}
Recall that 
$\hf_\lambda  = \hk_n(\cdot)^\top \hat{b}_\lambda$ for $\lambda>0$ 
with 
\begin{equation}\label{eqn: krr-a-1}
\hat{b}_\lambda = (K_n + \lambda I_n )^{-1} \hy. 
\end{equation}
Noting $(\hy)_i = \frac{1}{\sqrt{nq(x_i)}}(f^*(x_i) + \xi_i)$ for $i\in [n]$, we have 
\[
	\hb_\lambda = \underbrace{(K_n + \lambda I_n )^{-1}S_n\alpha^*}_{\hb_{\lambda}^{\text{bias}}} + \underbrace{(K_n + \lambda I_n )^{-1} Q\hat{\xi}}_{\hb_{\lambda}^{\text{var}}},
\]
where
\[
	\hat{\xi}=\frac{1}{\sqrt{n}}(\xi_1,\dots,\xi_n)^\top\in\RR^n, \qquad Q=\diag(q(x_1)^{-1/2},\dots, q(x_n)^{-1/2})\in\RR^{n\times n}.
\]
Then, we have the bias-variance decomposition:
\begin{align}
\hf_\lambda = 	\hf^{\text{bias}}_\lambda + \hf^{\text{var}}_\lambda,
\end{align}
where the two terms are given by 
\begin{itemize}
\item {\bf The bias term}, applying \eqref{eqn: 11}, \eqref{eqn: krr-a-1}, \eqref{matrix_inverse_lemma}, and the fact $\cS_n\cS_n^*=K_n$, is given by
\begin{align}\label{eqn: bias-term}
\notag    \hf^{\text{bias}}_\lambda &=\Phi \cS_n^*(K_n + \lambda I_n )^{-1}\cS_n\alpha^*= \Phi\cS_n^*(\cS_n\cS_n^* + \lambda)^{-1}\cS_n\alpha^* \\ 
&\stackrel{\text{}}{=}\Phi(\cS_n^*\cS_n + \lambda)^{-1}\cS_n^*\cS_n \alpha^* = \Phi (\hT_n + \lambda)^{-1} \hT_n \alpha^*;
\end{align}
\item {\bf The variance term} is determined by the label noise:
\begin{align}\label{eqn: variance-term}
\hf^{\text{var}}_\lambda = \Phi \cS_n^*\left(K_n+\lambda\right)^{-1}Q\hat{\xi}.
\end{align}
\end{itemize}

Then, we have the total {\bf error decomposition}:
\begin{align}\label{eqn: bias-variance decomposition}
\|\hf_\lambda - f^*\|_{\cH^p}^2 \leq 2\left(\|\hf^{\text{bias}}_\lambda-f^*\|_{\cH^p}^2 + \|\hf^{\text{var}}_\lambda\|_{\cH^p}^2\right).
\end{align}

\begin{remark}
Note that the bias term corresponds to the estimation error of noiseless KRR, which is our main focus, except in Section~\ref{sub:implication_for_noisy_krr}.
Consequently, in most of our statements, we simply write
\(\hat{f}_\lambda = \hat{f}_\lambda^{\mathrm{bias}}\)
without introducing any ambiguity.
\end{remark}

We next show that the bias term is equivalent to the error arising from certain random feature approximations.
This perspective, inspired by the abstract dual equivalence  developed in \cite{chen2023duality}, provides useful insights into noiseless KRR.

\begin{lemma}[Random feature approximation]
For any  $\alpha^* \in \tH^{s}$ with $s\geq -1$ and $\lambda>0$, let 
$$
    \hat{\alpha}_\lambda = \cA_{n,\lambda}\alpha^*, \,\,\text{ where }\,  \cA_{n,\lambda}=(\hT_n+\lambda)^{-1}\hT_n.
$$ 
Then, $\hat{\alpha}_\lambda$
is the optimal norm-controlled approximation of $\alpha^*$ in the subspace spanned by the weighted random features $\{q(x_i)^{-1/2}\phi(x_i,\cdot)\}_{i=1}^n$. Specifically,
\begin{equation}\label{eqn: 2}
\hat{\alpha}_\lambda = S_n^* \hat{a}_\lambda,\,\, \text{ where} \,\, \ha_\lambda =\argmin_{a\in\RR^n} \left(\left\|\alpha^* - S_n^*a\right\|_\pi^2 + \lambda\|a\|_2^2\right).
\end{equation}
\end{lemma}
Recalling that when $\alpha^* \in \tH^s$ with $s \geq -1$, then $f^* = \Phi\alpha^* \in L^2(\cX, \rho)$. Hence, $\cS_n\alpha^* = (\frac{f(x_1)}{\sqrt{nq(x_1)}},\dots,\frac{f(x_n)}{\sqrt{nq(x_n)}})^\top$ is well-defined. Hence, $\hat{\alpha}_\lambda$ is well-defined for any $\alpha^* \in \tH^{s}$ with $s\geq -1$. 

\begin{proof}
Noticing $
J(a):=\|\alpha^* - \cS_n^* a\|_\pi^2 + \lambda \|a\|_2^2
$
 is quadratic with respect to $a\in(\RR^n,\|\cdot\|_2)$, a simple calculation gives
\begin{equation}\label{eqn: optimal-a-for-RF}
    \ha_\lambda = (\cS_n\cS_n^*+\lambda)^{-1}\cS_n\alpha^*.
\end{equation}
Thus, the optimal random feature approximation is given by 
\begin{align*}
    \hat{\alpha}_\lambda &= \cS_n^* \ha_\lambda = \cS_n^*(\cS_n\cS_n^*+\lambda)^{-1}\cS_n\alpha^*\\ 
    &\stackrel{\text{(i)}}{=}(\cS_n^*\cS_n+\lambda)^{-1}\cS_n^*\cS_n\alpha\stackrel{\text{(ii)}}{=} (\hT_n+\lambda)^{-1}\hT_n \alpha^* = \cA_{n,\lambda}\alpha^*,
\end{align*}
where (i) and (ii) follow from \eqref{matrix_inverse_lemma} and \eqref{eqn: empirical-kernel}, respectively.
\end{proof}

The next theorem establishes an equivalence between noiseless KRR and random feature approximation.
\begin{proposition}\label{pro: equivalence}
Let $f^*\in\cH^s$ with $s \ge 0$ and $f^* = \Phi \alpha^*$ with $\alpha^*\in\tH^{s-1}$. Then, $\hf_\lambda^{\mathrm{bias}} =\Phi \hat{\alpha}_\lambda$ and for any $p\in[0, s]$, 
\[
\|\hat{f}^{\mathrm{bias}}_\lambda-f^*\|_{\cH^p}^2 = \|\hat{\alpha}_\lambda-\alpha^*\|_{\tH^{p-1}}^2.
\]
\end{proposition}
\begin{proof}
By \eqref{eqn: bias-term}, we have
$
\hat{f}^{\text{bias}}_\lambda - f^* = \Phi(\hat{\alpha}_\lambda - \alpha^*)
$
and consequently,
\begin{align*}
    \|\hat{f}^{\mathrm{bias}}_\lambda - f^*\|^2_{\cH^p} &= \|\cT^{-\frac{p}{2}}\Phi (\hat{\alpha}_\lambda - \alpha^*)\|_{\rho}^2  
     = \left\langle \cT^{-\frac{p}{2}}\Phi (\hat{\alpha}_\lambda - \alpha^*),\cT^{-\frac{p}{2}}\Phi (\hat{\alpha}_\lambda - \alpha^*)\right\rangle_\rho  \\
    &= \left\langle \hat{\alpha}_\lambda - \alpha^*,\Phi^*\cT^{-p}\Phi (\hat{\alpha}_\lambda - \alpha^*)\right\rangle_\pi  
    \stackrel{(i)}{=}\left\langle \hat{\alpha}_\lambda - \alpha^*,\cT^{1-p} (\hat{\alpha}_\lambda - \alpha^*)\right\rangle_\pi  \\
    &= \|\cT^{-\frac{p-1}{2}}(\hat{\alpha}_\lambda - \alpha^*)\|_\pi^2 
    = \|\hat{\alpha}_\lambda - \alpha^*\|_{\tH^{p-1}}^2,
\end{align*}
where $(i)$ is due to 
$
   \Phi^*\cT^{-p}\Phi = \Phi^*(\Phi\Phi^*)^{-p}\Phi = (\Phi^*\Phi)^{-p}\Phi^*\Phi = \cT^{1-p} 
$
\end{proof}

Combining  Propositions \ref{pro: feature-coeff} and  \ref{pro: equivalence}, we see that for any $s \ge 0$, the estimation of noiseless KRR  for the  target function $f^*\in \cH^s$, measured in the $\cH^p$-norm is equivalent to the random feature approximation for the target function $\alpha^* \in \tH^{s-1}$, measured in the $\tH^{p-1}$-norm.

\section{Proofs of Main Theorems}
Our proofs will leverage the random feature perspective of $\cH^s$ and the equivalence between noiseless KRR and random feature approximation established in Theorem~\ref{pro: equivalence}.

\subsection{Proof of Theorem~\ref{thm:lower_bound}}
\begin{proof}
    First, using the Proposition 4 in \cite{chen2023duality}, we have that
    \begin{equation}\label{eqn: i-complexity-lower-bound}
        \inf_{\cQ\in \cM_n}\sup_{\|f\|_{\cH^s} \le 1}\|\cQ f - f\|_{\cH^p} \ge \sup_{\|f\|_{\cH^s \le 1, f(x_1) = \dots = f(x_n) = 0}}\|f\|_{\cH^p}.
    \end{equation}
    We then proceed by considering the subspace spanned by the first $n+1$ eigenfunctions: 
    $f = \sum_{j=1}^{n+1}\mu_j^{\frac{s}{2}} a_j e_j$, with $a = (a_1,\dots,a_{n+1}) \in \RR^{n+1}$.  Within this subspace, 
    \[\|f\|_{\cH^s}^2=\|a\|_2^2,\qquad \|f\|_{\cH^p}^2=\sum_{j=1}^{n+1}\mu_j^{s-p}a_j^2 = a^\top D a,
     \]
    where $D = \mathrm{diag}(\mu_1^{s-p},\dots,\mu_{n+1}^{s-p})$. The constraint $f(x_i)=0$ for all $i\in [n]$, becomes $f(x_i)=\sum_{j=1}^{n+1}\mu_j^{s/2}a_je_j(x_i)=0$ for all $i\in [n]$. Let $E=(E_{i,j}) \in \RR^{n\times(n+1)}$ with $E_{ij} = \mu_j^{\frac{s}{2}} e_j(x_i)$. Then, the constraint becomes $Ea=0$.  Plugging these into \eqref{eqn: i-complexity-lower-bound} gives
    \begin{equation*}
        \sup_{\|f\|_{\cH^s \le 1, f(x_1) = \dots = f(x_n) = 0}}\|f\|_{\cH^p}^2 \ge \sup_{\|a\|_2 \le 1, Ea = 0} a^\top D a.
    \end{equation*}
    Finally, through Courant minimax principle \cite{courant2008methods}, we have that
    \begin{equation*}
        \sup_{\|a\| \le 1, Ea = 0} a^\top D a \ge \inf_{E} \sup_{\|a\| \le 1, Ea = 0} a^\top D a= \mu_{n+1}^{s-p}.
    \end{equation*}
\end{proof}

\subsection{Proof of Theorem \ref{thm:krr_error}}
\label{sec: proof-theorem-krr}

We first need the following lemmas. 
\begin{lemma}\label{lemma: minimum-norm}
Let $p \in [0,1]$ and $s \ge 1$. For any $f^*=\Phi\alpha^* \in \cH^s$ and $\alpha^* \in \cH^{s-1}$, $\lim_{\lambda \rightarrow 0^+} \|\hat{f}_\lambda-\hat{f}_0\|_{\cH^p} = 0$.
\end{lemma}

\begin{proof}
 Noticing that $\hat{y} = K_n\hb_0$, we know that
\begin{align*}
    \| \hat{f}_\lambda-\hat{f}_0 \|_{\cH}^2 &= (\hb_\lambda-\hb_0)^\top K_n (\hb_\lambda-\hb_0)\\ 
    &=((K_n + \lambda I_n)^{-1}K_n \hb_0-\hb_0)^\top K_n ((K_n + \lambda I_n)^{-1}K_n \hb_0-\hb_0) \\
    &=\lambda^2\hb_0^\top(K_n + \lambda I_n)^{-1}K_n(K_n + \lambda I_n)^{-1}\hb_0\\
    &\le \lambda \|\hb_0\|^2\|\lambda(K_n+\lambda I_n)^{-1}K_n(K_n+\lambda I_n)^{-1}\|\\
    &\le\lambda \|\hb_0\|^2 \sup_{z\geq 0}\frac{z\lambda}{(z+\lambda)^2}\leq \frac{1}{4}\lambda \|\hb_0\|^2.
\end{align*}
By Definition~\ref{def: Hks-space}, we have
\[
\|\hf_\lambda - \hf_0\|_{\cH^p} =\|\cT^{\frac{1-p}{2}}(\hf_\lambda - \hf_0)\|_\cH \le \|\cT^{\frac{1-p}{2}}\|\|\hf_\lambda-\hf_0\|_{\cH}\leq \|\cT^{\frac{1-p}{2}}\|\frac{\sqrt{\lambda}}{2}\|\hb_0\|.
\]
Therefore, $\lim_{\lambda\to 0^+}\|\hf_\lambda - \hf_0\|_{\cH^p} =0$. 
\end{proof}

\begin{lemma}\label{lemma: dof-rf}
Let $F_\gamma$ be the  DoF given in Definition \ref{def: degree-Fp}. Then, for any $\lambda>0$, it holds
\begin{align}\label{eqn: dof-feature}
F_\gamma(\lambda) = \esssup_{x \in \cX}\frac{1}{q(x)}\left\|\cT^{\frac{\gamma-1}{2}}(\cT+\lambda)^{-\frac{\gamma}{2}}\phi(x,\cdot)\right\|_\pi^2.
\end{align}
\end{lemma}
\begin{proof}
Using the singular value decomposition \eqref{eqn: svd}, we have 
\begin{align*}
\|\cT^{\frac{\gamma-1}{2}}(\cT+\lambda)^{-\frac{\gamma}{2}}\phi(x,\cdot)\|_\pi^2 &= \|\cT^{\frac{\gamma-1}{2}}(\cT+\lambda )^{-\frac{s}{2}}\sum_j \mu_j^{\half}e_j(x)\te_j \|_\pi^2\\ 
&=\Big\|\sum_j \mu_j^{\half}e_j(x)\frac{\mu_j^{\frac{\gamma-1}{2}}}{(\mu_j+\lambda)^{\frac{\gamma}{2}}} \te_j\Big\|_\pi^2\\ 
&=\sum_j \frac{\mu_j^\gamma}{(\mu_j+\lambda)^\gamma}e_j^2(x).
\end{align*}
Further applying \eqref{eqn: F-def} completes the proof.
\end{proof}

In the sequel, we shall use \eqref{eqn: dof-feature} to studying the maximal $\gamma$-DoF.
The key step for the proof of Theorem \ref{thm:krr_error} is the following lemma:
\begin{lemma}[Concentration of the empirical operator]\label{lem1}
Let $n\in \NN^+$, $0\leq \gamma \le 1$,  $\delta \in (0,1)$. If $\lambda\geq \lambda_n:=\Lambda_\gamma(n,\delta)$.
Then, with probability at least $1-\delta$, it holds that
    \begin{equation}\label{eqn: norm-concentra}
        \left\|\cT^{\frac{\gamma-1}{2}}(\cT+\lambda )^{-\frac{\gamma}{2}}(\hT_n-\cT)(\cT+\lambda )^{-\frac{\gamma}{2}}\cT^{\frac{\gamma-1}{2}}\right\|\le \frac{3}{4}.
    \end{equation}
\end{lemma}

\begin{proof}
Let $\Delta_n = \cT^{\frac{\gamma-1}{2}}(\cT+\lambda )^{-\frac{\gamma}{2}}(\hT_n-\cT)(\cT+\lambda )^{-\frac{\gamma}{2}}\cT^{\frac{\gamma-1}{2}}$. Noticing
\[
	\hT_n =\frac{1}{n}\sum_{i=1}^n \frac{1}{q(x_i)}\phi(x_i,\cdot)\otimes \phi(x_i,\cdot),\quad \cT = \EE_{x\sim \rho'}\left[\frac{1}{q(x)}\phi(x,\cdot)\otimes \phi(x,\cdot)\right].
\]
We shall bound $\|\Delta_n\|$ by applying concentration inequality for random self-adjoint PSD operators over Hilbert spaces (see Theorem~\ref{thm: concentration}).

Let $h_i = \cT^{\frac{\gamma-1}{2}}(\cT+\lambda )^{-\frac{\gamma}{2}}\phi(x_i,\cdot)\in L^2(\Omega, \pi)$ and 
$
    Z_i  = \frac{1}{n q(x_i)}  h_i\otimes h_i - \frac{1}{n}\cT^\gamma(\cT+\lambda )^{-\gamma}.
$
Then, $Z_i:L^2(\Omega, \pi)\mapsto L^2(\Omega,\pi)$ for $i\in [n]$ and $\{Z_i\}_{i=1}^n$ are \iid random self-adjoint operators, satisfying $\EE[Z_i]=0$ and
\begin{align}\label{eqn: 99}
\notag    \sum_{i=1}^n Z_i &=  \cT^{\frac{\gamma-1}{2}}(\cT+\lambda )^{-\frac{\gamma}{2}}\hT_n (\cT+\lambda )^{-\frac{\gamma}{2}}  \cT^{\frac{\gamma-1}{2}}- \cT^\gamma(\cT+\lambda )^{-\gamma} =\Delta_n
\end{align}

To bound the sum $\sumin Z_i$, we need  some moment estimates. By noting $\|h_i\|^2_\pi=\|\cT^{\frac{\gamma-1}{2}}(\cT+\lambda )^{-\frac{\gamma}{2}} \phi(x_i,\cdot)\|_\pi^2$ and Lemma~\ref{lemma: dof-rf}, we have
\begin{align*}
    \|Z_i\| &\le \frac{1}{n}\max\left\{\frac{1}{q(x_i)}\|h_i\|_\pi^2, \|\cT^\gamma(\cT+\lambda )^{-\gamma}\|\right\}\\ 
    &= \frac{1}{n}\max\left\{\frac{1}{q(x_i)}\|h_i\|_\pi^2, \frac{\mu_1^\gamma}{(\mu_1+\lambda)^\gamma}\right\}\le \frac{1}{n} F_\gamma(\lambda)
\end{align*}
and 
\begin{align*}
    \EE Z_i^2&\lem \frac{1}{n^2}\EE\left[\frac{1}{q(x_i)^2}\|h_i\|^2 h_i\otimes h_i\right] 
    \lem \frac{1}{n^2} F_\gamma(\lambda)\EE\left[\frac{1}{q(x_i)}  h_i\otimes h_i\right] \\ 
    &= \frac{1}{n^2} F_\gamma(\lambda)\cT^\gamma(\cT+\lambda)^{-\gamma}, 
\end{align*}
which implies
\begin{align*}
\left\|\sum_{i=1}^n\EE Z_i^2\right\| &\le \frac{F_\gamma(\lambda)}{n}\|\cT^\gamma(\cT+\lambda)^{-\gamma}\|\leq \frac{F_\gamma(\lambda)}{n} \\ 
\tr\left(\sum_{i=1}^n\EE Z_i^2\right) &\le \frac{F_\gamma(\lambda)}{n}\sum_{i=1}^\infty \frac{\mu_j^\gamma}{(\mu_j+\lambda)^\gamma} \stackrel{}{\le} \frac{F_\gamma(\lambda)^2}{n}.
\end{align*}
Then, it follows from Theorem \ref{thm: concentration} that for any $t$ satisfying $3nt^2 \ge 3F_\gamma(\lambda) + F_\gamma(\lambda) t $:
\begin{align*}
    \PP\Big(\|\Delta_n\|\ge t\Big)=\PP\left(\left\|\sum_{i=1}^n Z_i\right\|\ge t\right)\le 14F_\gamma(\lambda)\exp\left(-\frac{3nt^2}{F_\gamma(\lambda)(6+2t)}\right),
\end{align*}
Taking $t = 3/4$, it is easy to verify that when  $\lambda\ge \lambda_n$, it follows that 1) $3nt^2 \ge 3F_\gamma(\lambda) + F_\gamma(\lambda) t $; 2) $ 14F_\gamma(\lambda)\exp(-3nt^2/(F_\gamma(\lambda)(6+2t)) \le \delta$. Therefore,  $\|\Delta_n\|\leq 3/4$ holds with probability at least $1-\delta$.
\end{proof}

\begin{lemma}\label{lemma: 11} 
For any $n\in \NN^+$ and $\gamma\in (0,1)$, if $\lambda$ is sufficiently large  to guarantee that \eqref{eqn: norm-concentra} holds,
then we have for any $r \in [0,\half]$ that
    \begin{equation}\label{eq:norm_bound}
        \|(\hT_n + \lambda )^{-r}\cT^r\| \le \|(\hT_n+\lambda)^{-r}(\cT+\lambda)^r\| \le  4^r.
    \end{equation}
\end{lemma}
\begin{proof}
Noting 
\begin{align*}
\|(\hT_n + \lambda )^{-r}\cT^r\| &\le \|(\hT_n+\lambda)^{-r}(\cT+\lambda)^r (\cT+\lambda)^{-r}\cT^r\| \\ 
&\leq \|(\hT_n+\lambda)^{-r}(\cT+\lambda)^r \|\|(\cT+\lambda)^{-r}\cT^r\|\\ 
&\leq \|(\hT_n+\lambda)^{-r}(\cT+\lambda)^r \|,
\end{align*}
we only need to prove the second inequality in \eqref{eq:norm_bound}:

We first consider  the case of $r=\half$. Noting that  \eqref{eqn: norm-concentra} implies that $\cT^{\frac{\gamma-1}{2}}(\cT+\lambda )^{-\frac{\gamma}{2}}(\cT-\hT_n)(\cT+\lambda )^{-\frac{\gamma}{2}}\cT^{\frac{\gamma-1}{2}} \lem \frac{3}{4} $, we have
\begin{equation*}
    \cT - \hT_n \lem \frac{3}{4} \cT^{1-\gamma}(\cT+\lambda )^{\gamma}.
\end{equation*}
Hence,
\begin{align*}
    \hT_n + \lambda  &= \cT + \lambda  - (\cT - \hT_n)\\ 
    &\gem \cT + \lambda  - \frac{3}{4}\cT^{1-\gamma}(\cT+\lambda)^\gamma \\ 
    &\gem \frac{1}{4}(\cT + \lambda ) +\frac{3}{4}\left((\cT + \lambda )- \cT^{1-\gamma}(\cT+\lambda)^\gamma \right)\\
    &\gem \frac{1}{4}(\cT+\lambda),
\end{align*}
where the last inequality follows from the observation that 
\[
    (\cT + \lambda )- \cT^{1-\gamma}(\cT+\lambda)^\gamma = \sum_{j} \left(\mu_j+\lambda - \mu_j^{1-\gamma}(\mu_j+\lambda)^\gamma\right)\te_j\otimes \te_j\gem 0,
\]
where we use the inequality: $(a+\lambda)-a^{1-\gamma}(a+\lambda)^\gamma\geq 0$ for any $a\geq 0$.
Therefore, 
$$
(\hT_n + \lambda )^{-\frac{1}{2}}(\cT+\lambda )(\hT_n + \lambda )^{-\frac{1}{2}} \lem 4 I
$$
and hence 
\begin{align}\label{eqn: s-half}
\|(\hT_n + \lambda )^{-\frac{1}{2}}(\cT+\lambda )^{\half}\| = \sqrt{\|(\hT_n + \lambda )^{-\frac{1}{2}}(\cT+\lambda )(\hT_n + \lambda )^{-\frac{1}{2}}\|} \le 2.
\end{align}

For the case of $r\in [0,\half)$, we have
\begin{align*}
     \|(\hT_n+\lambda)^{-r}(\cT+\lambda)^r\| &= \|((\hT_n+\lambda )^{-\half})^{2r}((\cT+\lambda)^{\half})^{2r}\| \\ 
     &\stackrel{(i)}{\leq}\|(\hT_n+\lambda )^{-\half} (\cT+\lambda)^{\half}\|^{2r}\\
     &\stackrel{(ii)}{\leq} 4^r,
\end{align*}
where $(i)$ uses the Cordes inequality \eqref{ineq:cordes} and $(ii)$ is due to \eqref{eqn: s-half}.
\end{proof}

The following lemma establishes that controlling the uniform estimation error in noiseless KRR is equivalent to bounding a certain operator norm.
\begin{lemma}[Bias term]\label{lemma: uniform-operaotr-bound}
$
     \sup_{\|f^*\|_{\cH^s}\le 1}\|\hf_{\lambda} - f^*\|_{\cH^p} = \lambda\|\cT^{\frac{1-p}{2}}(\hT_n+\lambda)^{-1}\cT^{\frac{s-1}{2}} \|.
$
\end{lemma}
\begin{proof}
Using the equivalence between KRR and random feature approximation in Proposition \ref{pro: equivalence}, we have
\begin{equation*}
    \sup_{\|f^*\|_{\cH^s}\le 1}\|\hf_{\lambda} - f^*\|_{\cH^p} = \sup_{\|\alpha^*\|_{\cH^{s-1}}\le 1}\|\hat{\alpha}_{\lambda}-\alpha^*\|_{\cH^{p-1}}.
\end{equation*}
 Recalling that $\hat{\alpha}_{\lambda} = \cA_{n,\lambda}\alpha^*$ and 
 $\cA_{n,\lambda}=(\hT_n+\lambda )^{-1}\hT_n =  - \lambda (\hT_n+\lambda )^{-1}$, we have 
\begin{align}\label{eq: closed_form}
\notag	 \sup_{\|\alpha^*\|_{\cH^{s-1}}\le 1}\|\hat{\alpha}_{\lambda}-\alpha^*\|_{\cH^{p-1}} &= \sup_{\|\alpha^*\|_{\tilde{\cH}^{s-1}} \le 1}\|\cT^{\frac{1-p}{2}}[\cA_{n,\lambda} \alpha^* - \alpha^*]\|_\pi\\ 
\notag &= \lambda  \sup_{\|\alpha^*\|_{\tilde{\cH}^{s-1}} \le 1}\|\cT^{\frac{1-p}{2}}(\hT_n + \lambda )^{-1}\alpha^* \|_\pi\\ 
\notag &\stackrel{(i)}{=} \lambda\sup_{\|g\|_\pi \le 1}\|\cT^{\frac{1-p}{2}}(\hT_n+\lambda)^{-1}\cT^{\frac{s-1}{2}} g \|_\pi  \\
& = \lambda\|\cT^{\frac{1-p}{2}}(\hT_n+\lambda)^{-1}\cT^{\frac{s-1}{2}} \|.
\end{align}
\end{proof}

\begin{proof}[Proof of Theorem \ref{thm:krr_error}]

We only need to bound the operator norm in  Lemma~\ref{lemma: uniform-operaotr-bound}.
\paragraph*{\underline{Case I}: $s\in [1,2]$ and $p \in [0,1]$.}
For any $0 < \lambda \le \lambda_n$,
\begin{align*}
&\lambda\|\cT^{\frac{1-p}{2}}(\hT_n+\lambda)^{-1}\cT^{\frac{s-1}{2}} \|\\
     &\le  \lambda \|\cT^{\frac{1-p}{2}}(\hT_n+\lambda_n )^{\frac{p-1}{2}}\|\|(\hT_n+\lambda_n)^{\frac{1-p}{2}}(\hT_n+\lambda)^{\frac{p-1}{2}}\|\|(\hT_n + \lambda )^{\frac{s-p}{2}-1}\| \\
     &\quad\quad \times\|(\hT_n+\lambda)^{\frac{1-s}{2}}(\hT_n+\lambda_n)^{\frac{s-1}{2}}\|\|(\hT_n + \lambda_n )^{\frac{1-s}{2}}\cT^{\frac{s-1}{2}}\|\\
     &\leq \lambda  4^{\frac{1-p}{2}} \left(\frac{\lambda_n}{\lambda}\right)^{\frac{1-p}{2}}\lambda^{\frac{s-p}{2}-1}\left(\frac{\lambda_n}{\lambda}\right)^{\frac{s-1}{2}}4^{\frac{s-1}{2}}\\
     &\le 4\lambda_n^{\frac{s-p}{2}},
\end{align*} 
where the second step uses Lemmas~\ref{lemma: 11} and~\ref{lemma: AB-bound}.

When $\lambda=0$, we can apply Lemma \ref{lemma: minimum-norm}  to obtain the same conclusion as follows:
\begin{align*}
    \sup_{\|f^*\|_{\cH^s}\le 1}\|\hf_0 - f^*\|_{\cH^p} &= \sup_{\|f^*\|_{\cH^s}\le 1}\lim_{\lambda \rightarrow 0^+}\|\hf_{\lambda} - f^*\|_{\cH^p} \leq \limsup_{\lambda \rightarrow 0^+} \sup_{\|f^*\|_{\cH^s}\le 1}\|\hf_{\lambda} - f^*\|_{\cH^p} \\ 
    &\le 4\lambda_n^{\frac{s-p}{2}},
\end{align*}

\paragraph*{\underline{Case II}: $0 \le p \le s < 1$.} We only consider the case where $\lambda \geq \lambda_n$. Notice that
\begin{equation*}
      \lambda\|\cT^{\frac{1-p}{2}}(\hT_n+\lambda )^{-1}\cT^{\frac{s-1}{2}} \| \le \bM_1 + \bM_2, 
\end{equation*}
where $\bM_1 = \lambda\|\cT^{\frac{1-p}{2}}(\cT+\lambda )^{-1}\cT^{\frac{s-1}{2}} \|$ and $\bM_2 =  \lambda\|\cT^{\frac{1-p}{2}}[(\hT_n + \lambda)^{-1}-(\cT+\lambda )^{-1}]\cT^{\frac{s-1}{2}} \|$. 
\begin{itemize}
\item 
For $\bM_1$, we have that
\begin{align}\label{eqn: Q09}
 \notag  \bM_1 &= \lambda \|(\cT+\lambda)^{-1}\cT^{\frac{s-p}{2}}\| \le \lambda \|(\cT+\lambda)^{-1+\frac{s-p}{2}}\|\|(\cT+\lambda)^{-\frac{s-p}{2}}\cT^{\frac{s-p}{2}}\|\\ 
    & \le \lambda \|(\cT+\lambda)^{-1+\frac{s-p}{2}}\|\leq \lambda^{\frac{s-p}{2}}.
\end{align}

\item For $\bM_2$, applying $A^{-1} - B^{-1} = A^{-1}(B-A)B^{-1}$  gives
\begin{align}\label{eqn: Q10}
  \notag    \bM_2 &= \lambda\|\cT^{\frac{1-p}{2}}(\hT_n + \lambda)^{-1}[\hT_n-\cT](\cT+\lambda_n )^{-1}\cT^{\frac{s-1}{2}} \| \\
  \notag    &\le \lambda\|\cT^{\frac{1-p}{2}}(\hT_n+\lambda)^{-1}(\cT+\lambda )^{\frac{s}{2}}\cT^{\frac{1-s}{2}}\| \\
   \notag   &\quad \times\|\cT^{\frac{s-1}{2}}(\cT+\lambda )^{-\frac{s}{2}}(\hT_n-\cT)(\cT+\lambda )^{-\frac{s}{2}}\cT^{\frac{s-1}{2}}\|\\ 
   \notag   &\quad \times \|\cT^{-\frac{s-1}{2}}(\cT+\lambda)^{-1+\frac{s}{2}}\cT^{\frac{s-1}{2}}\| \\
  \notag    &\stackrel{(i)}{\le}\frac{3}{4}\lambda^{\frac{s}{2}}\|\cT^{\frac{1-p}{2}}(\hT_n+\lambda)^{-\frac{1-p}{2}}\|\|(\hT_n+\lambda)^{-\frac{p}{2}}\|\\
  \notag    &\quad \times\|(\hT_n+\lambda)^{-\half}(\cT+\lambda)^{\half}\|\|(\cT+\lambda)^{-\half}(\cT+\lambda )^{\frac{s}{2}}\cT^{\frac{1-s}{2}}\|\\
    &\stackrel{(ii)}{\le} 3\lambda^{\frac{s-p}{2}},
\end{align}
where  $(i)$  follows from Lemma~\ref{lem1} (using the condition that $\lambda\geq \Lambda_s(n,\delta)$) and $(ii)$ follows from Lemma~\ref{lemma: 11}. Combining \eqref{eqn: Q09} and \eqref{eqn: Q10}, it holds for  $0\leq p\leq s< 1$ that 
\begin{equation}\label{eq: small_s_1}
     \lambda\|\cT^{\frac{1-p}{2}}(\hT_n+\lambda )^{-1}\cT^{\frac{s-1}{2}} \|  \le 4\lambda^{\frac{s-p}{2}}.
\end{equation}
\end{itemize}

\paragraph*{\underline{Case III}: $1< p \le s \le 2$.} Take $s' = 2-p$ and $p' = 2 - s $. Then $0 \le p' \le s' < 1$ and the condition $\lambda\geq \Lambda_{2-p}(n,\delta)$ becomes $\lambda\geq \Lambda_{s'}(n,\delta)$. Therefore, we can apply the result in Case II to obtain
\begin{align*}
    \lambda\|\cT^{\frac{1-p}{2}}(\hT_n+\lambda )^{-1}\cT^{\frac{s-1}{2}} \| &= \lambda\|\cT^{\frac{s'-1}{2}}(\hT_n+\lambda )^{-1}\cT^{\frac{1-p'}{2}} \| = \lambda\|\cT^{\frac{1-p'}{2}}(\hT_n+\lambda)^{-1}\cT^{\frac{s'-1}{2}} \| \\
    &\le 4\lambda^{\frac{s'-p'}{2}} = 4\lambda^{\frac{s-p}{2}}.
\end{align*}
\end{proof}

\subsection{Proof of Theorem~\ref{thm:krr_error-one}}
    Let $f^* = \sum_{j=1}^\infty \mu_j^{\frac{s}{2}} a_j e_j$ with $\sum_{j=1}^\infty a_j^2  \le 1$.
     Let $m(\lambda) = \max\{ i \in \NN^+, \mu_i \ge \lambda\}$ and consider the  following {\bf frequency decomposition}: 
    \begin{equation}\label{eqn: target-decomposition}
        f^* = \underbrace{\sum_{j=1}^{m(\lambda)} \mu_j^{\frac{s}{2}} a_j e_j}_{h^*} \quad +\quad \underbrace{\sum_{j=m(\lambda)+1}^{\infty} \mu_j^{\frac{s}{2}} a_j e_j}_{g^*},
    \end{equation}
    where $h^*$ and $g^*$ represent the low- and high-frequency components, respectively. 
    We use $\hh_\lambda$ and $\hg_\lambda$ denote the KRR estimates of $h^*$ and $g^*$, respectively. Using the linearity of KRR (see Eq.~\eqref{eqn: krr-a}), we have 
    \[
        \hf_\lambda = \hh_\lambda + \hg_\lambda.
    \]
    Then, the error can be bounded as follows:
     \begin{align}\label{eqn: error-decomposition}
     \|\hat{f}_{\lambda} - f^*\|_{\cH^p} &\le \|\hh_{\lambda} - h^*\|_{\cH^p} + \|\hg_{\lambda}-g^*\|_{\cH^p}.
    \end{align}

    Typically, the low-frequency component is often sufficiently smooth, while the high-frequency component  is rough but has low energy. Thus, we can bound the two terms on the right hand side of \eqref{eqn: error-decomposition} using distinct strategies that leverage each component's unique property.

\paragraph*{Definition of Events $E_1$ and $E_2$.}
To simplify the statement, we define the events:
\begin{itemize}
\item {\bf Event $E_1$: Concentration of empirical operator.}
\begin{align}
    E_1=\left\{(x_1,\dots,x_n): \left\|\cT^{\frac{\gamma-1}{2}}(\cT+\lambda )^{-\frac{\gamma}{2}}(\hT_n-\cT)(\cT+\lambda )^{-\frac{\gamma}{2}}\cT^{\frac{\gamma-1}{2}}\right\|\le \frac{3}{4}\right\}
\end{align}
Theorem~\ref{thm: concentration} along with the condition $\lambda\geq \Lambda_\gamma(n)$ guarantees that 
\[
    \PP\{E_1\}\geq 1-\delta_1.
\]

\item {\bf Event $E_2$: Bound on the labels of the high-frequency component.}
The following lemma ensures that the energy of $g^*$ is sufficiently small:
\begin{lemma}\label{lemma: d2}
Let $Y_g=(Y_1,\dots,Y_n)^\top \in \RR^n$ with $Y_i = \frac{1}{\sqrt{n}q(X_i)} g^*(X_i)$ and $X_i\sim \rho'$. Let 
\[
    E_{2}=\left\{(x_1,\dots,x_n): \|Y_g\|_2 \le C_{s,\gamma} \delta_2^{-\frac{\gamma-s}{2\gamma}} \lambda^{\frac{s}{2}}\right\}.
\]
Then, under the assumption of Theorem~\ref{thm:krr_error-one}, it holds
\[
    \PP\{E_2\}\geq 1-\delta_2.
\]
\end{lemma}
\end{itemize}
We also need the following lemma, which characterizes how the $\cH^p$ norm of a KRR estimate depends on the label's $\ell^2$ norm.
\begin{lemma}\label{lemma: Hp-norm-KRR2}
Let $\{(x_i,y_i)\}_{i=1}^n$ be the training data, $\hf_\lambda$ be the KRR estimate, and  
$$
\hy = \left(\frac{y_1}{\sqrt{nq(x_1)}}, \frac{y_2}{\sqrt{nq(x_2)}}, \cdots, \frac{y_n}{\sqrt{nq(x_n)}}\right)^\top \in\RR^n.
$$
Suppose  
$\sup_{r\in [0,\half]}\|\cT^{r}(\hT_n+\lambda )^{-r}\|\leq B$ holds for some $\lambda, B>0$.
Then for any $p\in [0,1]$, we have 
\[
\|\hf_\lambda\|_{\cH^p}\leq B\lambda^{-p/2}\|\hy\|_2.
\]
\end{lemma}

\paragraph*{The low-frequency component ($h^*$)}
The main observation is that the low-frequency component is sufficiently smooth, as its $\cH^\gamma$ norm is well-controlled:
    $$
        \|h^*\|_{\cH^{\gamma}}^2 = \sum_{j=1}^{m(\lambda)} \mu_j^{s-\gamma} a_j^2 \le \lambda^{s-\gamma} \sum_{j=1}^{m(\lambda)} a_j^2 \le \lambda^{s-\gamma}.
    $$
    The proof of Theorem~\ref{thm:krr_error} implies that under the event $E_1$, it follows that: 
    \begin{equation}\label{eq:thm45_3}
        \|\hat{h}_{\lambda}-h^*\|_{\cH^p} \le \|h^*\|_{\cH^{\gamma}}\sup_{\|f^*\|_{\cH^{\gamma}}\le 1}\|f^* - \hat{f}_{\lambda}\|_{\cH^p} \leq \lambda^{\frac{s-\gamma}{2} }\times 4 \lambda^{\frac{\gamma-p}{2}} \le 4 \lambda^{\frac{s-p}{2}}.
    \end{equation}

\paragraph*{The high-frequency component ($g^*$)}
To bound $\|\hg_\lambda-g^*\|_{\cH^p}$, we use 
\[
   \|\hg_\lambda-g^*\|_{\cH^p}\leq \|\hg_\lambda\|_{\cH^p} + \|g^*\|_{\cH^p}.
\]
\begin{itemize}
\item First, 
\begin{align}\label{eqn: 123}
        \|g^*\|_{\cH^p}^2 = \sum_{j=m(\lambda)+1}^\infty \mu_j^{s-p} a_j^2 \leq \lambda^{s-p}\sum_{j= m(\lambda)+1}^\infty a_j^2\leq \lambda^{s-p}.
\end{align}
\item By Lemma~\ref{lemma: 11}, the event $E_1$ implies that $\sup_{0\leq r\leq 1/2}\|\cT^r (\hT_n+\lambda)^{-r}\|\leq 2$. Thus, Lemma~\ref{lemma: Hp-norm-KRR2} implies that we have under $E_1\cap E_2$:
\begin{align}\label{eqn: 234}
    \|g_\lambda\|_{\cH^p}\leq 2\lambda^{-\frac{p}{2}} \|Y_g\|_2\leq C_{s,\gamma}\delta_2^{-\frac{\gamma-s}{2\gamma}}\lambda^{\frac{s-p}{2}}.
\end{align}
\end{itemize}

\paragraph*{Final error bound.}
Plugging \eqref{eq:thm45_3}, \eqref{eqn: 123}, and \eqref{eqn: 234} into \eqref{eqn: error-decomposition}, we have  under $E_1\cap E_2$:
    \begin{align*}
        \|\hat{f}_{\lambda} - f^*\|_{\cH^p} 
        &\leq \|\hh_{\lambda} - h^*\|_{\cH^p}+\|\hg_{\lambda}\|_{\cH^p} + \|g^*\|_{\cH^p}\\ 
        &\le 4\lambda^{\frac{s-p}{2}}+ C_{s,\gamma} \delta_2^{-\frac{\gamma-s}{2\gamma}}\lambda^{\frac{s-p}{2}} + \lambda^{\frac{s-p}{2}}=(5+C_{s,\gamma}\delta_2^{-\frac{\gamma-s}{2\gamma}})\lambda^{\frac{s-p}{2}}.
    \end{align*}
Moreover, applying union bound gives
\[
    \PP\{E_1\cap E_2\} =1 - \PP\{E_1^c \cup E_2^c\}\geq 1- (\PP\{E_1^c\} + \PP\{E_2^c\}) \geq 1- \delta_1-\delta_2.
\]
This completes the proof.
\qed

\subsection{Proof of Theorem~\ref{thm: noisy-krr}}
\label{sec:proof_of_theorem_ref_thm_noisy_krr}

By the error decomposition~\eqref{eqn: error-decomposition}, we have
\begin{align*}
\|\hf_\lambda - f^*\|_{\cH^p}^2 \lesssim \|\hf^{\text{bias}}_\lambda-f^*\|_{\cH^p}^2 + \|\hf^{\text{var}}_\lambda\|_{\cH^p}^2.
\end{align*}
Note that the bias term has already been bounded in Theorem~\ref{thm:krr_error} and Theorem~\ref{thm:krr_error-one} under various conditions. The proof is then completed by employing the following bound on the variance term:
\begin{proposition}\label{pro:variance_estiamte}
Suppose Assumptions~\ref{assumption: noise} and \ref{assumption: weighted sampling} hold. 
  Then for any $p \in [1,2)$, for any $\delta\in (0,1)$, \wp at least $1-\delta$ we have
\begin{equation*}
    \left\|\hf^{\mathrm{var}}_\lambda\right\|_{\cH^p}^2 \lesssim \sigma^2\frac{1+\log(1/\delta)}{n}\lambda^{-p}F_{2-p}(\lambda).
\end{equation*}
\end{proposition}

We begin by examining the structure of the variance term. Let
$
	\xi=(\xi_1,\dots,\xi_n)^\top\in\RR^n.
$
Then, by Eq.~\eqref{eqn: variance-term}, we have
\begin{align}\label{eqn: xyz13}
\notag \left\|\hf^{\mathrm{var}}_\lambda\right\|_{\cH^p}^2 &=\left\|\cT^{-p/2}\Phi\cS_n^*(K_n+\lambda)^{-1}Q\hat{\xi}\right\|_\pi^2\\ 
&=\left\|\cT^{-p/2}\Phi(\hT+\lambda)^{-1}S_n^*Q\hat{\xi}\right\|_\pi^2 =\fn\xi^\top A_n \xi
\end{align}
with $A_n\in\RR^{n\times n}$ given by
\begin{align*}
A_n &= Q S_n(\hT_n + \lambda)^{-1}\Phi^* \cT^{-p}\Phi(\hT_n + \lambda)^{-1}\cS_n^*Q \\ 
&= QS_n(\hT_n + \lambda)^{-1}\cT^{1-p}(\hT_n + \lambda)^{-1}\cS_n^*Q.
\end{align*}
Observe that $\|\hf^{\mathrm{var}}_\lambda\|_{\cH^p}^2$ is a quadratic form with respect to $\xi$. Hence, we can apply the Hanson-Wright inequality \cite{vershynin2018high}*{Theorem~6.2.1} to bound the concentration. To this end, we first need to bound the trace of $A_n$.
\begin{lemma}
Given any $p$ such that $\sum_j\mu_j^p <\infty$, let  $\lambda\geq \Lambda_{2-p}(n,\delta)$. Then, \wp at least $1-\delta$, it holds that $\tr(A_n)\lesssim \lambda^{-p}F_{2-p}(\lambda)$.
\end{lemma}
\begin{proof}
Our proof relies on the following observation:
\[
    \tr(A_n) = \left\|\cT^{\frac{1-p}{2}}(\hT_n + \lambda)^{-1}\cS_n^* Q\right\|_{\HS}^2,
\]
where $\|\cdot\|_{\HS}$ stands for the Hilbert-Schmidt norm.
Due to $\|AB\|_{\HS} \le \|A\|_{\HS}\|B\|$, we have 
\begin{align*}
    \tr(A_n) &\le \|\cT^{\frac{1-p}{2}}(\hT_n + \lambda)^{-1}\cS_n^*\|_{\HS}^2\|Q\|^2 \\
    &\le 2 \tr[\cT^{\frac{1-p}{2}}(\hT_n + \lambda)^{-1}S_n^*S_n(\hT_n + \lambda)^{-1}\cT^{\frac{1-p}{2}}] \\
    &= 2 \|\cT^{\frac{1-p}{2}}(\hT_n + \lambda)^{-1}\hT_n^{\half}\|_{\HS}^2,
\end{align*}
where the first step is due to $q(x)\geq 1/2$.
Using  triangle inequality, we obtain
\begin{equation}\label{eqn:lem_D3_1}
    \begin{aligned}
    \sqrt{\half\tr(A_n)}& \leq \|\cT^{\frac{1-p}{2}}(\cT+\lambda)^{-1}\hT_n^\half\|_{\HS} + \|\cT^{\frac{1-p}{2}}[(\cT+\lambda)^{-1}-(\hT_n + \lambda)^{-1}]\hT_n^\half\|_{\HS}\\
    &= \|\cT^{\frac{1-p}{2}}(\cT+\lambda)^{-1}\hT_n^\half\|_{\HS} + \|\cT^{\frac{1-p}{2}}(\cT+\lambda)^{-1}(\hT_n-\cT)(\hT_n + \lambda)^{-1}\hT_n^\half\|_{\HS}, \\
    &=:\bM_1 + \bM_2,
\end{aligned}
\end{equation}
where the second steps use the equality: $A^{-1} - B^{-1} = A^{-1}(B-A)B^{-1}$.
\begin{itemize}
\item {\bf Bounding $\bM_1$.} Let $h_i = \cT^{\frac{1-p}{2}}(\cT+\lambda)^{\frac{p}{2}-1} \phi(x_i,\cdot) $ and note  $\hT_n=\frac{1}{n}\sum_{i=1}^n \frac{1}{q(x_i)}\phi(x_i,\cdot)\otimes \phi(x_i,\cdot)$. Then, we have
 \begin{align}\label{eqn: xyz15}
  \notag  \bM_1^2 &=   \tr((\cT + \lambda)^{-\frac{p}{2}}\cT^{\frac{1-p}{2}}(\cT+\lambda)^{\frac{p}{2}-1}\hT_n(\cT+\lambda)^{\frac{p}{2}-1}\cT^{\frac{1-p}{2}}(\cT + \lambda)^{-\frac{p}{2}})\\ 
   \notag     &\le \lambda^{-p} \tr\left(\frac{1}{n}\sum_{i=1}^n \frac{1}{q(x_i)} h_i\otimes h_i\right) =\frac{\lambda^{-p}}{n}\sum_{i=1}^n \frac{1}{q(x_i)}\|h_i\|_\pi^2 \\ 
    &\leq \lambda^{-p}\sup_{x}\frac{1}{q(x)}\|h(x)\|_\pi^2 = \lambda^{-p}F_{2-p}(\lambda),
 \end{align}
 where the last steps follows from Definition~\ref{def: degree-Fp} and Lemma~\ref{lemma: dof-rf}.

\item {\bf Bounding $\bM_2$.} Using the triangle inequality, we have
\begin{align}\label{eqn: xyz10}
\notag   \bM_2 &= \|\cT^{\frac{1-p}{2}}(\cT+\lambda)^{-1}(\hT_n-\cT)(\hT_n + \lambda)^{-1}\hT_n^\half\|_{\HS}\\ 
 \notag    &\leq  \|\cT^{\frac{1-p}{2}}(\cT+\lambda)^{-1}\hT_n(\hT_n + \lambda)^{-1}\hT_n^\half\|_{\HS} +  \|\cT^{\frac{1-p}{2}}(\cT+\lambda)^{-1}\cT(\hT_n + \lambda)^{-1}\hT_n^\half\|_{\HS}\\ 
  &= \|\cT^{\frac{1-p}{2}}(\cT+\lambda)^{-1}(\hT_n + \lambda)^{-1}\hT_n^{\frac{3}{2}}\|_{\HS} +  \|\cT^{\frac{3-p}{2}}(\cT+\lambda)^{-1}(\hT_n + \lambda)^{-1}\hT_n^\half\|_{\HS},
\end{align}
for which we shall bound the two terms on the right hand side separately. 
For the first term, using $\|ABC\|_{\HS} \le \|A\| \|B\|_{\HS}\|C\|$,  we obtain
\begin{align}\label{eqn: xyz11}
    \notag &\|\cT^{\frac{1-p}{2}}(\cT+\lambda)^{-1}(\hT_n + \lambda)^{-1}\hT_n^{\frac{3}{2}}\|_{\HS}\leq \|\cT^{\frac{1-p}{2}}(\cT+\lambda)^{-1}\hT_n^{\frac{1}{2}}\|_{\HS}\\ 
    \notag & \qquad \qquad\quad = \|(\cT+\lambda)^{-\frac{p}{2}}\cT^\frac{1-p}{2}(\cT+\lambda)^{\frac{p}{2}-1}\hT_n^\half\|_{\HS}\\ 
    \notag &\qquad \qquad\quad \le  \|(\cT+\lambda)^{-\frac{p}{2}}\|\|\cT^\frac{1-p}{2}(\cT+\lambda)^{\frac{p}{2}-1}\hT_n^\half\|_{\HS}\\ 
    &\qquad \qquad\quad = \|(\cT+\lambda)^{-\frac{p}{2}}\| \bM_1\le \lambda^{-\frac{p}{2}} F_{2-p}^{\half}(\lambda), 
\end{align}
where the last step follows from Eq.~\eqref{eqn: xyz15}.
For the second term, analogously, we obtain \wp at least $1-\delta$ that
\begin{align}\label{eqn: xyz12}
 \notag    &\|\cT^{\frac{3-p}{2}}(\cT+\lambda)^{-1}(\hT_n + \lambda)^{-1}\hT_n^\half\|_{\HS} \\ 
\notag   &\qquad\qquad\quad \leq \|\cT^{\frac{3-p}{2}}(\cT+\lambda)^{-\frac{3}{2}}\|_{\HS}\|(\cT+\lambda)^{\half}(\hT_n+\lambda)^{-\half}\|\|(\hT_n + \lambda)^{-\half}\hT_n^{\half}\|\\ 
 \notag    &\qquad\qquad\quad \leq \|\cT^{\frac{3-p}{2}}(\cT+\lambda)^{-\frac{3}{2}}\|_{\HS}\\
 \notag    &\qquad\qquad\quad \leq  \|\cT^{\half}(\cT+\lambda)^{-\half}\|\|\cT^{\frac{2-p}{2}}(\cT+\lambda)^{-\frac{2-p}{2}}\|_{\HS}\|(\cT+\lambda)^{-\frac{p}{2}}\|\\ 
 \notag    &\qquad\qquad\quad \leq \lambda^{-\frac{p}{2}}\sqrt{\tr(\cT^{2-p}(\cT+\lambda)^{p-2})} \\ 
    &\qquad\qquad\quad = \lambda^{-\frac{p}{2}}\sqrt{N_{2-p}(\lambda)} \leq \sqrt{\lambda^{-p} F_{2-p}(\lambda)},
\end{align}
where the second step uses Lemma~\ref{lemma: 11}.
Combining \eqref{eqn: xyz11} and \eqref{eqn: xyz12}, we obtain 
\[
    \bM_2 \leq 2\lambda^{-\frac{p}{2}} F_{2-p}^{\half}(\lambda).
\]
\end{itemize}
Combining the above bounds of $\bM_1$ and $\bM_2$ gives 
\begin{align}
\tr(A_n)\leq 4(M_1^2+M_2^2)\leq  20\lambda^{-p}F_{2-p}(\lambda).
\end{align}
\end{proof}

\paragraph*{The proof of Proposition~\ref{pro:variance_estiamte}.}
 Noticing that $\xi_1,\dots,\xi_n$ are i.i.d. $\sigma^2$-sub-Gaussion mean-zero random variables, we can use Hanson-Wright inequality \cite[Theorem~6.2.1]{vershynin2018high} to obtain that
\begin{equation*}
    \PP\left(\xi^\top A_n \xi \ge  \EE \xi^\top A_n \xi + t\right)\le 2\exp\left(-c\min\left\{\frac{t^2}{\sigma^4\|A_n\|_F^2},\frac{t}{\sigma^2\|A_n\|}\right\}\right).
\end{equation*}
Noticing  
\begin{align*}
 \EE \xi^\top A_n \xi &=\tr(A_n\EE[\xi\xi^\top])\le \sigma^2 \tr(A_n) \\ 
  \|A_n\| &\le \tr(A_n) \\ 
  \|A_n\|_F^2 &= \tr(A_n^2) \le \|A_n\| \tr(A_n) \le [\tr(A_n)]^2, 
\end{align*}
 we can choose $t \sim \sigma^2\tr(A_n) \log(1/\delta)$ to obtain that with probability $1-\delta$,
\begin{equation*}
    \xi^\top A_n \xi \lesssim \EE[\xi^\top A_n \xi]+t\leq (1+\log(1/\delta))\sigma^2 \tr(A_n) \lesssim  (1+\log(1/\delta))\sigma^2 \lambda^{-p} F_{2-p}(\lambda).
\end{equation*}
We conclude the proof by plugging the above estimate into \eqref{eqn: xyz13}.

\subsection{Proof of Theorem~\ref{thm: saturation-1}}
\begin{proof}
To prove the theorem, we need only to show that with probability $1$, $\|\sum_{i=1}^n a_i k(\cdot,x_i)\|_{\cH^p} < \infty$ implies that $a= 0$.
For  $m\in \NN$, let
\[
    \alpha_m=\inf_{x \in \cX}(\sum_{j=1}^{m}\mu_j^{2-p}|e_j(x)|^2)/(\sum_{j = 1}^{m}\mu_j^{2-p}).
\]
Then, the assumption~\eqref{eqn: condition-eigenfunction} becomes 
\[
    \limsup_{m\to \infty}\alpha_m>0.
\]
Noticing that
\begin{align*}
    \left\|\sum_{i=1}^n a_i k(\cdot,x_i)\right\|_{\cH^p}^2 &= \sum_{k =1}^\infty \mu_k^{2-p}\left(\sum_{i=1}^n a_i e_k(x_i)\right)^2\geq \sum_{j=1}^{m}\mu_j^{2-p} \left(\sum_{i=1}^n a_i e_k(x_i)\right)^2\\ 
    &=\sum_{i,i'=1}^n a_i a_{i'}\sum_{j=1}^m \mu_j e_j(x_i)e_j(x_{i'}) \\ 
    &= a^\top \bK_{\leq m} a \ge \lambda_{\min}(\bK_{\leq m}) \|a\|_2^2,
\end{align*}
where  $\bK_{\leq m} \in \RR^{n\times n}$ with $(\bK_{\leq m})_{i,i'} = \sum_{j = 1}^{m}\mu_j^{2-p} e_j(x_i) e_j(x_{i'})$. By \cite[Corollary 1]{barzilai2023generalization}, we have that for any $\delta \in (0,1)$, with probability  at least $1-\delta$, 
\begin{align*}
    \lambda_{\min}(\bK_{\leq m}) &\ge \alpha_{m}\left[\sum_{j=1}^{m}\mu_j^{2-p} - \frac{n}{\delta}\sqrt{\sum_{j=1}^{m}\mu_j^{2(2-p)}}\right] \ge \alpha_{m}\left[\sum_{j=1}^{m}\mu_j^{2-p} - \frac{n}{\delta}\sqrt{\mu_1^{(2-p)}\sum_{j=1}^{m}\mu_j^{(2-p)}}\right]\\
    &= \alpha_m \left[ \Big(\sqrt{\sum_{j=1}^{m}\mu_j^{(2-p)}} - \frac{n}{2\delta}\mu_1^{1-\frac{p}{2}}\Big)^2 - \frac{n^2}{4\delta^2}\mu_1^{2-p}\right]
\end{align*}
By the assumption $\sum_{j=1}^\infty \mu_j^{2-p} = \infty$ and $\limsup_{m \rightarrow \infty} \alpha_m > 0$, we have $\limsup_{ m \rightarrow \infty} \lambda_{\min}(\bK_{\leq m}) = \infty $. Hence, for any $\delta \in (0,1)$, with probability $1-\delta$, $ \|\sum_{i=1}^n a_i k(\cdot,x_i)\|_{\cH^p}^2 = \infty$ unless $a = 0$. Taking $\delta \rightarrow 0$, we complete the proof.
\end{proof}

\subsection{Proof of Theorem~\ref{thm: target-saturation}}
\label{sec:proof_of_theorem_ref_thm_target_saturation}

Our proof builds upon the approach developed in \cite{hinrichs2022lower}, which establishes a lower bound on the estimation error by leveraging the corresponding lower bound for integration error of functions in an RKHS. However, it is important to emphasize that the lower bound in \cite{hinrichs2022lower} applies to the entire RKHS $\cH$, which is insufficient to  justify the saturation effect. To address this, we establish a lower bound for the specific function $f^*=\mu_0e_0$, which exhibits smoothness of arbitrary order.

We first need the following generalization of classical Schur product theorem.
\begin{theorem}[\cite{vybiral2020variant}]\label{thm: schur-product}
Let $M\in\RR^{n\times n}$ be a symmetric positive semidefinite matrix. Then, 
\[
    M\circ M \succeq \frac{1}{n}(\mathrm{diag} M) (\mathrm{diag} M)^\top.
\]
\end{theorem}

\subsubsection{Periodic Kernels}\label{sec: periodic}
Recall that for the eigenfunctions of a periodic kernel $k:\TT\times \TT\mapsto\RR$ are the Fourier modes. Specifically,  its spectral decomposition is given by 
\[
 k(x,x') = \sum_{j\in \ZZ}\mu_j e_j(x)\overline{e_j(x')}= \sum_{j\in \ZZ} e_j(x-y), 
\]
where  $e_j(x) := \exp(2\pi \ii j x)$  with $\ii$ denoting the imaginary unit. To streamline our presentation, we let 
\[
    \he_j=(e_j(x_1),\dots,e_j(x_n))^\top \in\RR^n. 
\]

To use Theorem~\ref{thm: schur-product}, the key  is the following observation, developed in \cite{hinrichs2022lower}: squaring the periodic kernel is still a periodic kernel.  
Given a periodic kernel $m_1: [0,1]^2 \to \mathbb{C}$ given by $m_1(x, y) = \sum_{j \in \mathbb{Z}} \alpha_j e_j(x - y)$. Then, $m_2(x, y) := |m(x, y)|^2$ is still a periodic kernel and satisfies
\begin{align}\label{eqn: squaring}
m_2(x, y) = \sum_{i,j \in \mathbb{Z}} \alpha_i \alpha_j e_{i-j}(x - y) = \sum_{\ell \in \mathbb{Z}} \mu_\ell e_\ell(x - y),
\end{align}
where $\mu_\ell=\sum_{j}\alpha_j\alpha_{j-t}$. This means that  $\mu = \alpha * \alpha$, which provide a simple relationship between the eigenvalues of $m_2$ and $m_1$.

\begin{lemma}\label{lemma: lower-bound-Dirichlet}
Let $h(x,y)=\sum_{|j|\leq m}e_j(x-y)$ be the Dirichlet kernel, where  $m$ is even, and  $H_n=(h(x_i,x_j))_{i,j}\in\RR^{n\times n}$ denote the corresponding kernel matrix. Then,  for any $x_1,\dots,x_n\in\TT$, the following holds:
\begin{equation}
H_n\succeq \frac{m+1}{n}\bm{1}\bm{1}^\top.
\end{equation}
\end{lemma}
\begin{proof}
We shall construct another kernel such that $h(x,y)$ is its square and then apply Theorem~\ref{thm: schur-product}. Specifically,
let $p(x,y)=\sum_{|j|\leq m/2} e_j(x-y)$ and denote by $P_n$ the corresponding kernel matrix. By \eqref{eqn: squaring}, it is easy to verify 
\begin{equation}\label{eqn: xy3}
|p(x,y)|^2 = \sum_{|j|\leq m}(m+1-j)e_j(x-y).
\end{equation}
Thus, we have
\begin{align}
H_n = \sum_{|j|\leq m} \he_j\he_j^\top \succeq \frac{1}{m+1}\sum_{|j|\leq m} (m+1-j) \he_j\he_j^\top =\frac{1}{m+1} P_n\circ P_n,
\end{align}
where the last step is due to \eqref{eqn: xy3}.
Noting that $p(x,x)=m+1$ for any $x\in \TT$ and applying Theorem~\ref{thm: schur-product}, we have 
\[
    H_n\succeq \frac{1}{m+1} P_n\circ P_n \succeq \frac{1}{(m+1)n} (\diag P_n) (\diag P_n)^\top = \frac{m+1}{n}\bm{1}\bm{1}^\top.
\]
\end{proof}

We are now ready to prove the Theorem~\ref{thm: target-saturation}.
Noting $f^*=\mu_0 e_0\in \cH^{\infty}$, we have
\begin{align}\label{eqn: xy1}
\notag \left\|\sumin a_i k(x_i,\cdot) - \mu_0 e_0\right\|^2_{\cH^p}&= \left\|\sumin a_i \left(\sum_j \mu_j e_j(x_i)e_j\right) - \mu_0e_0\right\|^2_{\cH^p}\\ 
\notag  &= \left\|\sum_j \mu_j \left( \sumin a_i  e_j(x_i)\right) e_j - \mu_0 e_0\right\|^2_{\cH^p}\\ 
\notag  &= \sum_{j\neq 0}\mu_j^{2-p} (a^\top \he_j)^2  + \mu_0^{2-p} (a^\top \he_0-1)^2\\ 
&\geq \mu_m^{2-p} \left((a^\top \he_0-1)^2 + \sum_{j\neq 0,|j|\leq m} (a^\top \he_j)^2\right)
=: \mu_m^{2-p} F(a).
\end{align}
Due to $\he_0 = \bm{1}$, we  have
\begin{align}\label{eqn: xy2}
\notag F(a) &= 1 - 2a^\top \bm{1} + a^\top \sum_{|j|\leq m} \he_j\he_j^\top a = 1 - 2 a^\top \bm +  a^\top H_n a \\ 
\notag &\geq 1 - 2 a^\top \bm{1} +  a^\top \left(\frac{m+1}{n}\bm{1}\bm{1}^\top\right) a  \\ 
&\geq \min_{t} \left(\frac{m+1}{n}t^2 - 2t +1\right) = 1 - \frac{n}{m+1},
\end{align}
where the third step follows from Lemma~\ref{lemma: lower-bound-Dirichlet}.
Plugging \eqref{eqn: xy2} into \eqref{eqn: xy1}, we obtain
\[
    \left\|\sumin a_i k(x_i,\cdot) - \mu_0 e_0\right\|^2_{\cH^p}\geq \left(1-\frac{n}{m+1}\right)\mu_m^{2-p}.
\]
Taking $m=2n$ and plugging the above estimate into \eqref{eqn: app-lower-bound-saturation} completes the proof.

\subsubsection{Dot-Product Kernels}
We consider a dot-product kernel $k(x,x'):=\kappa(x^\top x')$ on $\SS^{d-1}$ with some $\kappa: [-1,1]\mapsto\RR$, $\cX=\SS^{d-1}$ and  $\rho=\mathrm{Unif}(\SS^{d-1})$. 
In this part, we restrict our attention to the scenario where $d \ge 3$ because dot-product kernels in $\SS^1$ coincide with the periodic kernels in $\TT$ discussed in Section \ref{sec: periodic}. 
By \cite{dot-product-kernel}, the spectral decomposition of $k$ is given by:
\begin{equation}\label{eqn: def_mu}
   k(x,x') = \sum_{k=0}^\infty\sum_{j=1}^{N(d,k)} \lambda_k Y_{k,j}(x)Y_{k,j}(x'),
\end{equation}
where $\lambda_k$ is the eigenvalue, the spherical harmonics $Y_{k,j}$ is the corresponding eigenfunction that satisfies  
$\EE_{x'\sim\tau_{d-1}}[\kappa(x^Tx')Y_{k,j}(x')]=\lambda_k Y_{j,k}(x)$ and
\begin{equation*}
    N(d,k)= \begin{cases}
\frac{2k+d-2}{k}\binom{k+d-3}{d-2} & \text{ if } k \ge 1,\\
1 & \text{ if } k = 0.
\end{cases}
\end{equation*}
Note that $\{\mu_j\}$ are the eigenvalues counted with multiplicity, while $\{\lambda_k\}$ are the eigenvalues counted without multiplicity.  We refer to \cite{schoenberg1988positive,dot-product-kernel} for more details on the spectral decomposition of a dot product kernel on $\SS^{d-1}$.  

Consider the Dirichlet kernel 
\begin{equation*}
    D_m(x,x') = \sum_{k=0}^m \sum_{j=1}^{N(d,k)} Y_{k,j}(x) Y_{k,j}(x').
\end{equation*}
By \cite{clutton1987generalized}, we have that
\begin{equation}\label{eq:dirichlet}
     D_m(x,x') = \sum_{k=0}^m T_k(x^\top x'),\, T_k(s) = \left(1+\frac{2k}{d-2}\right) C_k^{\frac{(d-2)}{2}}(s),
\end{equation}
where $C_k^\lambda(\cdot)$ represents the Gegenbauer polynomials, which are defined using the generating function:
$$
    (1-2st + t^2)^{-\lambda} = \sum_{k=0}^\infty C_k^\lambda(s) t^k,\, \forall \lambda > 0, \, s \in [-1,1]\, , t \in (-1,1).
$$
Noticing that the spherical harmonics \( Y_{k,j} \) are normalized with respect to the probability distribution \( \mathrm{Unif}(\mathbb{S}^{d-1}) \) rather than the unnormalized measure on \( \mathbb{S}^{d-1} \), Eq.~\eqref{eq:dirichlet} differs from the formula in \cite{clutton1987generalized} by a factor corresponding to the surface area of \( \mathbb{S}^{d-1} \). We refer to \cite{andrews1999special} for more details on sphere harmonics
and Gegenbauer polynomials.

Analogous to the proof in Section \ref{sec: periodic}, we consider the square of the Dirichlet kernel. The following lemma provides the key estimate for this kernel to prove Theorem~\ref{thm: target-saturation}:
\begin{lemma}\label{lem:square_dirichlet}
        $[D_m(x,x')]^2 = \sum_{k=0}^{2m} A_{k,m} T_{k}(x^\top x')$. Moreover, $A_{k,m} \ge 0$ and $\sup_{0 \le k \le 2m}A_{k,m} \lesssim (m+1)^{d-1}$.
\end{lemma}

\begin{proof}
By (3.11) in \cite{clutton1987generalized}, we have
\begin{equation*}
    [D_m(x,x')]^2 = \sum_{k=0}^{2m} A_{k,m} T_k(x^\top x'),
\end{equation*}
where
\begin{equation*}
    A_{k,m} = \frac{d-2}{2k+d-2} \sum_{\substack{0 \leq l,l' \leq m \\
    |l - l'| \leq k \leq l + l' \\
    l + l' + k \text{ even}}}
    \left(1+\frac{2l}{d-2}\right)\left(1+\frac{2l'}{d-2}\right) \Pi(l,l';k),
\end{equation*}
with, for any $0 \le r \le \min(l,l')$,
\begin{equation*}
    \Pi(l,l'; l+l'-2r) = \frac{l+l'+\frac{d}{2}-1-2r}{l+l'+\frac{d}{2}-1-r}\frac{\left(\frac{d}{2}-1\right)_r\left(\frac{d}{2}-1\right)_{l-r}\left(\frac{d}{2}-1\right)_{l'-r}(d-2)_{l+l'-r}(l+l'-2r)!}{r!(l-r)!(l'-r)!\left(\frac{d}{2}-1\right)_{l+l'-r}(d-2)_{l+l'-2r}}.
\end{equation*}

Since $(x)_r = \frac{\Gamma(x+r)}{\Gamma(x)}$, $ r! = \Gamma(r+1)$, we can apply Stirling's formula $\Gamma(x) \sim \sqrt{2\pi} x^{x-\half} e^{-x}$ to estimate
\begin{equation*}
    \frac{\left(\frac{d}{2}-1\right)_r}{r!} = \frac{\Gamma(\frac{d}{2}-1+r)}{\Gamma(r+1)\Gamma(\frac{d}{2}-1)} \asymp (r+1)^{\frac{d}{2}-2}.
\end{equation*}
Similarly,
\begin{align*}
\frac{\left(\frac{d}{2}-1\right)_{l-r}}{(l-r)!} &\asymp (l-r+1)^{\frac{d}{2}-2},\\
\frac{\left(\frac{d}{2}-1\right)_{l'-r}}{(l'-r)!} &\asymp (l'-r+1)^{\frac{d}{2}-2},\\
\frac{(d-2)_{l+l'-2r}}{(l+l'-2r)!} &\asymp (l+l'-2r+1)^{d-3}.
\end{align*}
Moreover, applying Stirling's formula again, we have
\begin{align*}
    \frac{(d-2)_{l+l'-r}}{(\frac{d}{2}-1)_{l+l'-r}} &= \frac{\Gamma(l+l'-r + d-2)\Gamma(\frac{d}{2}-1)}{\Gamma(l+l'-r+\frac{d}{2}-1)\Gamma(d-2)}\\&\asymp \frac{(l+l'-r+d-2)^{l+l'-r+d-5/2}}{(l+l'-r+\frac{d}{2}-1)^{l+l'-r+\frac{d}{2}-3/2}}\asymp(l+l'-r+1)^{\frac{d}{2}-1}.
\end{align*}
Combining with the fact that
\begin{equation*}
    \frac{l+l'+\frac{d}{2}-1-2r}{l+l'+\frac{d}{2}-1-r} \asymp \frac{l + l'- 2r + 1}{l + l' -r + 1},
\end{equation*}
we have that
\begin{align*}
&\left(1+\frac{2l}{d-2}\right)\left(1+\frac{2l'}{d-2}\right)\Pi(l,l'; l+l'-2r)\\
&\quad \lesssim (l+1)(l'+1)(r+1)^{\frac{d}{2}-2}(l-r+1)^{\frac{d}{2}-2}(l'-r+1)^{\frac{d}{2}-2}(l+l'-r+1)^{\frac{d}{2}-2}(l+l'-2r+1)^{4-d}.
\end{align*}

For $d \ge 4$, by the fact that $l,l' \in [r,m]$, we have:
\begin{align*}
&\left(1+\frac{2l}{d-2}\right)\left(1+\frac{2l'}{d-2}\right)\Pi(l,l'; l+l'-2r)\\
&\quad \lesssim (l+1)(l'+1)(r+1)^{\frac{d}{2}-2}(l+l'-r+1)^{\frac{d}{2}-2}\left[\frac{(l-r+1)(l'-r+1)}{(l+l'-2r+1)^2}\right]^{\frac{d-4}{2}}\\
&\quad \lesssim (m+1)^{d-2}.
\end{align*}
Hence, we have
\begin{equation*}
A_{k,m} \lesssim \frac{(m+1)^{d-2}}{k+1}\sum_{\substack{0\leq l,l' \leq m\\|l-l'|\leq k \leq l+l'\\l+l'+k\text{ even}}}1 \lesssim \frac{(m+1)^{d-2}(m+1)(k+1)}{k+1} = (m+1)^{d-1}.
\end{equation*}

When $d = 3$, recalling that $k = l + l' - 2r$, we have that
\begin{align*}
&\left(1+\frac{2l}{d-2}\right)\left(1+\frac{2l'}{d-2}\right)\Pi(l,l'; l+l'-2r)\\
&\quad \lesssim (l+1)(l'+1)(r+1)^{-\half}(l-r+1)^{-\half}(l'-r+1)^{-\half}(l+l'-r+1)^{-\half}(l + l' -2r + 1) \\
&\quad \lesssim (m+1)^{\frac{3}{2}}(r+1)^{-\half}(l-r+1)^{-\half}(l'-r + 1)^{-\half}(l + l' - 2r+1) \\
&\quad \lesssim (m+1)^{\frac{3}{2}}(k+1)(l + l' - k + 1)^{-\half}(l - l' + k + 1)^{-\half}(l' - l + k + 1)^{-\half}
\end{align*}
Hence, we have that
\[
\begin{aligned}
  A_{k,m} &\lesssim (m+1)^{\frac{3}{2}} \sum_{\substack{0\le l,l'\le m \\ |l-l'|\le k \le l+l' \\ l+l'+k\text{ even}}} (l+l'-k+1)^{-\frac{1}{2}} (l-l'+k+1)^{-\frac{1}{2}} (l'-l+k+1)^{-\frac{1}{2}}\\[1mm]
  &\lesssim (m+1)^{\frac{3}{2}} \sum_{L_1=k}^{2m} (L_1-k+1)^{-\frac{1}{2}} \sum_{L_2=-k}^{k} \left[(k+1)^2 - L_2^2\right]^{-\frac{1}{2}}.
\end{aligned}
\]
Recall that
\[
(n+1)^{-\frac{1}{2}} \le 2\Big[(n+1)^{\frac{1}{2}} - n^{\frac{1}{2}}\Big],
\]
so that
\[
\sum_{L_1=k}^{2m} (L_1-k+1)^{-\frac{1}{2}} \le 2(2m-k+1)^{\frac{1}{2}} \lesssim (m+1)^{\frac{1}{2}}.
\]
Moreover, 
\[
\begin{aligned}
\sum_{L_2=-k}^{k}\left[(k+1)^2-L_2^2\right]^{-\frac{1}{2}} 
&\le 2\sum_{L_2 = 0}^K [(k+1) + L_2]^{-\half}[(k+1) - L_2]^{-\half} \\
&\lesssim (k+1)^{-\half}\sum_{L_2 = 0}^K (k+1-L_2)^{-\half} \lesssim 1.
\end{aligned}
\]
Therefore,
\[
A_{k,m} \lesssim (m+1)^{\frac{3}{2}} (m+1)^{\frac{1}{2}} = (m+1)^2 = (m+1)^{d-1},
\]
which completes the proof.

\end{proof}

Analogous to the method shown in Section \ref{sec: periodic}, we will fix $x_1,\dots,x_n\in\SS^{d-1}$ and let
\[
   \he_{k,j}=(Y_{k,j}(x_1),\dots,Y_{k,j}(x_n))^\top \in\RR^n. 
\]

\begin{lemma}\label{lemma: lower-bound-Dirichlet-dot}
Let $H_{n,m} = (D_{m}(x_i,x_j))_{1\le i, j\le n}$ be the kernel matrix corresponding to the Dirichlet kernel.
Then, there exists constant $C > 0$ which only depending on $d$ such that we have that:
\begin{equation}
H_{n,2m}\succeq C\frac{(m+1)^{d-1}}{n}\bm{1}\bm{1}^\top.
\end{equation}
\end{lemma}
\begin{proof}
By Lemma \ref{lem:square_dirichlet}, we have that there exists positive constant $B_{k,m}$ such that
\begin{align}
H_{n,2m} - \frac{C}{(m+1)^d} H_{n,m}\circ H_{n,m}  = \sum_{k=0}^{2m}B_{k,m}\sum_{j=1}^{N(d,k)} \he_{k,j}\he_{k.j}^\top \succeq 0.
\end{align}
By Theorem 9.6.3 in \cite{andrews1999special}, we know that for any $x \in \SS^{d-1}$
$$
D_m(x,x) = \sum_{k = 0}^m \sum_{j=1}^{N(d,j)} |Y_{k,j}(x)|^2 = \sum_{k=0}^m N(d,j),
$$
Noticing that $N(d,k+1) + N(d+1,k) = N(d+1,k+1)$ and Stirling's formula, we know that 
\begin{equation}\label{eqn:estimate_N_d_k}
    \sum_{k=0}^m N(d,j) = N(d+1,m) \asymp \frac{\Gamma(m+d-1)}{\Gamma(d)\Gamma(m)} \asymp \frac{(m+d-1)^{m+d-3/2}}{m^{m-1/2}} \asymp (m+1)^{d-1}
\end{equation}
Hence, by applying Theorem~\ref{thm: schur-product}, we have that
\[
H_{n,2m} \succeq \frac{C}{(m+1)^{d-1}} H_{n,m}\circ H_{n,m} \succeq \frac{C}{(m+1)^{d-1} n} (\diag H_{n,m})(\diag H_{n,m})^\top \succeq \frac{C (m+1)^{d-1}}{n} \bm{1}\bm{1}^\top.
\]
\end{proof}

The remaining part of Theorem~\ref{thm: target-saturation} concerning the dot-product kernel closely mirrors the proof for the periodic kernel presented in Section \ref{sec: periodic}.
Noting $f^*=\lambda_0 Y_{0,1}\in \cH^{\infty}$, we have
\begin{align}\label{eqn: xy1-dot}
\notag \left\|\sumin a_i k(x_i,\cdot) - \lambda_0 Y_{0,1}\right\|^2_{\cH^p}&= \left\|\sumin a_i \left(\sum_{k=0}^\infty\sum_{j=1}^{N(d,j)} \lambda_k Y_{k,j}(x_i)Y_{k,j}\right) - \lambda_0Y_{0,1}\right\|^2_{\cH^p}\\ 
\notag  &= \left\|\sum_{k=0}^\infty \lambda_k\sum_{j=1}^{N(d,k)}  \left( \sumin a_i  Y_{k,j}(x_i)\right) Y_{k.j} - \lambda_0 Y_{0,1}\right\|^2_{\cH^p}\\ 
\notag  &= \sum_{k\neq 0}\lambda_k^{2-p} \sum_{j=1}^{N(d,k)}(a^\top \he_{k,j})^2  + \lambda_0^{2-p} (a^\top \he_{0,1}-1)^2\\ 
&\geq \lambda_{2m}^{2-p} \left((a^\top \he_{0,1}-1)^2 + \sum_{k=1}^{2m} \sum_{j=1}^{N(d,k)} (a^\top \he_{k,j})^2\right)
=: \lambda_{2m}^{2-p} F(a).
\end{align}
Due to $\he_{0,1} = \bm{1}$, we  have
\begin{align}\label{eqn: xy2-dot}
\notag F(a) &= 1 - 2a^\top \bm{1} + a^\top \sum_{k=0}^{2m} \sum_{j=1}^{N(d,j)} \he_{k,j}\he_{k,j}^\top a = 1 - 2 a^\top \bm{1} +  a^\top H_{n,2m} a \\ 
\notag &\geq 1 - 2 a^\top \bm{1} +  a^\top \left(\frac{C(m+1)^{d-1}}{n}\bm{1}\bm{1}^\top\right) a  \\ 
&\geq \min_{t} \left(\frac{C(m+1)^d}{n}t^2 - 2t +1\right) = 1 - \frac{n}{C(m+1)^{d-1}},
\end{align}
where the third step follows from Lemma~\ref{lemma: lower-bound-Dirichlet-dot}.
Plugging \eqref{eqn: xy2-dot} into \eqref{eqn: xy1-dot}, we obtain
\[
    \left\|\sumin a_i k(x_i,\cdot) - \mu_0 e_0\right\|^2_{\cH^p}\geq \left(1-\frac{n}{C(m+1)^{d-1}}\right)\lambda_{2m}^{2-p}.
\]
Choosing $m$ such that $n \asymp \frac{C}{2}(m+1)^d$ or $m \asymp n^{1/(d-1)}$, we have that
\[
    \left\|\sumin a_i k(x_i,\cdot) - \mu_0 e_0\right\|^2_{\cH^p}\geq \frac{\lambda_{2m}^{2-p}}{2}.
\]
Noticing that $\lambda_{2m} = \mu_{\sum_{k = 0}^{2m} N(d,k)}$ and recalling that $\sum_{k = 0}^{2m} N(d,k)\asymp m^{d-1}$ in \eqref{eqn:estimate_N_d_k}, we know that $N(d+1,2m)\asymp n$.
Plugging the above estimate into \eqref{eqn: app-lower-bound-saturation} completes the proof.

\section{Proofs of Important Lemmas}

\subsection{Proof of Lemma \ref{lem:symmetry}}
\label{sec:proof_of_lemma_ref_lem_symmetry}

Let $k_x=k(x,\cdot)$ for brevity. 
Our proof needs the following well-known result: By the reproducing property of $\cH$,  it holds for any $x,x'\in\cH$ that
\begin{equation}\label{eqn: kernel-norm}
    \langle k_x, k_{x'}\rangle_{\cH}  = k(x,x').
\end{equation}

\begin{definition}
For any $A \in \cG$, we define a linear operator $\bar{A}:\cH\mapsto\cH$ by $(\bar{A} f)(x) = f(Ax)$ for any $f \in \cH$ and $x\in\cX$.
\end{definition}
\begin{lemma}
For any $x\in\cX,A\in \cG$, $\bar{A} k_x = k_{A^{-1}x}$.
\end{lemma}
\begin{proof}
By the symmetry condition,  it holds for any $x,x'\in\cX$ that $k(A^{-1}x, x')=k(AA^{-1}x, Ax')=k(x,Ax')$. This means $\bar{A} k(x,\cdot) = k(A^{-1}x,\cdot)$.
\end{proof}

Moreover, we show it possesses the following properties. 
\begin{lemma}\label{lemma: commutativity}
For any $A\in\cG$, $\bar{A}$ and $\cL$ are commutable.
\end{lemma}
\begin{proof}
For  any $A \in \cG,f \in \cH$, and $x\in\cX$, we have
\begin{align*}
    (\cL \bar{A} f)(x) &= \int_{\cX}k(x,x')f(Ax')\dd\rho(x') \stackrel{\text{let } z=Ax'}{=} \int_{\cX}k(x,A^{-1}z)f(z)\dd \rho(z) \\
    &= \int_{\cX} k(Ax,z)f(z)\dd \rho(z) = (\bar{A}\cL f)(x),
\end{align*}
where the second step uses the measure-preserving property of the group action and third steps  is due to $k(x,A^{-1}z)=k(Ax,AA^{-1}z)=k(Ax,z)$.
\end{proof}

\begin{lemma}\label{lemma: norm-presevation}
For any $A\in\cG$, $\|\bar{A} f\|_{\cH} =\|f\|_{\cH}$ for any $f\in\cH$.
\end{lemma}
\begin{proof}
For any $\alpha_1,\dots, \alpha_m\in\RR$ and $x_1,\dots,x_m\in\cX$, we have 
\begin{align*}
    \left\|\bar{A}\left(\sum_{j=1}^m \alpha_j k_{x_j}\right)\right\|_{\cH}^2&= \left\|\sum_{j=1}^m \alpha_j \bar{A} k_{x_j}\right\|_{\cH}^2=  \left\|\sum_{j=1}^m \alpha_j  k_{A^{-1}x_j}\right\|_{\cH}^2\\ 
    &=\sum_{i,j=1}^m \alpha_i\alpha_j \left\langle k_{A^{-1}x_i}, k_{A^{-1}x_j}\right\rangle_{\cH} = \sum_{i,j=1}^m \alpha_i\alpha_j k(A^{-1}x_i, A^{-1}x_j)\\ 
    &= \sum_{i,j=1}^m \alpha_i\alpha_j k\left(x_i, x_j\right)= \sum_{i,j=1}^m \alpha_i\alpha_j \langle k_{x_i}, k_{x_j}\rangle_{\cH}\\ 
    &=     \left\|\sum_{j=1}^m \alpha_j k_{x_j}\right\|_{\cH}^2,
\end{align*}
where the forth and sixth steps use \eqref{eqn: kernel-norm}. The above equality shows that $\bar{A}$ preserves the norm on the subspace $\cH_0:=\spn\{k_x\}_{x\in\cH}\subset \cH$. Since $\cH_0$ is dense in $\cH$, the continuous extension theorem ensures that $\bar{A}$ also preserves the norm on $\cH$.
\end{proof}

\paragraph*{\underline{Proof of Lemma~\ref{lem:symmetry}.}}
For any $x, x' \in \cX$, let $A \in \cG$ satisfy $A^{-1} x = x'$. Then, we have
\begin{align*}
    \|\cL^{\frac{\gamma-1}{2}}(\cL + \lambda )^{-\frac{\gamma}{2}}k_{x'}\|_\cH &=\|\cL^{\frac{\gamma-1}{2}}(\cL + \lambda )^{-\frac{\gamma}{2}}k_{A^{-1}x}\|_\cH\\ 
    &=  \|\cL^{\frac{\gamma-1}{2}}(\cL + \lambda )^{-\frac{\gamma}{2}}\bar{A} k(x,\cdot)\|_\cH \\
    &=  \|\bar{A}\cL^{\frac{\gamma-1}{2}}(\cL + \lambda )^{-\frac{\gamma}{2}}k(x,\cdot)\|_\cH =  \|\cL^{\frac{\gamma-1}{2}}(\cL + \lambda )^{-\frac{\gamma}{2}}k(x,\cdot)\|_\cH,
\end{align*}
where the third and forth steps follow from Lemmas~\ref{lemma: commutativity} and \ref{lemma: norm-presevation}, respectively.
Hence, $ \|\cL^{\frac{\gamma-1}{2}}(\cL + \lambda )^{-\frac{\gamma}{2}}k_x\|_\cH$ is  constant in $x$. Therefore, when $q(x)\equiv 1$, we have  $F_\gamma(\lambda) = N_\gamma(\lambda)$.

\subsection{Proof of Lemma \ref{lemma: dof}}
\label{sec: proof-dof}

Let $m(\lambda)=\max\{j: \mu_j\ge \lambda, j\in \NN^+\}$. First, we have
\[
	N_\gamma(\lambda) \ge \sum_{j\leq m(\lambda)}\frac{\mu_j^\gamma}{(\mu_j+\lambda)^\gamma}\ge \frac{m(\lambda)}{2^\gamma}.
\]
Second,
\begin{align}
\notag	N_\gamma(\lambda) &= \sum_{j\leq m(\lambda)} \frac{\mu_j^\gamma}{(\mu_j+\lambda)^\gamma} +  \sum_{j\ge m(\lambda)+1} \frac{\mu_j^\gamma}{(\mu_j+\lambda)^\gamma} \\ 
&\leq m(\lambda) + \lambda^{-\gamma} \sum_{j\ge m(\lambda)+1} \mu_j^\gamma=:m(\lambda) + q(\gamma,\lambda),
\end{align}
where $q(\gamma,\lambda)=\lambda^{-\gamma} \sum_{j\ge m(\lambda)+1} \mu_j^\gamma$. 

{\bf The case of polynomial decay.} When $\mu_j\asymp j^{-\beta}$, we have $m(\lambda)\asymp \lambda^{-\frac{1}{\beta}}$ and
\[
q(\gamma,\lambda) \leq \lambda^{-\gamma} \sum_{j\ge m(\lambda)+1}j^{-\beta\gamma}\leq \lambda^{-\gamma}\frac{1}{\beta \gamma-1} (m(\lambda)+1)^{1-\beta s}\lesssim \lambda^{-\gamma} (\lambda^{-\frac{1}{\beta}})^{1-\beta\gamma} = \lambda^{-\frac{1}{\beta}}.
\]
Together, we have $N_\gamma(\lambda)\asymp\lambda^{-\frac{1}{\beta}}$.

{\bf The case of exponential decay.} When $\mu_j \asymp c^{-j}$, we have $m(\lambda)\asymp \log(1/\lambda)$ and 
\[
q(\gamma,\lambda) \le \lambda^{-\gamma} \sum_{j\ge m(\lambda)+1} c^{-\gamma j} \lesssim  \lambda^{-s}\frac{\lambda^{\gamma}}{1-c^{-\gamma}}\lesssim 1.
\]
Thus it follows that $N_\gamma(\lambda)\asymp_\gamma \log(1/\lambda)$.

\subsection{Proof of Lemma~\ref{lemma: uniform-necessity}}
\label{sec:proof_of_lemma_ref_lemma_uniform_necessity}

\begin{proof}
By the condition,  $\sup
_{\|f\|_{\cH^s} \le 1}\|\hat{f}_\lambda\|_{\cH^p} < \infty$.
 Noticing that $\hat{f}_\lambda = \hat{k}_n(\cdot)^\top \hat{a}_\lambda$,   we have 
\begin{equation*}
    \|\hat{f}_\lambda\|_{\cH^p}^2 = \hat{a}_\lambda^\top G_n \hat{a}_\lambda,
\end{equation*}
where $G_n \in \RR^{n\times n}$ and
$
    (G_n)_{i,i'} = \frac{1}{n}\sum_{j=1}^\infty \mu_j^{2-p} e_j(x_i)e_j(x_{i'}).
$
Noticing that $k(\cdot,x_1),\dots,k(\cdot,x_n)$ are linearly independent, we have that the smallest eigenvalue of $\kappa_n:=\lambda_{\min}(G_n)>0$. Hence,
\begin{align*}
    \sup_{\|f^*\|_{\cH^s}\le 1}\|\hat{f}_\lambda\|^2_{\cH^p} & \ge \sup_{\|f^*\|_{\cH^s} \le 1}\kappa_n\|\hat{a}_\lambda\|^2 \ge \lambda^{-2}\kappa_n\sup_{\|f^*\|_{\cH^s} \le 1}\|\hat{y}\|^2= \lambda^{-2}\kappa_n\sup_{\|f^*\|_{\cH^s} \le 1} \fn\sum_{i=1}^n|f^*(x_i)|^2.
\end{align*}
This implies that if $\sup
_{\|f\|_{\cH^s} \le 1}\|\hat{f}_\lambda\|_{\cH^p} < \infty$ holds almost surely, then  $\sup_{\|f\|_{\cH^s} \le 1}|f(x)|^2 < \infty$ must hold almost surely. 

Notice that
\begin{equation}\label{eqn: 000}
     \sup_{\|f\|_{\cH^s} \le 1}|f(x)|^2= \sup_{\sum_{j=1}^\infty \mu_j^{-s} a_j^2 \le 1}|\sum_{j=1}^\infty a_j e_j(x)|^2 = \sum_{j=1}^\infty \mu_j^s |e_j(x)|^2.
\end{equation}
Analogous to the proof of Lemma~\ref{lem:symmetry}, we can show that $\sum_{j=1}^\infty \mu_j^s |e_j(x)|^2$ is independent of $x$. Therefore, the condition  $\sup_{\|f\|_{\cH^s} \le 1}|f(x)|^2 < \infty$ holding a.s.~implies that $\sum_{j=1}^\infty\mu_j^s<\infty$.
\end{proof}

\subsection{Proof of Lemma~\ref{lemma: d2}}
Let $g_q^*(x) = \frac{g^*(x)}{\sqrt{q(x)}}$. By the definition of $m(\lambda)$, we have
\begin{equation}\label{eqn: 0001}
        \|g_q^*\|_{\rho'}^2 = \|g^*\|_{\rho}^2 = \sum_{j= m(\lambda)+1}^\infty \mu_j^{s} a_j^2 \leq \lambda^s\sum_{j= m(\lambda)+1}^\infty a_j^2\leq \lambda^s.
\end{equation}
Applying Theorem \ref{thm:embedding} gives
    \begin{equation}\label{eqn: 0002}
        \|g_q^*\|_{L^{\frac{2\gamma}{\gamma-s}}(\rho')}  \le C_{s,\gamma}\lambda^{\frac{s}{2}}F^{\frac{s}{2\gamma}}_\gamma(\lambda)\|g^*\|_{\cH^s} \leq C_{s,\gamma}\lambda^{\frac{s}{2}}F^{\frac{s}{2\gamma}}_\gamma(\lambda).
    \end{equation}
    Let $Y=(Y_1,\dots,Y_n)^\top \in \RR^n$ with $Y_i = \frac{1}{\sqrt{n}} g_q^*(X_i)$ and $X_i\sim \rho'$. By the Rosenthal-type inequality (Lemma~\ref{lemma: rosenthal-inequality}), we have
    \begin{align*}
        \left(\EE \|Y\|_2^\frac{2\gamma}{\gamma-s}\right)^{\frac{\gamma-s}{\gamma}} &= \left(\EE\left[\left(\sum_{i=1}^n |Y_i|^2\right)^{\frac{\gamma}{\gamma-s}}\right]\right)^{\frac{\gamma-s}{\gamma}} \le C_{s,\gamma}\max\left\{\left(n \EE|Y_1|^\frac{2\gamma}{\gamma-s}\right)^\frac{\gamma-s}{\gamma}, n \EE|Y_1|^2\right\}\\
        &=C_{s,\gamma}\max\left\{n^{-\frac{s}{\gamma}}\|g_q^*\|^2_{L^\frac{2\gamma}{\gamma-s}(\rho')},  \|g_{q}^*\|^2_{\rho'} \right\} \\ 
        &\stackrel{(i)}{\leq}  C_{s,\gamma} \max\left\{\left(\frac{F_\gamma(\lambda)}{n}\right)^{s/\gamma}, 1\right\}\lambda^s \\ 
        &\stackrel{(ii)}{\leq} C_{s,\gamma} \lambda^s.
    \end{align*}
    where $(i)$ uses \eqref{eqn: 0001} and \eqref{eqn: 0002}; $(ii)$ follows from $F_\gamma(\lambda)\leq F(\lambda_n)\leq n$.

    By  Markov-type concentration inequality (Lemma~\ref{eqn: markov}), for any $\delta_2\in (0,1)$, we have with probability at least $1-\delta_2$ that
    \begin{equation}\label{eq:thm45_5}
        \|Y\|_2 \le C_{s,\gamma} \delta_2^{-\frac{\gamma-s}{2\gamma}} \lambda^{s/2}.
    \end{equation}
\qed

\subsection{Proof of Lemma~\ref{lemma: Hp-norm-KRR2}}
Recalling 
\[
    \hf_\lambda = \Phi \cS_n^*(\cS_n\cS_n^*+\lambda I_n)^{-1} \hy = \Phi(\cS_n^*\cS_n + \lambda )^{-1}\cS_n^* \hy = \Phi(\hT_n+\lambda)^{-1}\cS_n^* \hy,
\] 
we have that
    \begin{align}\label{eq:thm45_6}
     \notag   \|\hf_\lambda\|_{\cH^p} &= \|\cT^{-\frac{p}{2}}\Phi(\hT_n + \lambda_n )^{-1} \cS_n^* \hy\|_\rho \\ 
        &\le \|\cT^{-\frac{p}{2}}\Phi(\hT_n + \lambda )^{-\frac{1}{2}}\|\|(\hT_n + \lambda )^{-\half}\cS_n^*\|\|\hy\|_2.
    \end{align}
    First,  we have
    \begin{equation*}\label{eq:thm45_8}
        \|(\hT_n + \lambda )^{-1}\cS_n^*\|^2 = \|(\hT_n + \lambda )^{-\frac{1}{2}}\cS_n^*\cS_n (\hT_n + \lambda )^{-\frac{1}{2}}\| = \|\hT_n(\hT_n + \lambda)^{-1}\| \le 1.
    \end{equation*}
    Second, noticing $\Phi^* \cT^{-p}\Phi = \cT^{1-p}$, we have
    \begin{align*}
        \|\cT^{-\frac{p}{2}}\Phi(\hT_n + \lambda )^{-\frac{1}{2}}\|^2 &= \|(\hT_n + \lambda )^{-\frac{1}{2}}\Phi^*\cT^{-p}\Phi(\hT_n + \lambda )^{-\frac{1}{2}}\|\\ 
        &=\|(\hT_n + \lambda )^{-\frac{1}{2}}\cT^{1-p}(\hT_n + \lambda )^{-\frac{1}{2}}\|\\ 
        & = \|\cT^{\frac{1-p}{2}}(\hT_n + \lambda )^{-\half}\|^2\\ 
        & \le \lambda^{-p}\|\cT^{\frac{1-p}{2}}(\hT_n + \lambda )^{\frac{p-1}{2}}\|^2 \le B^2 \lambda^{-p}.
    \end{align*}   
\qed

\end{document}